\documentclass[twoside]{article}

\usepackage[accepted]{aistats2025}
\usepackage[utf8]{inputenc} 
\usepackage[T1]{fontenc}    
\usepackage{hyperref}       
\usepackage{url}            
\usepackage{booktabs}       
\usepackage{amsfonts}       
\usepackage{nicefrac}       
\usepackage{microtype}      
\usepackage{xcolor}   
\usepackage{graphicx,subcaption} 
\usepackage{wrapfig}
\usepackage{floatflt}

\usepackage{floatrow}
\newfloatcommand{capbtabbox}{table}[][\FBwidth]

\usepackage{amsmath,amsfonts,bm,physics,bbm,mathtools}
\usepackage{amsthm}

\usepackage[title,titletoc]{appendix}

\def\eqref#1{equation~\ref{#1}}

\def\1{\bm{1}}

\def\1{\mathbf{1}}

\usepackage{algorithm2e}
\RestyleAlgo{ruled}
\usepackage{graphicx}
\makeatletter
\newcommand{\shorteq}{%
	\settowidth{\@tempdima}{-}
	\resizebox{\@tempdima}{\height}{=}%
}
\definecolor{betterblue}{rgb}{0,0,0.65}

\newcount\Comments
\newcommand{\kibitz}[2]{\ifnum\Comments=1{\textcolor{#1}{{\footnotesize #2}}}\fi}

\theoremstyle{plain}

\usepackage{enumitem,amssymb}
\newlist{todolist}{itemize}{2}
\setlist[todolist]{label=$\square$}
\usepackage{pifont}

\usepackage{stackengine}

\newcommand{\indep}{\perp \!\!\! \perp}

\theoremstyle{plain}
\newtheorem{theorem}{Theorem}
\newtheorem{proposition}{Proposition}
\newtheorem{lemma}{Lemma}

\theoremstyle{definition}

\newcounter{assume}
\newtheorem{assumption}[assume]{Assumption}
\theoremstyle{remark}
\newtheorem{remark}{Remark}
\newtheorem*{theorem*}{Theorem}

\newcommand{\argmax}{\operatornamewithlimits{argmax}}
\newcommand{\argmin}{\operatornamewithlimits{argmin}}

\usepackage[T1]{fontenc}
\usepackage[font=small,labelfont=bf,tableposition=top]{caption}

\DeclareCaptionLabelFormat{andtable}{#1~#2  \&  \tablename~\thetable}
\usepackage{blindtext}

\begin{document}

\twocolumn[

\aistatstitle{Optimal Downsampling for Imbalanced Classification with Generalized Linear Models}

\aistatsauthor{ Yan Chen \And Jose Blanchet \And  Krzysztof Dembczynski \And Laura Fee Nern \And Aaron Flores}

\aistatsaddress{ Duke University \And  Stanford University \And Yahoo \And Yahoo \And Yahoo} 

]

\begin{abstract}
    Downsampling or under-sampling is a technique that is utilized in the context of large and highly imbalanced classification models. We study optimal downsampling for imbalanced classification using generalized linear models (GLMs). We propose a pseudo maximum likelihood estimator and study its asymptotic normality in the context of increasingly imbalanced populations relative to an increasingly large sample size. 
    We provide theoretical guarantees for the introduced estimator. Additionally, we compute the optimal downsampling rate using a criterion that balances statistical accuracy and computational efficiency. Our numerical experiments, conducted on both synthetic and empirical data, further validate our theoretical results, and demonstrate that the introduced estimator outperforms commonly available alternatives.    
\end{abstract}

\section{Introduction}\label{intro}
The problem of training a machine learning classification model with imbalanced populations (or, equivalently, predicting rare events given contextual data) arises in a wide range of applications such online advertisement, healthcare, insurance or fraud detection\cite{basha2022review,haixiang2017learning,krawczyk2016learning,sisodia2022data,sisodia2023hybrid,razzaghi2015fast,huda2016hybrid,li2017adaptive,hassan2016modeling,more2017review,yan2015deep}). 
The rule of thumb in these imbalanced settings is downsampling or under-sampling \cite{lee2022downsampling,zhang2014downsampled,wu2009low,elad1997restoration,taha2021multilabel,li2017rare,basha2022review,haixiang2017learning,krawczyk2016learning,ksieniewicz2018undersampled}. For example, in online advertising the conversion rate is usually less than $0.1\%$ out of tens of millions of impressions, making downsampling a necessity for training. The downsampling process involves sampling a proportion of the majority population according to a suitable sampling rate, called the downsampling rate. 
This technique involves sampling a proportion of the majority population according to a suitable sampling rate, called the downsampling rate. After that, a model is directly trained on the down-sampled set of instances. If the model is used directly on test data, it will be expected to predict according to the original distribution, so a correction is needed either during training or during the prediction phase. There are multiple ways in which one can correct for downsampling, these methods are well documented in the literature, one can either apply a proper rescaling of the estimated conditional probabilities~\cite{Elkan_2001}, one can reweight the samples in training procedure, or one can apply an additional correction step (using, for example, isotonic regression~\cite{Niculescu-Mizil_Caruana_2005}). All of the above techniques will correct biases induced by downsampling, but obviously they have different mean-squared error.

Perhaps surprisingly, given how prevalent downsampling is in practice, as far as we know, a systematic study of the ``best" estimator (at least in theory) and its practical implications for applying downsampling has not been studied. The most recent paper that partially addresses these issues is \cite{wang2020logistic}, where the author explored the asymptotic property of downsample estimator under imbalanced classification. But analysis there was constrained within logistic regression model and didn't include any guidelines for selecting downsample rate.
This motivates our study. While most of our discussion focuses on generalized linear models, our results also offer insights into the application of downsampling in classification tasks using neural networks. 

Our goal is to offer new estimators that are not only consistent and asymptotically normal, but are guided by optimal design in terms of variance. It is well known that the maximum likelihood estimator attains the Cramer-Rao lower bound and therefore it is optimal in this sense. We then use this minimum variance insight to help users select an optimal downsampling rate. 

\paragraph{Contributions and Overview of Results.} Our contributions are as follows:
\begin{enumerate}
    \item[1)] Provide an explicit characterization of the maximum likelihood estimator in highly imbalanced data. 
    \vspace{-0.3cm}
    \item[2)] Introduce new MLE-informed estimators that offer significant improvements relative to the available alternatives.
    \vspace{-0.3cm}
    \item[3)] Provide guidance for the choice of optimal downsampling rates. 
\end{enumerate}

In order to formally capture highly imbalanced data, we consider an asymptotic regime of the form $\mathbb{P}(Y=1|X=x)=1-F(\tau_n+\theta_{*}^T x)\equiv\bar{F}(\tau_n+\theta_{*}^T x)$, where $F$ is the \textit{cumulative distribution function} (c.d.f.) of an underlying latent variable. The highly imbalanced data setup is captured by the location parameter $\tau_n\rightarrow\infty$ so that $n(1-F(\tau_n))\rightarrow\infty$. By doing this, we capture a wide range of situations for which data is imbalanced relative to the sample size. We do not think that the data is becoming increasingly imbalanced as the sample size grows, rather, the asymptotic statistics provide tools for approximate inference that are applied for a fixed sample size and, in the same spirit, our scaling introduces approximations that are valid within a wide range of imbalanced proportions.   
This creates a regime of data imbalance in which the minority proportion converges to zero as the location parameter $\tau_n$ becomes extreme. This setup is useful in practice. For example, if $n=10^6$, $\tau_n=0.6\log(n)\approx8$, and $\theta_*^Tx\geq1$. If $F$ follows a logistic regression model such that $F(\tau_n+\theta_*^Tx)=e^{\tau_n+\theta_*^Tx}/(1+e^{\tau_n+\theta_*^Tx})$, then the conversion rate $\mathbb{P}(Y=1|X=x)=1/(1+e^{\tau_n+\theta_*^Tx})<0.1\%$. So by appropriately controlling the growth rate of $\tau_n$, we can accurately recover the ratio of the rare events that are the focus of our analysis.

Firstly, we show that the prediction score from down-sample still follows a GLM structure, such that $\mathbb{P}(\tilde{Y}=1|\tilde{X}=x)=1-G(\tau_0+\theta_{*}^T x)\equiv\bar{G}(\tau_0+\theta_{*}^T x)$, where $(\tilde{Y},\tilde{X})$ is the random variable induced from downsample data, $\alpha$ is the downsample rate, and $G(z)=\frac{\alpha F(z)}{1-(1-\alpha)F(z)}$ is also a c.d.f. Based on this discovery, we propose a \textit{pseudo maximum likelihood estimator} (pseudo MLE) which can be computed directly on the downsample from imbalanced data by (\ref{pseudo-MLE}). 

Secondly, we obtain the asymptotic normality of the proposed pseudo MLE (see Theorem \ref{thm:rate-of-convergence-infinity} and Theorem \ref{thm:generalize:asymptotics}). The estimator is unbiased and has generally larger variance than the full-sample estimator (with zero downsampling rate). But for some small values of $\alpha$ the asymptotic variance stays the same as that of the full-sampling estimator, meaning that downsampling with these $\alpha$ results in no efficiency loss at all. We conduct numerical experiments and apply our estimator to logistic regression for different values of $\tau_n$. We use the mean squared error metric and compare the pseudo MLE with two commonly used alternative estimators: the inverse-weighting estimator (\ref{eq:inverse-weighting}) employed by \cite{wang2020logistic}, and the conditional MLE (\ref{eq:conditional-MLE}). The findings are that our estimator outperforms both of them when $\tau_n$ is large, but as $\tau_n$ decreases, it loses its advantage gradually, which is consistent with our theoretical conclusions. 

Additionally, we introduce a notion that trades the statistical error induced by the downsampling mechanism with the computational benefits derived by downsampling. We adapt the framework of \cite{glynn1992asymptotic} to introduce a concept of \textit{efficiency cost} that reflects the trade-off between statistical error and the computational cost. Our reasoning is as follows.

Suppose that there is a given computational routine to be applied in the training process. We assume that the routine produces an estimator with an $\epsilon$ error with complexity $O(\log(1/\epsilon)) \times n \times (\alpha(1-p_1)+p_1)$, where $n$ is the size of the original training set, $\alpha$ is the downsampling rate and $p_1$ is the minority population rate. The contribution of $\log(1/\epsilon)$ arises when training via gradient descent of a smooth convex loss \cite{rakhlin2011making,foster2019complexity,bubeck2015convex,johnson2013accelerating} and is used as an illustration of our reasoning (and is applicable to logistic regression, for example); the discussion that follows can be adapted to the cost incurred using other methods with different assumptions. In the instance we adopted here, consequently, with a computational budget of size $b = c\times(\alpha(1-p_1)+p_1)n\log(n)\log(\log(n))$ for some $c>0$ we obtain an overall estimator with a mean squared error $\sigma(\alpha)^2/n+o(1/n)$, where $\sigma(\alpha)^2/n$ is the mean squared error of the asymptotic downsampled estimator. We can write $n$ in terms of the budget, obtaining, for large $n$, $n = b(1+o(1))/(c\times(\alpha(1-p_1)+p_1)\log(b)\log(\log(b)))$. Therefore, minimizing the total error subject to a given budget constraint is asymptotically equivalent to just minimizing $\sigma(\alpha)^2\times(\alpha(1-p_1)+p_1)$ as a function of $\alpha$. We provide expressions to guarantee this optimal choice of $\alpha$ in practice in Section~\ref{section:efficiency} and we apply our theoretical findings to logistic regression in Section~\ref{sec:application:logistic regression}. 

Finally, we apply our estimator to real-world imbalanced datasets using both logistic regression and neural network models. The numerical results indicate that our estimator performs well in this experiment. More details can be found in Section \ref{sec:empirical} and Appendix \ref{appendix:additional simulation}. 

\paragraph{Related Work} The concept of Generalized Linear Model (GLM) has a long history dating back to as early as its introduction by \cite{nelder1972generalized} in the 1970s. 
Specifically, GLM plays an important role in classification tasks both in theory and applications \cite{gorriz2021connection, dobson2018introduction,deng2022model}. For example, a statistical test was developed by \cite{gorriz2021connection} based on GLM for pattern classification in machine learning applications. A high-dimensional GLM classification problem was explored by \cite{hsu2021proliferation} via support vector machine classifiers. GLM classifiers were also applied by \cite{ding2005classification,arnold2020reflection} in healthcare and computational biology.  

Particularly, classification for imbalanced data has generated the interest of many researchers \cite{japkowicz2000learning,king2001logistic,chawla2004special,chawla2010data,douzas2017self,lemaavztre2017imbalanced,estabrooks2004multiple,fithian2014local,han2005borderline,mathew2017classification,rahman2013addressing,drummond2003c4,owen2007infinitely}. In this context, also referred to as a rare event setting, one class within a population has significantly fewer instances than the other. Specifically, for binary labeled data, cases (the label group with fewer observations) are observed much less frequently than controls (the other label group). One example is online advertising data, where typically only 1 conversion is observed out of $1{,}000$ to $1{,}000{,}000$ impressions \cite{lee2012estimating,lee2021comparison,shah2021impacts,jiang2021multi}. An extensive survey of the area is given by \cite{chawla2004special}, where significant attention was given to downsampling methods, in which some available controls are removed for a balanced subsample. Another approach is oversampling \cite{douzas2017self,mathew2017classification,wang2020logistic,shelke2017review,yan2019oversampling} where additional instances are generated as cases.

Our work focuses on exploring the effects of downsampling \cite{fithian2014local,wang2020logistic,chawla2004special}, in which we maintain all the cases and uniformly sample an equal number of controls to make a balanced subsample dataset. This is also referred to as ``case-control sampling'' by previous literature \cite{fithian2014local}. A method of subsampling for logistic regression proposed by \cite{fithian2014local} was to adjust the class balance locally in feature space via an accept-reject scheme. 

The findings of a recent paper \cite{wang2020logistic} suggest that the available information in rare event data for binary logistic regression is at the scale of the number of positive examples, and downsampling a small proportion of the controls may induce an estimator with identical asymptotic distribution to the full-data maximum likelihood estimator (MLE). Besides, a different perspective was provided by \cite{drummond2003c4} for the performance analysis of sampling technique by considering cost sensitivity of misclassification. Motivated by the previous work, we study the imbalanced binary classification problem with generalized linear model structure. Similar to \cite{wang2020logistic}, we derived the asymptotic distribution of the induced downsampling estimator, and come to a similar conclusion that downsampling a small proportion of the controls by selecting downsampling rate within a certain range results in identical statistical efficiency as the full-sample MLE. 

However, our paper is more general than \cite{wang2020logistic} in the following aspects. Firstly, we consider generalized linear model as a more general and commonly seen family of models. Secondly, we use a new covariate-adapted estimator based on the correct maximum likelihood estimator for the downsample distribution, which turned out to be more efficient than the estimator used by \cite{wang2020logistic} under rare-event setup. Thirdly, we also provide clear guidance on the selection of downsample rate.

\paragraph{Notation} Given a thrice differentiable function $f$, we use $f^{(1)}$ or $f^{\prime}$, $f^{(2)}$ or $f^{\prime\prime}$, and $f^{(3)}$ or $f^{\prime\prime\prime}$ to denote the first-order, second-order or the third-order derivatives. We use $\bar{f}(\cdot)$ to denote $1-f(\cdot)$. For a sequence of random variables $Z_1,Z_2,\ldots$ in a metric space $\mathcal{Z}$, we say $Z_n\overset{d}{\rightarrow}Z$ if $Z_n$ converges in distribution to $Z$, and $Z_n\overset{p}{\rightarrow}Z$ if $Z_n$ converges in probability to $Z$. Given matrix $A, B$, $A\succ B$ indicates $A-B$ is positive definite. Given target parameter $\theta\in\mathbb{R}^d$ and its estimator $\hat{\theta}\in\mathbb{R}^d$, we use mean-squared-error or mean squared estimation error to refer to  {\footnotesize$\mathbb{E}\left[\norm{\hat{\theta}-\theta}^2\right]$}. We denote $\mu(\cdot)$ as the density of $X$ and $\mathcal{X}$ as the covariate space. 

\section{Problem Setup}
We observe an i.i.d.\ generated dataset $\{(Y_i,X_i)\}_{i=1}^{n}$, where $X_i\in\mathbb{R}^d, d\geq1$, $Y_i\in\{0,1\}$. We use $\mathbb{P}$ to denote the joint probability distribution of $(X_i,Y_i)$, and $\mathbb{E}[\cdot]$ to denote the expectation induced by $\mathbb{P}$. The data generating process follows \textit{Generalized Linear Model} (GLM), such that given a latent random variable $Z$, the observed binary label $Y$ given covariate $X$ is defined as {\footnotesize$Y = \mathbf{1}\left(Z>\tau_n+\theta_*^TX\right)$}, with {\footnotesize$\Bar{F}_Z(\tau_n+\theta_*^Tx)=\mathbb{P}\left(Y=1\big|X=x\right)$}, where $\bar{F}(\cdot):=\bar{F}_Z(\cdot)=1-F_Z(\cdot)$ and $F_Z(\cdot)$ is the \textit{cumulative distribution function} (c.d.f.) of $Z$. We are interested in studying the impact of downsampling for the unbalanced dataset with $\mathbb{P}(Y=1)\ll\mathbb{P}(Y=0)$. To do this we keep all the positive samples (i.e., those with $Y_i=1$), and uniformly sample a proportion $\alpha\in(0,1]$ of the samples with $Y_i=0$, so that we get a more balanced downsample $\{(\tilde{Y}_i,\tilde{X}_i)\}_{i=1}^N$ of size {\footnotesize$N=np_1+(1-p_1)n\alpha$}, where $p_1$ is the ratio of positive samples. We use $\tilde{\mu}$ to denote the joint density of downsample random pairs $(\tilde{Y},\tilde{X})$. 

Specifically, we consider the case where $\tau_n$ is a known sequence such that $\tau_n\rightarrow\infty$ as $n\rightarrow\infty$. This corresponds to the rare-event setup where $\bar{F}(\tau_n+\theta_{*}^T x)\rightarrow0$ for any $x\in\mathcal{X}$, where $\mathcal{X}$ is bounded in $\mathbb{R}^d$. 
The assumption that $\tau_n\rightarrow\infty$ is without loss of generality by noting that if $\tau_n\rightarrow-\infty$, then $\mathbb{P}(Y=0)\approx0$ and $\mathbb{P}(Y=1)\approx1$, so the analysis is similar to the case $\tau_n\rightarrow$ by switching the role of $\mathbb{P}(Y=0)$ and $\mathbb{P}(Y=1)$. We consider the following data generation process: Given sample size $n$ and $\tau_n$, we observe i.i.d. data $\{X_i,Y_i\}_{i=1}^n$ with 
$$\mathbb{P}(Y_i=1|X_i)=\bar{F}(\tau_n+\theta_*^T X_i)=1-F(\tau_n+\theta_*^T X_i),$$
where $\theta_*$ is the estimation target. 

\section{Proposed Estimator}\label{sec:calibration}
In this section, we first demonstrate that under Generalized Linear Model (GLM) setup, the induced variables by downsampling still follow a GLM model. Based on this finding, we propose a \textit{pseudo maximum likelihood estimator} which can be computed directly from downsample. 

Given joint law $\mathbb{P}$ of the full-sample random variable $(Y,X)$, and denote the induced joint law of downsample random variable $(\tilde{Y},\tilde{X})$ as $\tilde{P}$, there exists a monotone function $m$ such that $\mathbb{P}(Y=1|X;\theta)=m(\tilde{P}(Y=1|X;\theta))$, where the probability distributions are parametrized by $\theta\in\Theta$. Note that this justifies the correction method for downsampling by imposing a monotone transformation on the predicted scores based on the model trained from the downsample. Besides, we explicitly formulate the downsample maximum likelihood estimator (MLE), where the target parameter for the downsample MLE coincides with the one from the full-sample model. 

One common practice for estimation on imbalanced dataset is that people usually fit the parameter as if the downsample model still belongs to the same class of models as $F(\cdot)$, after which they obtain $F(\tau_n+\hat{\theta}_1^T X)$ as the score from the model trained on the downsample, and then use isotonic regression, i.e. utilizing a monotone mapping $g(\cdot)$ so that $g(F(\tau_n+\hat{\theta}_1^T X))$ is obtained as the final output of their prediction score given any covariate $X_i$, where given $\tau_n$ known, $\hat{\theta}_1$ is achieved by solving the following problem:
\begin{equation}\label{eq:downsample:original-model}
\!\!\!\!\!
\begin{array}{r@{}l}    
\hat{\theta}_1&=\argmax_{\theta_1}\frac{1}{N}\sum_{i=1}^N\tilde{Y}_i\log(1-F(\tau_n+\theta_1^T \tilde{X}_i))\\
&\quad+(1-\tilde{Y}_i)\log F(\tau_n+\theta_1^T \tilde{X}_i).
\end{array}
\end{equation}
Lemma \ref{lemma:counterexample:downsample-use-original} in Appendix \ref{appendix:calibration:downsampling} shows a counterexample to illustrate why this method can lead to biased prediction score for some covariates. Specifically, Lemma \ref{lemma:counterexample:downsample-use-original} constructs a counterexample where $\hat{\theta}_1$ leads to biasedness. And if $h(x)=\tilde{\theta}_1^T x$ maps two distinct $x,x^\prime$ to the same value while $\theta_*^T x\neq\theta_*^T x^\prime$, and the c.d.f. $F(\cdot)$ is strictly increasing, then the true model has different probabilities for $x,x^\prime$ whereas the model trained on the downsample maps them to the same prediction score, thus both parameter estimator and the induced prediction score model are biased. 

\begin{proposition}[Downsample Prediction Score]\label{lemma:downsample GLM}
The prediction score of the downsampled random variables follows 
$\mathbb{P}(\tilde{Y}_i=1|\tilde{X}_i)=\frac{\mathbb{P}(Y_i=1|X_i=1)}{\alpha+(1-\alpha)\mathbb{P}(Y_i=1|X_i=1)}.$
Hence 
{\footnotesize$\{(\tilde{Y}_i,\tilde{X}_i)\}_{i=1}^N$} still follows GLM, with the conditional probability {\footnotesize$\mathbb{P}(\tilde{Y}_i=1|\tilde{X}_i)=\bar{G}(\tau_n+\theta_1^T\tilde{X}_i)$}, and $\bar{G}(z)=1-G(z)$, where  
\begin{equation}\label{eq:G}
    G(z):=\frac{\alpha F(z)}{1-(1-\alpha)F(z)}
\end{equation}
is the c.d.f. of some random variable. 
\end{proposition}
Proposition \ref{lemma:downsample GLM} illustrates that the downsample random variables generated from GLM still follow a GLM, where $F(z)=\frac{G(z)}{\alpha+(1-\alpha)G(z)}$, indicating that the true class probabilities are the monotonic transformations of the predicted class probabilities conditional on the downsample. Based upon this, the joint likelihood of the downsample random variables is as follows.

\begin{proposition}[Downsample Joint Likelihood of $(\tilde{Y},\tilde{X})$]\label{lemma:identification-theta}
Given any $\tau_n$, the maximum likelihood estimator for the target estimand $\theta_*$ (i.e. the true value of the unknown parameter) is 
{\begin{equation}\label{eq:identification}
    \hat{\theta}:=\argmax_{\theta_1}\frac{1}{N}\sum_{i=1}^N\ell(\tilde{X}_i,\tilde{Y}_i,\theta_1;\tau_n),
\end{equation}}
where 
{\begin{equation}\label{eq:likelihood-downsample}
\begin{array}{rl}
&\quad\displaystyle\ell(\tilde{x},\tilde{y},\theta_1;\tau_n)\\
&=\displaystyle\tilde{y}\log\bar{F}(\tau_n+\theta_1^T \tilde{x})+(1-\tilde{y})\log\alpha F(\tau_n+\theta_1^T \tilde{x})\\
&\quad\displaystyle-\log\int_{\mathcal{X}}\left[1-(1-\alpha)F(\tau_n+\theta_1^T x)\right]\mu(x)dx.
\end{array}
\end{equation}}

\end{proposition}

Proposition \ref{lemma:identification-theta} exhibits the empirical maximum likelihood estimator of downsample GLM with respect to the joint distribution of $(\tilde{Y},\tilde{X})$. However, we note that (\ref{eq:likelihood-downsample}) utilizes the marginal density of full-sample $X$, which is expensive to estimate for large-scale imbalanced data. To tackle this, we first note that the density of down-sample $\tilde{X}$ can be written as $\tilde{\mu}(x) = \frac{[1-(1-\alpha)F(\tau_n+\theta_*^T x)]\mu(x)}{\mathbb{P}(Y=1)+\alpha\mathbb{P}(Y=0)}$ (see Lemma \ref{lemma:distribution-downsample-covariate} in Appendix \ref{appendix:calibration:downsampling}), where $\mu(\cdot)$ is the density of full-sample $X$. Thus the last term in (\ref{eq:likelihood-downsample}) can be rewritten as 
{$$\begin{array}{rl}
&\quad\displaystyle \log\int_{\mathcal{X}}\left[1-(1-\alpha)F(\tau_n+\theta_1^T x)\right]\mu(x)dx\\
\\
&=\displaystyle \log\int_{\mathcal{X}}[1-(1-\alpha)F(\tau_n+\theta_1^T x)]\times\\
&\quad\quad\quad\quad\displaystyle \frac{\mathbb{P}(Y=1)+\alpha\mathbb{P}(Y=0)}{1-(1-\alpha)F(\tau_n+\theta_*^T x)}\tilde{\mu}(x)dx.
\end{array}$$}
Note that when $\tau_n\rightarrow\infty$, $1-(1-\alpha)F(\tau_n+\theta_*^Tx)\rightarrow\alpha$, and $\mathbb{P}(Y=1)+\alpha\mathbb{P}(Y=0)=\mathbb{E}[1-(1-\alpha)F(\tau_n+\theta_*^TX)]\rightarrow\alpha$, thus as $n\rightarrow\infty$, we have
\begin{equation}\label{eq:approximation}
\begin{array}{rl}
\frac{\log\int_{\mathcal{X}}[1-(1-\alpha)F(\tau_n+\theta_1^T x)]\frac{\mathbb{P}(Y=1)+\alpha\mathbb{P}(Y=0)}{1-(1-\alpha)F(\tau_n+\theta_*^T x)}\tilde{\mu}(x)dx}{\log\int_{\mathcal{X}}[1-(1-\alpha)F(\tau_n+\theta_1^T x)]\tilde{\mu}(x)dx}\rightarrow1.
\end{array}
\end{equation}
Furthermore, by the law of large numbers, the sample average $\frac{1}{N}\sum_{i=1}^N[1-(1-\alpha)F(\tau_n+\theta_1^T \tilde{X}_i)]$ converges in probability to $\int_{\mathcal{X}}[1-(1-\alpha)F(\tau_n+\theta_1^T x)]\tilde{\mu}(x)dx$. 
Based upon the discovery above, we propose the following pseudo maximum likelihood estimator computed from downsample data:
\begin{equation}\label{pseudo-MLE}
\begin{array}{rl}
\hat{\theta}_*: = \argmax_{\theta_1}\frac{1}{N}\sum_{i=1}^N\tilde{\ell}(\tilde{X}_i,\tilde{Y}_i,\theta_1;\tau_n),
\end{array}
\end{equation}
where 
{\begin{equation}\label{pseudo-likelihood}
\begin{array}{rl}
&\quad\tilde{\ell}(\tilde{X}_i,\tilde{Y}_i,\theta_1;\tau_n)\\
&=\tilde{Y}_i\log(1-{F}(\tau_n+\theta_1^T \tilde{X}_i))\\
&\quad+(1-\tilde{Y}_i)\log\alpha F(\tau_n+\theta_1^T \tilde{X}_i)\\
&\quad-\log\left\{\frac{1}{N}\sum_{i=1}^N\left[1-(1-\alpha)F(\tau_n+\theta_1^T \tilde{X}_i)\right]\right\}.
\end{array}
\end{equation}}

In the following, we use $\hat{\theta}_*$ to denote the pseudo MLE computed from (\ref{pseudo-MLE}). 
\begin{remark}\label{rmk:tau_n}
Note that (\ref{eq:approximation}) holds only when $\tau_n\rightarrow\infty$, but the estimator $\hat{\theta}_*$ may be not consistent when $\tau_n<\infty$. So in practice the pseudo MLE should perform well for large values of $\tau_n$, and its benefit could disappear as $\tau_n$ decays. We verify this claim through numerical experiment on logistic regression models by comparing the performance of $\hat{\theta}_*$ under different values of $\tau_n$ in Section \ref{section:numerical}.
\end{remark}

Additionally, there are two alternative natural estimators for $\theta_*$, which are \textit{inverse-weighting estimator} and \textit{conditional maximum likelihood estimator} (a.k.a. \textit{conditional MLE}). The inverse-weighting estimator $\hat{\theta}_I$ is defined as 
\begin{equation}\label{eq:inverse-weighting}
\!\!\!\!\!%
\begin{array}{r@{}l}
\hat{\theta}^I&:=\argmax_{\theta_1}\!\frac{1}{N}\sum_{i=1}^N\tilde{Y_i}\log(1-F(\tau_n+\theta_1^T\tilde{X}_i))\!\\
&\quad\quad\quad\quad\quad\quad\quad\quad\quad+\frac{1-\tilde{Y_i}}{\alpha}\log F(\tau_n+\theta_1^T\tilde{X}_i),
\end{array}
\end{equation}
and the condtional MLE is defined as 
\begin{equation}\label{eq:conditional-MLE}
\!\!\!\!\!\!%
\begin{array}{r@{}l}
\hat{\theta}^C&:=\argmax_{\theta_1}\!\frac{1}{N}\sum_{i=1}^N\tilde{Y_i}\log(1-G(\tau_n+\theta_1^T\tilde{X}_i))\!\\
&\quad\quad\quad\quad\quad\quad\quad\quad+(1-\tilde{Y_i})\log G(\tau_n+\theta_1^T\tilde{X}_i),
\end{array}
\end{equation}
where $G(\cdot)$ is defined as in (\ref{eq:G}). By the theory of M-estimators \cite{van2000asymptotic}, both of the estimators are consistent, i.e. $\hat{\theta}^I\overset{p}{\rightarrow}\theta_*$ and $\hat{\theta}^C\overset{p}{\rightarrow}\theta_*$. We compare the performance of our proposed pseudo MLE $\hat{\theta}_*$ as (\ref{pseudo-MLE}) with both $\hat{\theta}^I$ and $\hat{\theta}^C$ through numerical experiments on synthetic data (see Section \ref{section:numerical}) and empirical data (see Section \ref{sec:empirical} and Appendix \ref{appendix:additional simulation}). We find that $\hat{\theta}_*$ outperforms $\hat{\theta}^I$ and $\hat{\theta}^C$ for large values of $\tau_n$, but it gradually under-performs as we decrease the value $\tau_n$, which verifies our conjecture as in remark \ref{rmk:tau_n}, implying its value under rare-event setup.

\section{Asymptotic Normality}\label{sec:asymptotic}
In this section, we discover that for the rare event setup, there is a certain range of downsample rate $\alpha\ll1$, as long as it doesn't go to zero too fast, downsampling with rate $\alpha$ maintains the statistical efficiency as that of the full-sampling estimator. To demonstrate our statement, we consider the case where {\footnotesize$\mathbb{P}(Y=1|X=x)=1-F(\tau_n+\theta_*^T X)\ \mbox{with}\ \  \tau_n\rightarrow\infty.$} 

\paragraph{Failure of Classical Maximum Likelihood Estimator Analysis} Given the definition of $\tilde{\ell}(\cdot,\cdot,\cdot;\tau_n)$ above, we note that {\footnotesize$\mathbb{E}\left[\tilde{\ell}(\tilde{X},\tilde{Y},\theta_*;\tau_n)\right]\rightarrow0$} regardless of $\theta_*$ if $\alpha>0$. Thus as $n\rightarrow\infty$, the criterion function doesn't rely on the value of $\theta_1$ in the varying-rate regime, hence the classical MLE theory cannot be applied here. A more detailed discussion is presented in the Appendix \ref{section:varying-rate-regime:appendix}. We now use an alternative way to derive the asymptotic normality of $\hat{\theta}_{*}$. 

\begin{assumption}\label{ass:F:regularity:basic}
    $F\in C^1(\mathbb{R})$ is strictly increasing, and the covariate space $\mathcal{X}\subset\mathbb{R}^d$ is compact. 
\end{assumption}

\begin{assumption}\label{ass:regularity}
The matrix $\mathbb{E}[XX^T]$ is nonsingular.
\end{assumption}

\begin{assumption}\label{ass:F-regularity}
$F$ is thrice differentiable and there exists a thrice differentiable function $h:\mathbb{R}\rightarrow\mathbb{R}_{+}$ such that the function {\footnotesize$\bar{F}(Z>\cdot+\tau_n)/\bar{F}(Z>\tau_n)$}
together with its first three derivatives converges uniformly over compact sets to $h(.)$ and its three first derivatives denoted as $h^{(1)}, h^{(2)}, h^{(3)}$, respectively.
\end{assumption}

The distributions that satisfy Assumption \ref{ass:F-regularity} include exponential distribution, logistic regression model, and etc. We now provide our first main result.

\begin{theorem}[Asymptotic Normality of MLE as $\tau_n\rightarrow\infty$]\label{thm:rate-of-convergence-infinity}
Suppose Assumptions \ref{ass:F:regularity:basic}, \ref{ass:regularity}, \ref{ass:F-regularity} hold, and for some $\Theta\subset\mathbb{R}^d$ as a neighborhood of $\theta_*$, $\{F(\tau_n+\theta_1^T x):\theta_1\in\Theta\}$ is differentiable in quadratic mean at $\theta_*$. Assume that $n(1-F(\tau_n))\rightarrow\infty$, {\footnotesize$\lim_{n\rightarrow\infty}\frac{(1-\alpha)^2(1-F(\tau_n))}{\alpha}=c,$} and 
{\begin{equation}\label{eq:positive-definite}
\begin{array}{rl}
&\quad\mathbb{E}\left[\frac{h^{(1)}(\theta_*^T X)^2}{h(\theta_*^T X)}XX^T\right]\\
&\succ c\mathbb{E}\left[h^{(1)}(\theta_*^T X)X\right]\mathbb{E}\left[h^{(1)}(\theta_*^T X)X^T\right],
\end{array}
\end{equation}}
then as $\tau_n\rightarrow\infty$, we have {\footnotesize$\sqrt{n(1-F(\tau_n))}(\hat{\theta}_{*}-\theta_*)\overset{d}{\rightarrow}\mathcal{N}\left(\mathbf{0},\mathbf{V}^{-1}\right),$}
where 
$$\begin{array}{rl}
\mathbf{V}&=\mathbb{E}\left[\frac{h^{(1)}(\theta_*^T X)^2}{h(\theta_*^T X)}XX^T\right]\\
&\quad-c\mathbb{E}\left[h^{(1)}(\theta_*^T X)X\right]\mathbb{E}\left[h^{(1)}(\theta_*^T X)X^T\right].
\end{array}$$
\end{theorem}

\begin{remark}
Theorem \ref{thm:rate-of-convergence-infinity} suggests that if $$\lim_{n\rightarrow\infty}\frac{(1-\alpha^2)(1-F(\tau_n))}{\alpha}=0,$$ then the resulting estimator is as efficient as the full-sample estimator (i.e. when $\alpha=1$). This happens when either $\alpha$ is bounded from below by a positive constant, or $\alpha\rightarrow0$ but $\frac{1-F(\tau_n)}{\alpha}\rightarrow0$,
which is consistent with the discovery of Remark 2 in \cite{wang2020logistic} for logistic regression under rare event case. 
\end{remark}

\begin{remark}\label{rmk:asymptotic-condition}
We demonstrate the necessity of condition (\ref{eq:positive-definite}) through a numerical illustration. 
\begin{figure}
    \centering
    \includegraphics[width=0.7\linewidth]{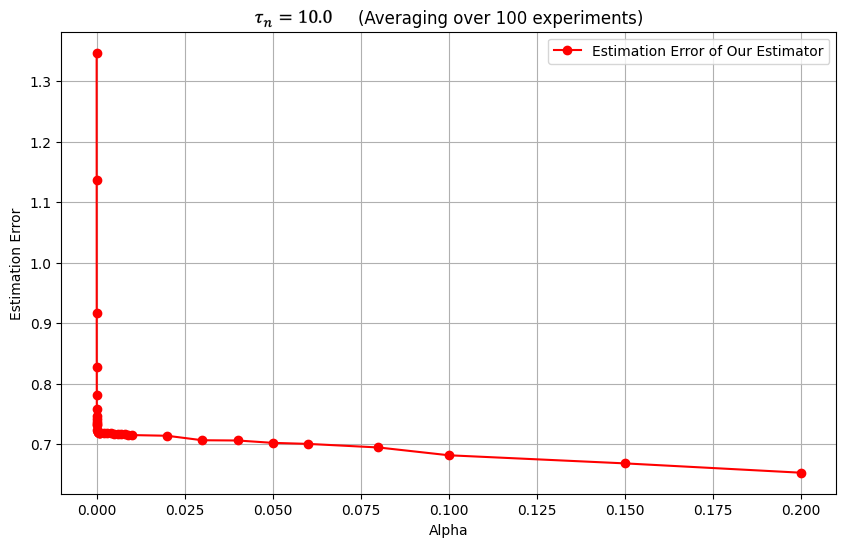}
    \caption{Estimation error for different $\alpha$.}
    \label{fig:illustration:condition:positive}
\end{figure}
Imagine that when $\alpha$ is too close to zero, the right hand side of (\ref{eq:positive-definite}) can blow up so the condition will be violated. We look at a case of logistic regression where $\tau_n=10$, $\theta_*=0.5$, and $F(\tau_n+\theta_*^Tx)=\frac{e^{\tau_n+\theta_*^Tx}}{1+e^{\tau_n+\theta_*^Tx}}$ and $X\sim\mathrm{Unif}[0,1]$. We generate $10^5$ random samples and $\mathbb{P}(Y=1)\approx 5.89\times10^{-5}$. The estimation error is computed as the average value of $|\theta_*-\hat{\theta}_{*}|$ from $500$ random experiments.  The $x$-axis corresponds to the $\alpha$ as the downsample rate from negative samples. 
Figure \ref{fig:illustration:condition:positive} shows that when $\alpha$ is too close to $0$, the estimation error is large. But when $\alpha$ falls in a proper range such that condition (\ref{eq:positive-definite}) holds, the estimation error for $\hat{\theta}_{*}$ stabilizes and is kept as a small value, consistent with the theoretical findings that the mean squared error should be small and $\mathbb{E}[\|\hat{\theta}_*-\theta_*\|^2]\approx\mathrm{tr}[\mathbf{V}^{-1}]/(n(1-F(\tau_n)))$.
\end{remark}

\paragraph{Generalized Scaled Asymptotic Normality} Now we generalize Assumption \ref{ass:F-regularity} to Assumption \ref{ass:generalized:asymptotic}, which is satisfied by more distributions in addition to those satisfying Assumption \ref{ass:F-regularity}, including both heavy-tailed distributions (e.g. Pareto) and very light tailed distributions (e.g. Gaussian tails). 
\begin{assumption}\label{ass:generalized:asymptotic}
Suppose that $F$ is thrice differentiable and that there exists a thrice differentiable function $g:\mathbb{R}\rightarrow\mathbb{R}_{+}$ such that the function {\footnotesize$\bar{F}(Z>r(\tau_n)\cdot+\tau_n)/\bar{F}(Z>\tau_n)$}
together with its first three derivatives converges uniformly over compact sets to $g(.)$ and its three first derivatives denoted as $g^{(1)}, g^{(2)}, g^{(3)}$, respectively.

\end{assumption}
Assumption \ref{ass:generalized:asymptotic} includes more common distributions except for those satisfying Assumption \ref{ass:F-regularity}. For example, for the standard normal distribution, we can take $r(\tau)=\frac{1}{\tau}$ and $g(z)=e^{-z}$. For the Pareto distribution, we can take $r(\tau)=\tau$ and $g(z)=\left(\frac{1}{1+z}\right)^{\gamma}$ with $\gamma>1$, to satisfy the assumption but in this case we must assume $X$ has compact support so that we may assume (after a rescaling) that $\theta_*^TX$ is less than 1, this will ensure that $z\dot r(\tau_n) + \tau_n > 0$ when evaluating at $z=\theta_*^Tx$ and keep the limits well defined.  \ref{ass:generalized:asymptotic}. Now we can extend our result of the asymptotic distribution to more general cases. 

\begin{theorem}[Generalized Scaled Asymptotic Normality]\label{thm:generalize:asymptotics}
Suppose the underlying binary classification model is defined as {\footnotesize$Y=\mathbf{1}(\tau_n+r(\tau_n)\theta_*^T X)>0$} with $r(\cdot)$ defined as in Assumption \ref{ass:generalized:asymptotic}. Assume that $n(1-F(\tau_n))\rightarrow\infty$,{\footnotesize$\lim_{n\rightarrow\infty}\frac{(1-\alpha)^2(1-F(\tau_n))}{\alpha}=c,$} and 
{\begin{equation}\label{eq:positive-definite:generalized}
\begin{array}{rl}
&\quad\mathbb{E}\left[\frac{g^{(1)}(\theta_*^T X)^2}{g(\theta_*^T X)}XX^T\right]\\
&\succ c\mathbb{E}\left[g^{(1)}(\theta_*^T X)X\right]\mathbb{E}\left[g^{(1)}(\theta_*^T X)X^T\right].
\end{array}
\end{equation}}
Then as $\tau_n\rightarrow\infty$, we have 
$$\sqrt{n(1-F(\tau_n))}r(\tau_n)(\hat{\theta}_{*}-\theta_*)\overset{d}{\rightarrow}\mathcal{N}\left(\mathbf{0},\mathbf{V}^{-1}\right),$$
where 
\vspace{-0.3cm}
{$$\begin{array}{rl}
\displaystyle\mathbf{V}&=\mathbb{E}\left[\frac{g^{(1)}(\theta_*^T X)^2}{g(\theta_*^T X)}XX^T\right]\\
&\quad\displaystyle-c\mathbb{E}\left[g^{(1)}(\theta_*^T X)X\right]\mathbb{E}\left[g^{(1)}(\theta_*^T X)X^T\right].
\end{array}$$}
\end{theorem}
\vspace{-0.3cm}
Based on the asymptotic covariances, we are now ready to formulate the optimization problem by considering both statistical efficiency and computational gains. We now introduce the efficiency concept with a computational budget constraint. 

\section{Efficiency with a Budget Constraint}\label{section:efficiency}
Theorem \ref{thm:rate-of-convergence-infinity} and Theorem \ref{thm:generalize:asymptotics} indicate that full-sampling (i.e. $\alpha=1$) and downsampling with a rate $\alpha$ such that $\alpha\ll1$ and $\frac{1-F(\tau_n)}{\alpha}=\mathrm{o}(1)$ lead to the same rate of convergence and asymptotic MSE. At the same time, if we choose $\alpha\approx\mathbb{P}(Y=1)$, we still have the same convergence rate while the trace norm of the asymptotic covariance is kept as $\mathrm{O}(1)$. Intuitively, downsampling with a much smaller $\alpha$ can help to reduce computational cost significantly while maintaining statistical efficiency under the current rare event setup. This motivates us to formulate the choice of optimal downsampling rate by considering this tradeoff between statistical efficiency and computational cost with a budget constraint. 

We adapt the framework of \cite{glynn1992asymptotic} for the definition of algorithm efficiency. Let the computational cost be some strictly increasing function of the downsample size, i.e. {\footnotesize$C(\alpha;n,p_1)=f(np_1+\alpha(1-p_1)n)=f(n[p_1+\alpha(1-p_1)])$}, where $f(\cdot)$ is strictly increasing, $p_1=\mathbb{P}(Y=1)$ and $C(\alpha;n,p_1)$ is the cost function of sampling $\alpha$ proportion of the negative samples (i.e. observations with label $Y=0$) with the original sample size $n$ and positive ratio $p_1$. According to the asymptotic efficiency principle in the canonical case of \cite{glynn1992asymptotic}, we define the \textit{asymptotic algorithm efficiency cost value} as the product of the sampling variance and the cost rate, i.e. 
{\footnotesize$
v(\alpha;p_1)=\lim_{n\rightarrow\infty}C(\alpha;n,p_1)\mathrm{Var}\left(\hat{\theta}_{*}\right).$}
Specifically, when $f(\cdot)$ is a linear function, i.e. with some constant $c_0$, we have 
{\begin{equation}\label{eq:algorithm-efficiency-cost}
v(\alpha;p_1)=\lim_{n\rightarrow\infty}c_0n\left(p_1+\alpha(1-p_1)\right)\mathrm{Var}\left(\hat{\theta}_{*}\right).
\end{equation}}

Recall from Theorem \ref{thm:rate-of-convergence-infinity} that under regular conditions, when $\tau_n\rightarrow\infty$, we have {\footnotesize$\sqrt{n(1-F(\tau_n))}(\hat{\theta}_{*}-\theta_*)\overset{d}{\rightarrow}\mathcal{N}\left(\mathbf{0},\mathbf{V}^{-1}\right).$} So 
{\footnotesize$\frac{n(1-F(\tau_n))\mathrm{tr}\left[\mathrm{Cov}(\hat{\theta}_{*}-\theta_*)\right]}{\mathrm{tr}\left(\mathbf{V}^{-1}\right)}\rightarrow1\ \ \mbox{as}\ n\rightarrow\infty.$} Under the definition of the efficiency cost (\ref{eq:algorithm-efficiency-cost}), we have 
{\footnotesize$\frac{c_0(p_1+\alpha(1-p_1))d^2}{\mathrm{tr}\left((1-F(\tau_n))\mathbf{V}\right)}\leq\nu(\alpha;p_1)\leq\frac{c_0(p_1+\alpha(1-p_1))\kappa d^2}{\mathrm{tr}\left((1-F(\tau_n))\mathbf{V}\right)},$}
where $\kappa$ is the condition number of {\footnotesize$\mathbf{V}$}. 
A natural objective for efficiency optimization is 
{\footnotesize$\min_{\alpha\in[0,1]}\lim_{n\rightarrow\infty}\frac{p_1+\alpha(1-p_1)}{\mathrm{tr}\left(\mathbf{V}\right)}.$}

\begin{theorem}[Optimal Downsampling Rate for Imbalanced Classification]\label{thm:asymptotic:efficiency-cost}
Suppose Assumptions \ref{ass:F:regularity:basic}, \ref{ass:regularity}, \ref{ass:F-regularity}, and $\tau_n\rightarrow\infty$, $n(1-F(\tau_n))\rightarrow\infty$, then the optimal choice of downsampling rate is 
{\footnotesize$$\alpha^*=\frac{2(1-F(\tau_n))\mathrm{tr}\left\{\mathbb{E}\left[h^{(1)}(\theta_*^T X)X\right]\mathbb{E}\left[h^{(1)}(\theta_*^T X)X^T\right]\right\}}{\mathrm{tr}\left\{\mathbb{E}\left[\left\{h^{(1)}(\theta_*^T X)^2/h(\theta_*^T X)\right\}XX^T\right]\right\}}.$$}
\end{theorem}

Henceforth, we have obtained explicit formulations of the optimization problems for selecting the optimal downsample rate. Theorem \ref{thm:asymptotic:efficiency-cost} has provided guidelines for downsampling schemes in practice. The result can easily be extended to generalized scaled result from Theorem \ref{thm:generalize:asymptotics}.

\section{Application to Logistic Regression}\label{sec:application:logistic regression}

Equipped with all the findings from the previous sections, we now focus on the logistic regression model as an application. Given a sequence of $\tau_n\rightarrow\infty$, define {\footnotesize$F(\tau_n+\theta_1^T x):=\frac{e^{\tau_n+\theta_1^T x}}{1+e^{\tau_n+\theta_1^T x}}$.} A direct application of Theorem \ref{thm:rate-of-convergence-infinity} indicates:

\begin{proposition}[Asymptotic Normality for Logistic Regression]\label{prop:logistic-regression:asymptotic} Assume $\frac{n}{1+e^{\tau_n}}\rightarrow\infty$ and $\lim_{n\rightarrow\infty}\frac{(1-\alpha)^2}{\alpha(1+e^{\tau_n})}=c$ where $c$ is a constant. Then as $\tau_n\rightarrow\infty$ and , we have 
{$$\sqrt{n\mathbb{P}(Y=1)}\left(\hat{\theta}_{*}-\theta_*\right)\overset{d}{\rightarrow}\mathcal{N}\left(\mathbf{0},\mathbb{E}_X\left[e^{-\theta_*^T X}\right]\mathbf{V}^{-1}\right),$$}
where     
{\footnotesize$\mathbf{V}=\mathbb{E}\left[e^{-\theta_*^T X}XX^T\right]-c\mathbb{E}\left[e^{-\theta_*^T X}X\right]\mathbb{E}\left[e^{-\theta_*^T X}X\right].$}
\end{proposition}
\begin{remark}
If either $\alpha$ is bounded away from zero, i.e., $\exists$ some absolute constant $\underline{\alpha}\in(0,1)$ such that $0<\underline{\alpha}\leq\alpha\leq1$, or $\alpha\rightarrow0$ and $\frac{1}{\alpha(1+e^{\tau_n})}\rightarrow0$, then $c=0$, and the asymptotic distribution is the same as that of the full-sampling case, so the estimator is as efficient as the full-sampling one. This is also consistent with the findings of \cite{wang2020logistic}. 
\end{remark}
\begin{remark}
Our estimator is different from \cite{wang2020logistic}, where \cite{wang2020logistic} uses an inverse-weighting estimator for the inverse-weighting estimator (\ref{eq:inverse-weighting}).
We illustrate the differences in the performance of our estimator and that of the estimator considered by \cite{wang2020logistic} further through numerical experiments in Section \ref{section:numerical}. 
\end{remark}

\begin{proposition}[Optimal Downsampling Rate for Logistic Regression]\label{prop:efficiency:asymptotic}
Suppose $\tau_n\rightarrow\infty$, $\frac{n}{1+e^{\tau_n}}\rightarrow\infty$, the optimal choice of downsampling rate is 
{$$\begin{array}{rl}
\displaystyle\alpha^*=\frac{2(1+e^{\tau_n})^{-1}\mathrm{tr}\left\{\mathbb{E}[e^{-\theta_*^T X}X]\mathbb{E}[e^{-\theta_*^T X}X]\right\}}{\mathrm{tr}\{\mathbb{E}[e^{-\theta_*^T X}XX^T]\}}.
\end{array}$$}
\end{proposition}

\subsection{Numerical Experiments}\label{section:numerical}
We focus on a setting where $\theta_*=0.5$, and the covariates $X$ are drawn i.i.d. from a uniform distribution $[0,1]$. The sample size is $n=10^5$. We fix different values of $\tau_n$ and estimate $\theta_*$ for these different $\tau_n$. Under each $\alpha$, we compute the solutions to (\ref{eq:identification}), i.e., maximum likelihood estimators, under $500$ random environments, and we also compute the mean-squared-error (MSE) for $\theta_*$ with respect to each $\alpha$ by averaging over the $500$ environments. 

Firstly, we compare the mean squared estimation error of our estimator and the inverse-weighting estimator \cite{wang2020logistic} for $\tau_n=6,7,8,9$ for some of the $\alpha\in[0.00005,0.5]$. These values $\tau_n$ correspond to $\mathbb{P}(Y=1)$ approximately equal to $0.002,0.0007,0.0002,0.000097$. The numerical results shown by Figure  \ref{fig:convergence:comparison} (in Appendix \ref{appendix:additional simulation}) are consistent with our findings: when $\alpha$ is small and falls into the proper range satisfying the conditions of Theorem \ref{thm:rate-of-convergence-infinity}, the mean-squared-error is close to that generated by $\alpha=0.5$.

Secondly, We focus on the range of $\alpha$ very close to $\mathbb{P}(Y=1)$ for $\tau_n=10.0,9.8,6.0,5.0$, we compute the mean squared error for logistic regression, which correspond to the cases where $\mathbb{P}(Y=1)$ is approximately equal to $3.57\times10^{-5}$, $4.36\times10^{-5}$, $0.0019$, $0.0053$ respectively. 
We replicate our simulations $500$ times for each $\tau_n$ by comparing inverse-weighting estimator, conditional maximum likelihood estimator and our proposed pseudo MLE. 

From Figure \ref{fig:logit-experiment-2} we see that for $\tau_n=10,9.8$ (large), our proposed estimator outperforms both the inverse-weighting estimator and the conditional maximum likelihood estimator. This verifies our statement in Remark \ref{rmk:tau_n}. However for $\tau_n=6,5$ (small), our proposed estimator is worse than the other two. Note that our estimator is not consistent when $\tau_n$ is small by remark \ref{rmk:tau_n}, so its under-performance is not surprising in this scenario. 

Finally, we plot the efficiency cost in Figure \ref{fig:efficiency-cost} (in Appendix \ref{appendix:additional simulation}) with computational budget constraint according to our definition previously and the tradeoff is depicted numerically for the downsample rate selection as we discussed in Section \ref{section:efficiency}. We have observed that when $\tau_n$ increases, the efficiency cost function becomes sharper, indicating it's more sensitive to the choice of $\alpha$. Although our theoretical finding shows that under the rare event case, a small choice of $\alpha$ such that $(1-F(\tau_n))/\alpha=o(1)$ doesn't bring in information loss while reducing computational cost massively, the efficiency cost can be very sensitive to $\alpha$ as the positive ratio goes to $0$ according to Figure \ref{fig:efficiency-cost}, suggesting the necessity of a more prudent method of downsampling rate selection.

\section{Empirical Performance}\label{sec:empirical}

To verify the performance of the proposed pseudo MLE on real imbalanced data, we compare the performance of our estimator with the inverse-weighting estimator on \href{https://imbalanced-learn.org/dev/datasets/index.html}{UCI imbalanced datasets} in Table \ref{tab:dataset:UCI}.  

\begin{table}[htp!]
\centering
\resizebox{\textwidth}{!}{
\begin{tabular}{|c||c||c||c|}
    \hline
    \textbf{Dataset} & \textbf{Sample Size} & \textbf{Neg:Pos Ratio} & \textbf{Feature Number} \\
    \hline	
    \texttt{abalone\_19} & 4,177 & 130:1 & 10 \\
    \hline
    \texttt{mammography} & 11,183 & 42:1 & 6 \\
    \hline
    \texttt{yeast\_me2} & 1,484 & 28:1 & 8 \\
    \hline
    	
    \texttt{abalone} & 4,177 & 9.7:1  & 10 \\
    \hline
    \texttt{ecoli} & 336 & 8.6:1 & 7 \\
    \hline
\end{tabular}
}
\caption{Summary of UCI imbalanced datasets : sample sizes, negative:positive ratios, feature numbers.}
\label{tab:dataset:UCI}
\end{table}

In order to adapt to our setup where we consider the regime with $\tau_n\rightarrow\infty$, we set $\tau_n$ such that $1/(1+e^{\tau_n})=p_1$, where the values of $p_1$ are the positive ratio in the imbalanced dataset. Then we use inverse-weighting estimator and our pseudo-MLE estimator to fit $\theta_*$ (the coefficient for the features), and we compute the log-losses for both estimators on the testing dataset. We replicate the experiment $500$ times, and during each round we randomly split the dataset into $80\%$ for training and $20\%$ for testing. We then plot the average log-losses and the confidence intervals for the log-losses for each downsampling rate $\alpha$ in Figure \ref{fig:real-data-1}, where these $\alpha$'s are chosen close to $p_1$ (i.e. positive ratio) of each dataset. We refer the readers to Appendix \ref{sec:logistic:additional:appendix} for performance with additional moderate and small values of $\tau_n$. 

Lastly, we also apply our method to neural networks on some of those datasets. The simulation details and insights are presented in Appendix \ref{sec:NN:additional:appendix}. 

The numerical results suggest that the application of the pseudo maximum likelihood estimator reduces the out-of-sample log-loss errors compared to the commonly used inverse-weighting estimators in practice. 

\begin{figure}[htp!]
\begin{subfigure}{0.8\textwidth}
     \includegraphics[width=\textwidth]{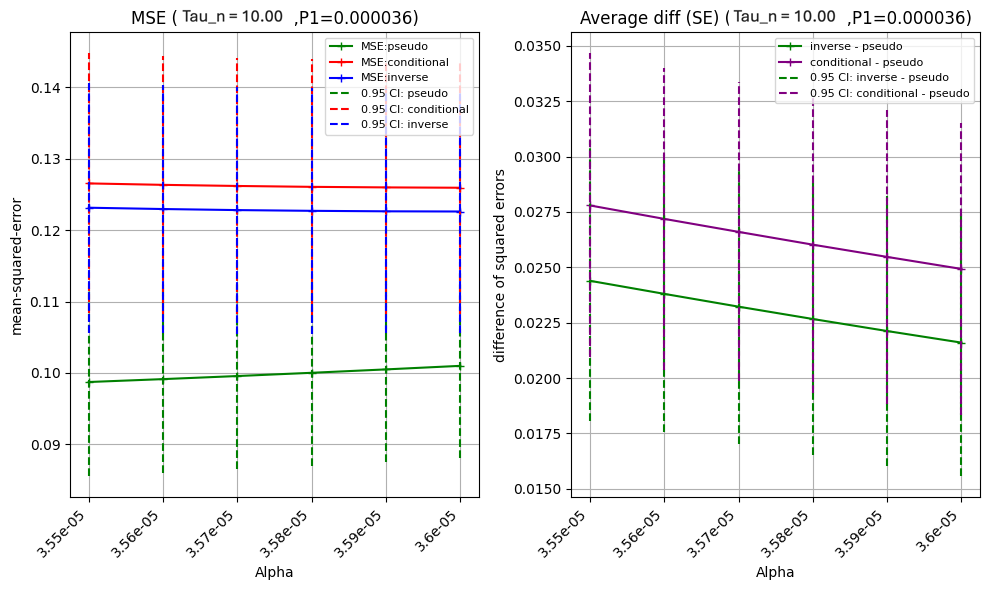}
     \caption{MSE:$\tau_n=10.0$}
     \label{fig:tau_n=10}
 \end{subfigure}
  \vfill
 \begin{subfigure}{0.8\textwidth}
     \includegraphics[width=\textwidth]{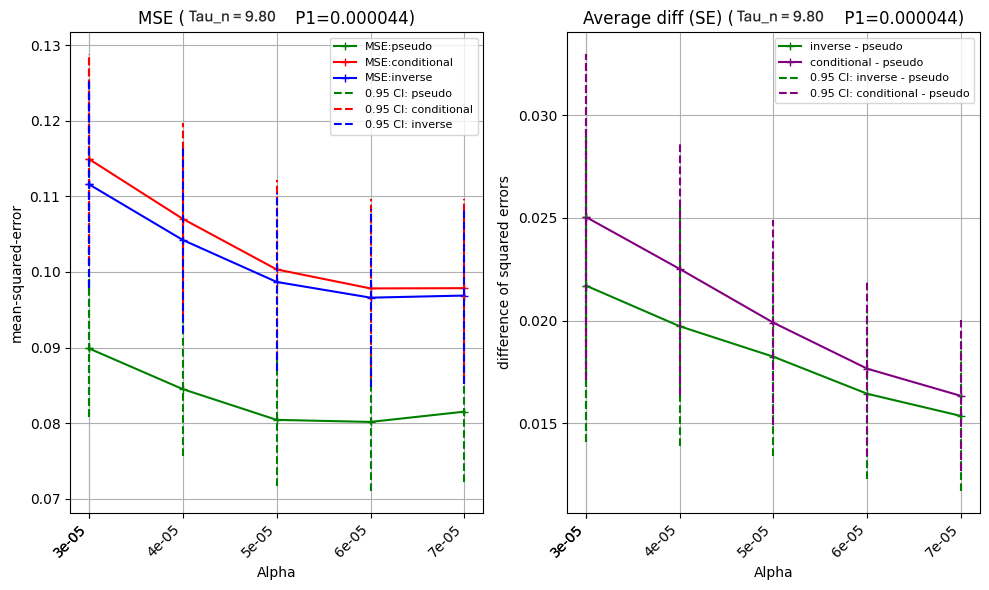}
     \caption{MSE:$\tau_n=9.8$}
     \label{fig:tau_n=9.8} 
 \end{subfigure} 
 \vfill
 \begin{subfigure}{0.8\textwidth}
     \includegraphics[width=\textwidth]{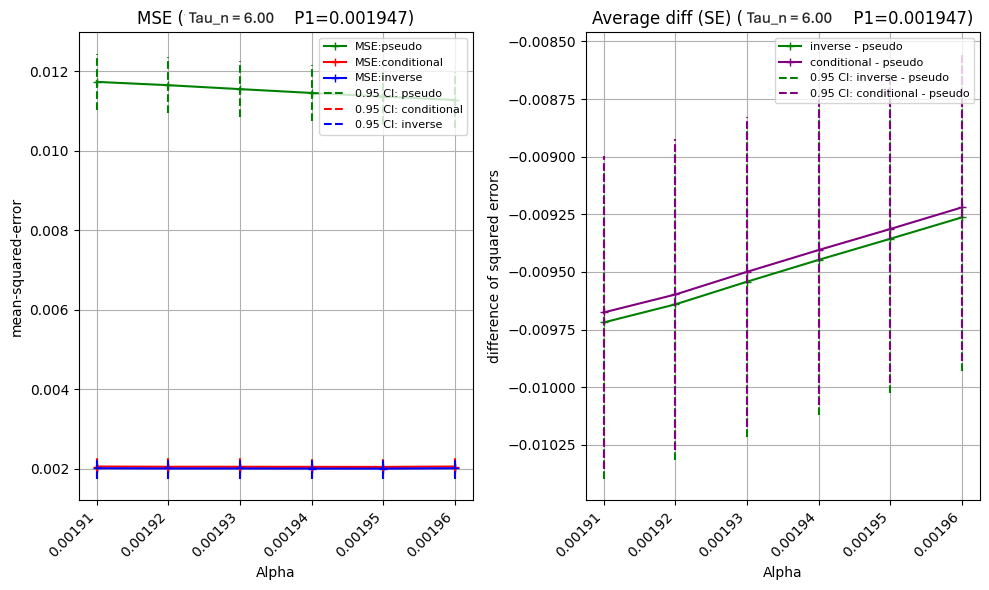}
     \caption{MSE:$\tau_n=6.0$}
     \label{fig:tau_n=6}
 \end{subfigure}
 \vfill
 \begin{subfigure}{0.8\textwidth}
     \includegraphics[width=\textwidth]{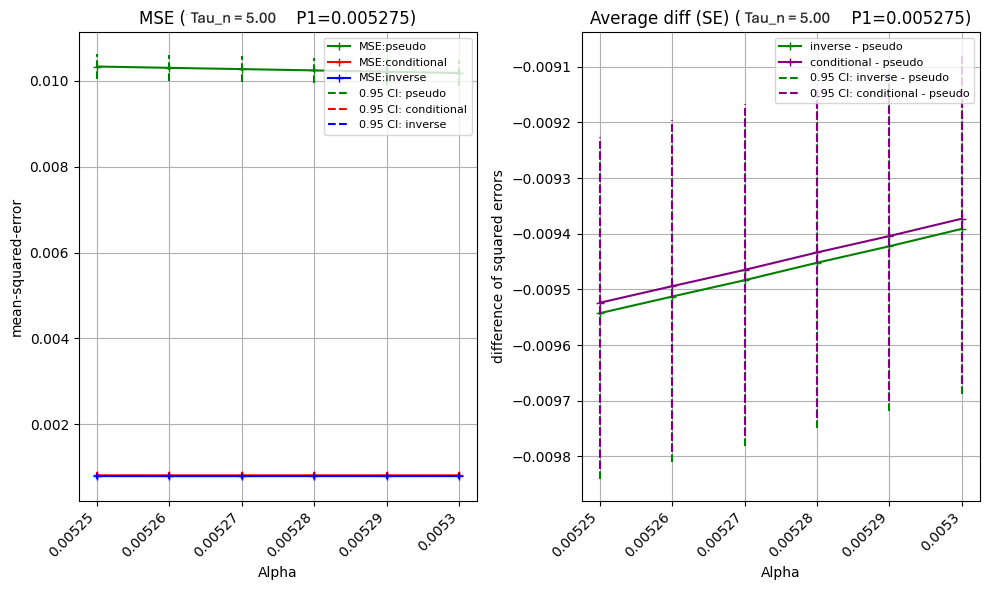}
     \caption{MSE:$\tau_n=5.0$}
     \label{fig:tau_n=5}
 \end{subfigure}
 \caption{On the left panel of each figure, we plot MSE of inverse-weighting (blue) vs. pseudo-MLE (green) vs. conditional MLE (red) for $\alpha$ chosen around $\mathbb{P}(Y=1)$ for $\tau_n=10.0,9.8,6.0,5.0$ with Logistic Regression. The blue,green,red dashed lines correspond to the $95\%$ confidence intervals for the squared losses of inverse-weighting estimator, pseudo MLE and conditional MLE. The upper and lower ends are computed by $\pm1.96*\frac{\hat{\sigma}}{\sqrt{500}}$ and $\hat{\sigma}$ is the standard deviation of squared losses at each alpha computed over $500$ random environments. On the right panel of each figure the green solid lines correspond to the average squared loss differences between inverse-weighting estimator and pseudo MLE, and the purple one corresponds to that of conditional MLE minus pseudo MLE. And the dash lines are the $95\%$ confidence intervals. }
 \label{fig:logit-experiment-2}
\end{figure}

\section{Discussion}
We propose a pseudo maximum likelihood estimator for a Generalized Linear Model binary classifier under downsampling, with theoretical convergence guarantees for imbalanced data. We propose an efficiency cost notion to guide downsampling rate selection. For future work we are interested in exploring over-sampling and focusing on other performance metrics except for mean-squared-errors.

\newpage

\vskip 0.2in

\newpage

\onecolumn

\appendix
\begin{center} {\bf \Large Supplementary Material} \end{center}

\section{Additional Simulation Details}\label{appendix:additional simulation}

\subsection{Additional Results for Applying Logistic Regression}\label{sec:logistic:additional:appendix}
\paragraph{Varying values of $\tau_n$ for Logistic Regression on Empirical Data} We run additional simulations on \texttt{abalone\_19} and \texttt{yeast\_me2} data for varying values of $\tau_n$. The previous setting of \texttt{abalone\_19} data is $\tau_n=\log(1/p_1-1)=4.86$ and we plot the results with $\tau_n=0.01,0.5,1.0,2.0,3.0$ in Figure \ref{fig:abalone:additional tau_n}. The previous setting of \texttt{yeast\_me2} is $\tau_n=\log(1/p_1-1)=3.34$ and we plot the results with $\tau_n=0.01,0.5,1.0,2.0,2.5$ in Figure \ref{fig:yeast_me2:additional tau_n}. The plots show that our pseudo MLE estimator outperforms the inverse-weighting estimator even for those moderate and small values of $\tau_n$. 

\begin{figure}[htp!]
\begin{subfigure}{0.48\textwidth}
     \includegraphics[width=\textwidth]{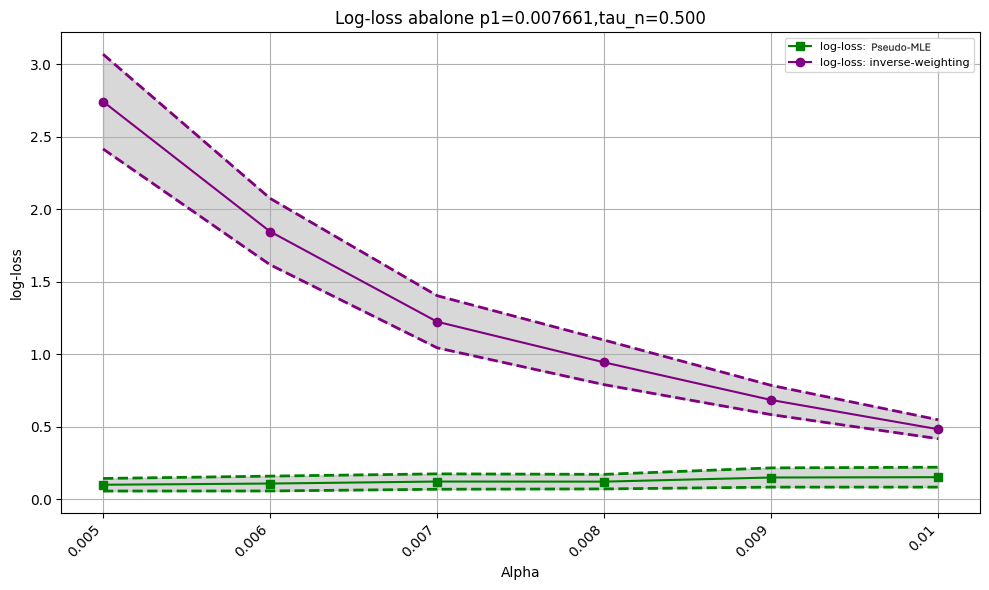}
     \caption{log-loss:$\tau_n=0.5$}
     \label{fig:tau_n=0.5:abalone}
 \end{subfigure}
  \hfill
 \begin{subfigure}{0.48\textwidth}
     \includegraphics[width=\textwidth]{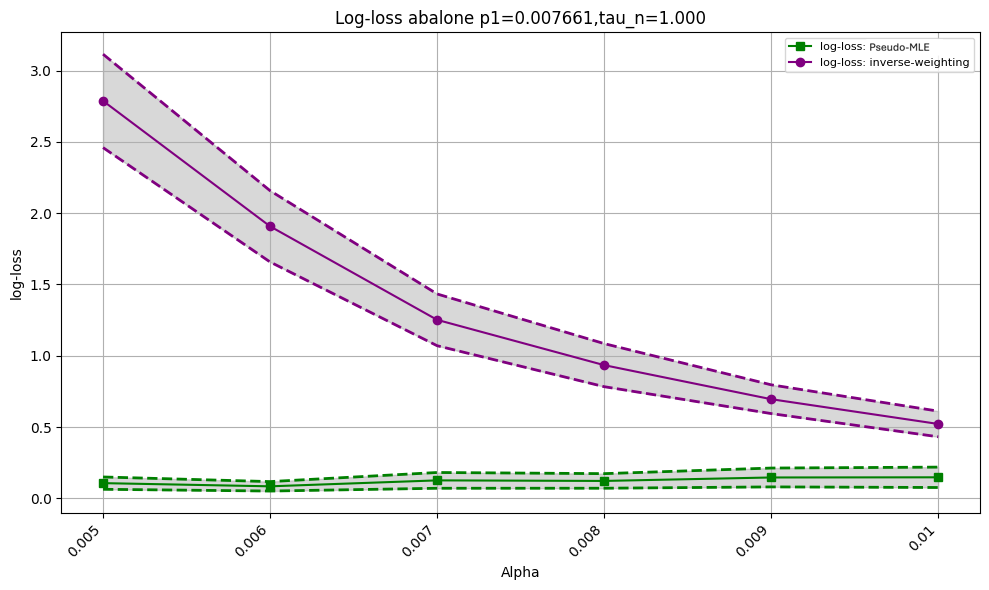}
     \caption{log-loss:$\tau_n=1.0$}
     \label{fig:tau_n=1:abalone} 
 \end{subfigure} 
 \vfill
 \begin{subfigure}{0.48\textwidth}
     \includegraphics[width=\textwidth]{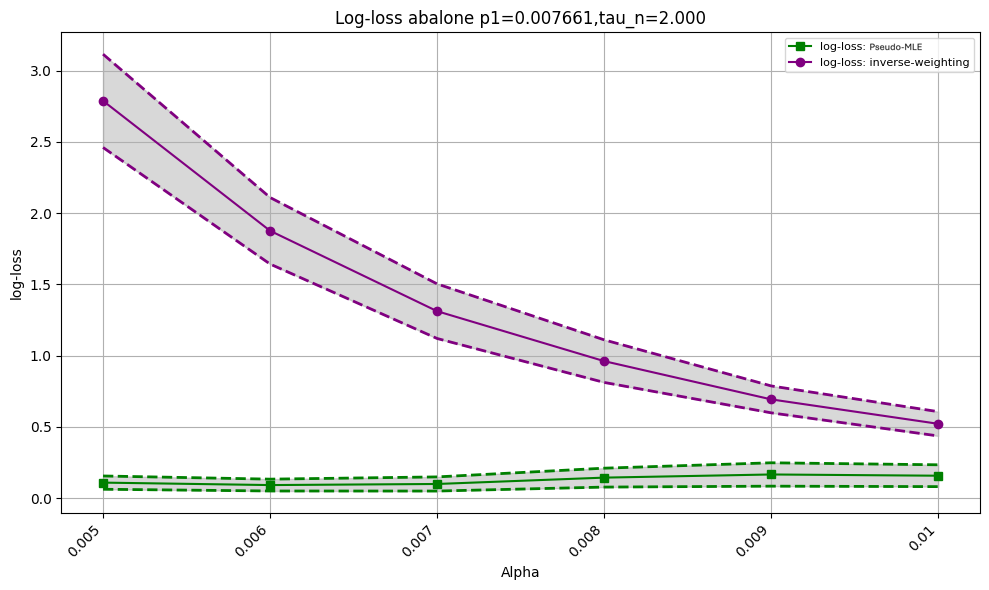}
     \caption{log-loss:$\tau_n=2.0$}
     \label{fig:tau_n=2:abalone}
 \end{subfigure}
 \hfill
 \begin{subfigure}{0.48\textwidth}
     \includegraphics[width=\textwidth]{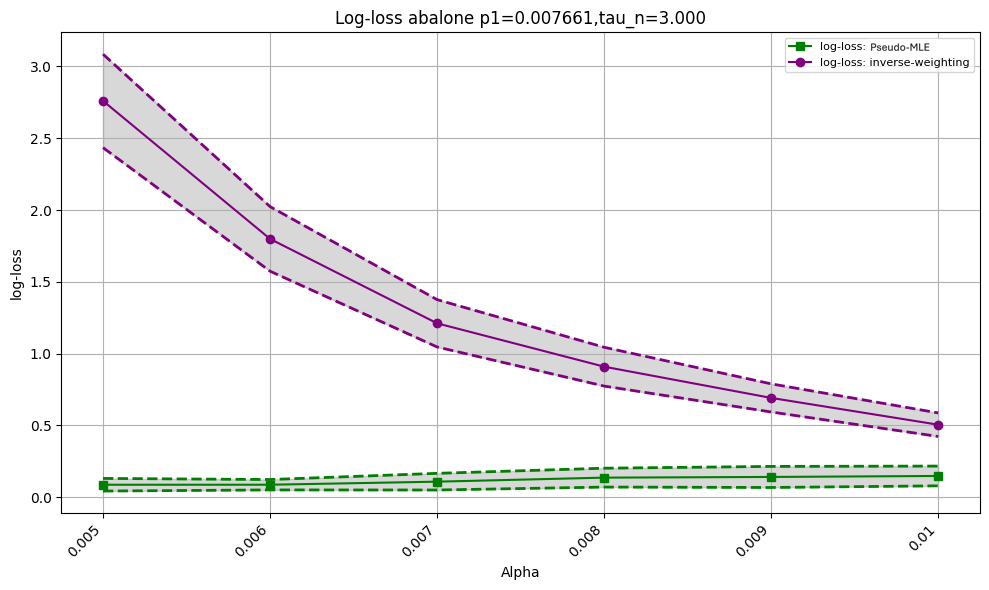}
     \caption{log-loss:$\tau_n=3.0$}
     \label{fig:tau_n=3:abalone}
 \end{subfigure}
 \hfill
 \begin{subfigure}{0.48\textwidth}
     \includegraphics[width=\textwidth]{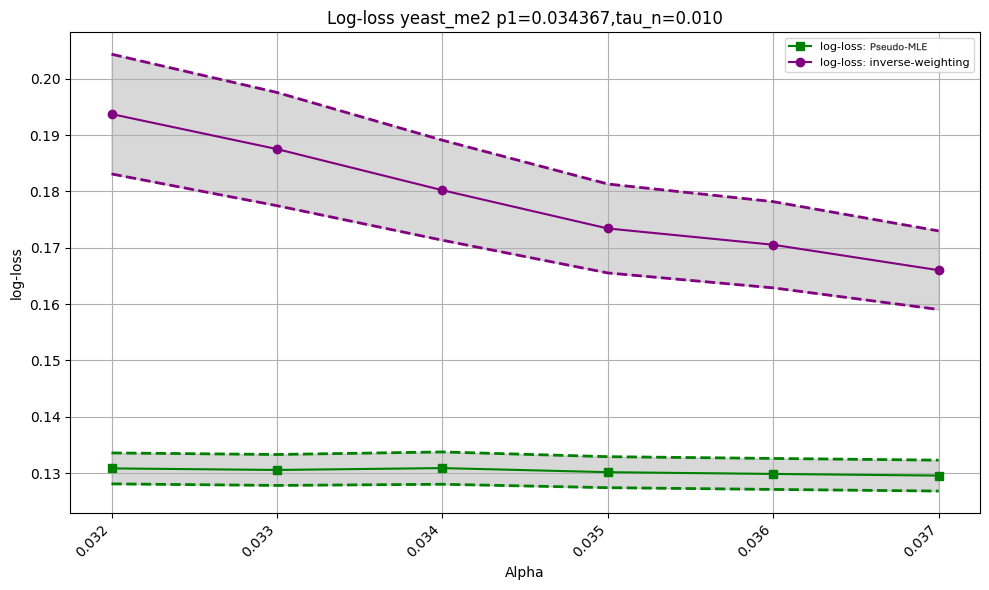}
     \caption{log-loss:$\tau_n=0.01$}
     \label{fig:tau_n=0.01:abalone}
 \end{subfigure}
 \caption{Additional results for \texttt{abalone\_19} dataset for small and moderate values of $\tau_n$.}
 \label{fig:abalone:additional tau_n}
\end{figure}

\begin{figure}[htp!]
\begin{subfigure}{0.48\textwidth}
     \includegraphics[width=\textwidth]{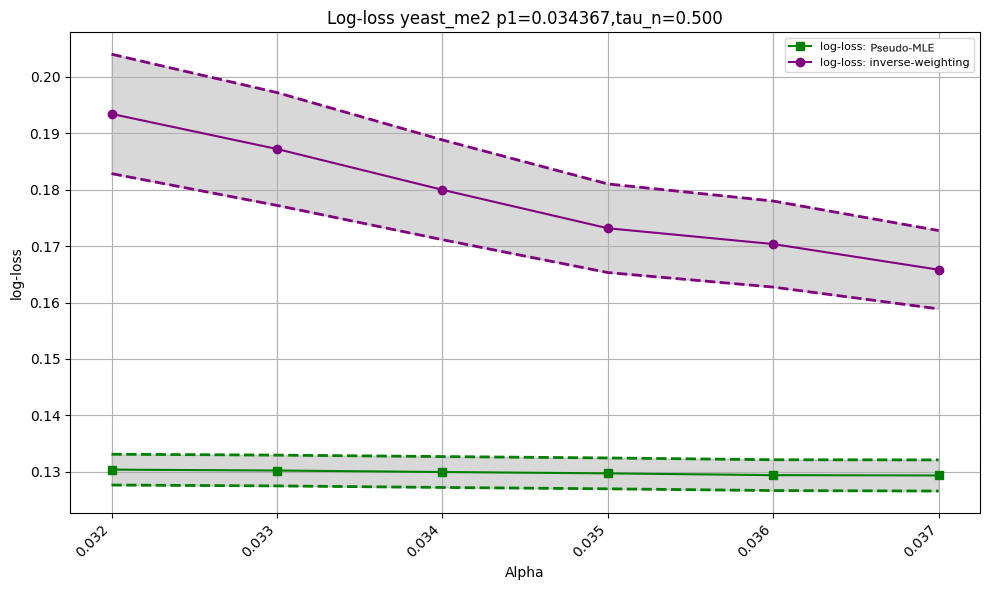}
     \caption{log-loss:$\tau_n=0.5$}
     \label{fig:tau_n=0.5:yeast_me2}
 \end{subfigure}
  \hfill
 \begin{subfigure}{0.48\textwidth}
     \includegraphics[width=\textwidth]{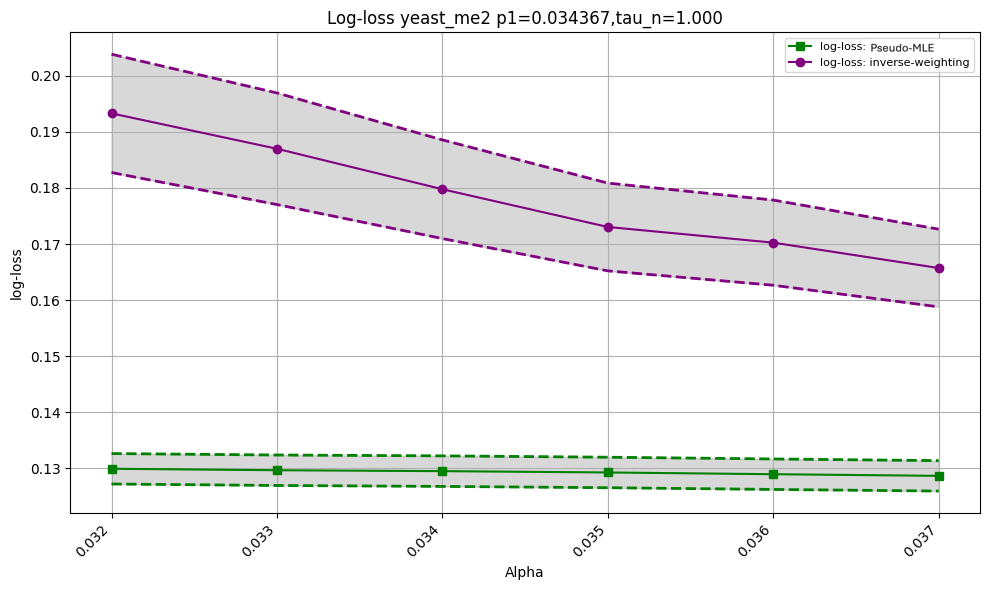}
     \caption{log-loss:$\tau_n=1.0$}
     \label{fig:tau_n=1:yeast_me2} 
 \end{subfigure} 
 \vfill
 \begin{subfigure}{0.48\textwidth}
     \includegraphics[width=\textwidth]{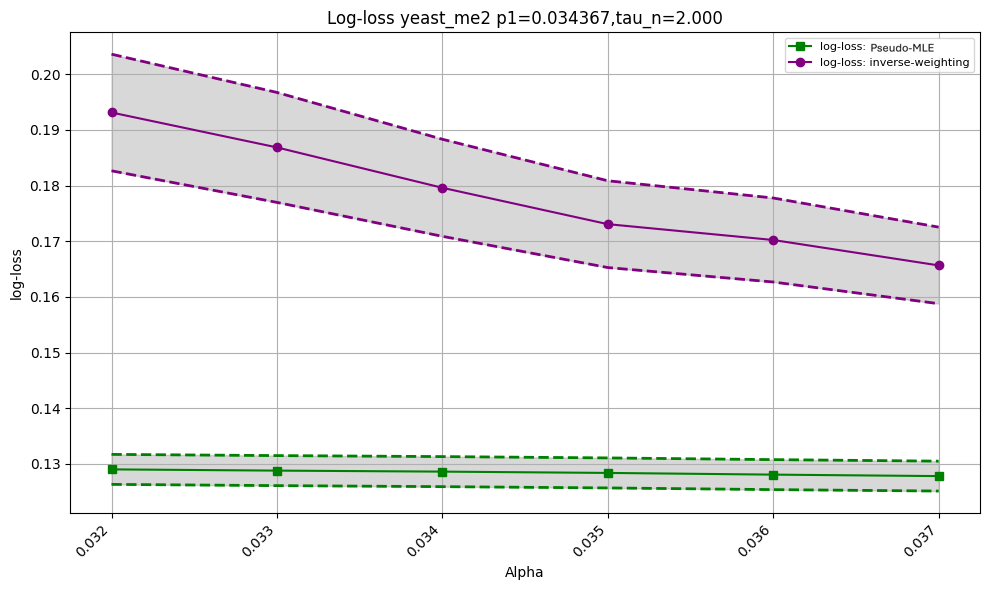}
     \caption{log-loss:$\tau_n=2.0$}
     \label{fig:tau_n=2:yeast_me2}
 \end{subfigure}
 \hfill
 \begin{subfigure}{0.48\textwidth}
     \includegraphics[width=\textwidth]{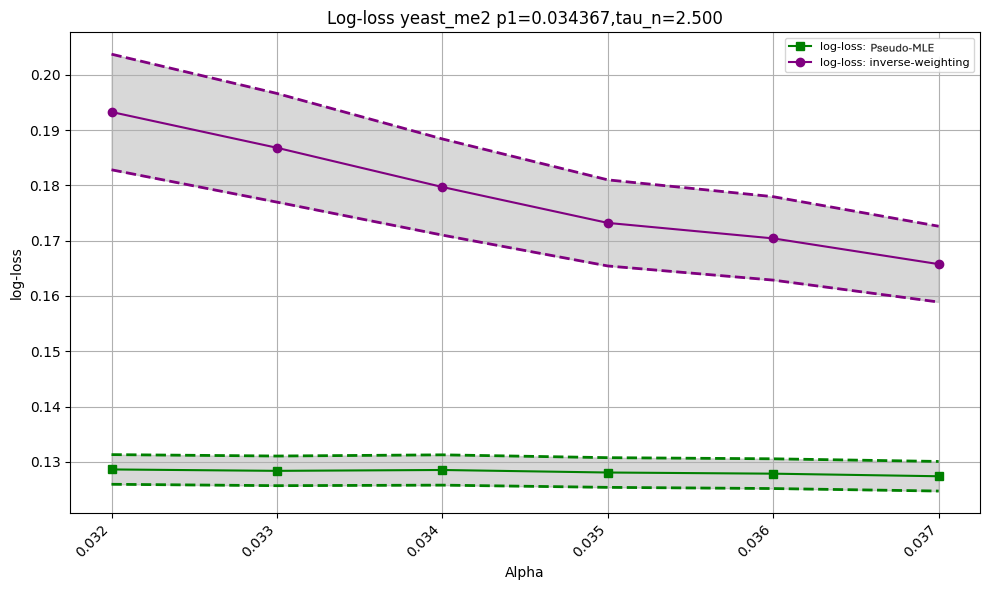}
     \caption{log-loss:$\tau_n=2.5$}
     \label{fig:tau_n=2.5:yeast_me2}
 \end{subfigure}
 \vfill
 \begin{subfigure}{0.48\textwidth}
     \includegraphics[width=\textwidth]{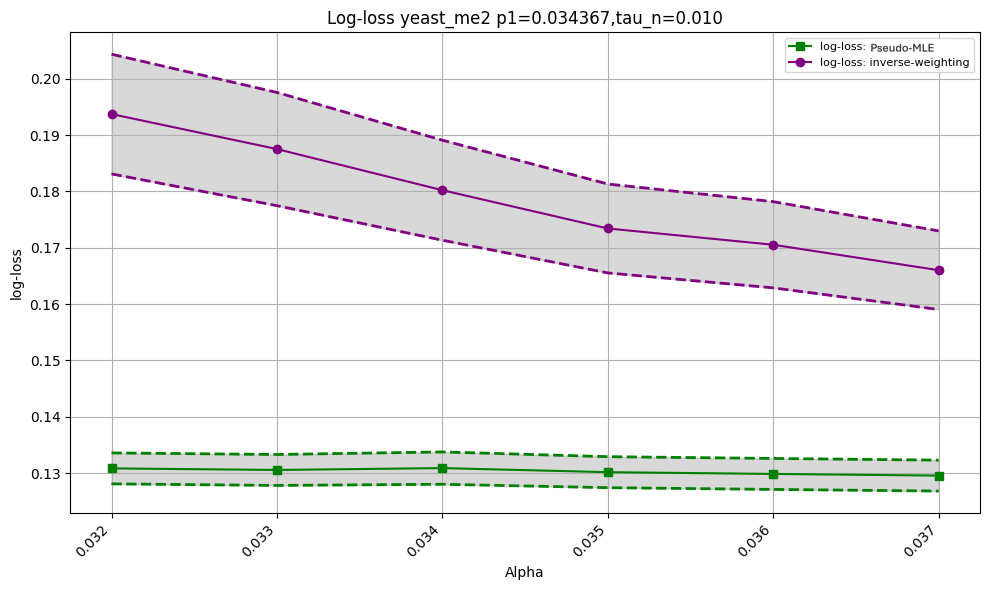}
     \caption{log-loss:$\tau_n=0.01$}
     \label{fig:tau_n=0.01:yeast_me2}
 \end{subfigure}
 \caption{Additional results for \texttt{yeast\_me2} dataset for small and moderate values of $\tau_n$.}
 \label{fig:yeast_me2:additional tau_n}
\end{figure}

\subsection{Simulation Results and Insights for Neural Networks}\label{sec:NN:additional:appendix}
\paragraph{Neural Networks} We plot the log-losses of neural networks applied to imbalanced UCI real data (\texttt{yeat\_me2}, \texttt{abalone\_19}, \texttt{ecoli}) for different downsampling rate $\alpha$ in Figures \ref{fig:NN1}, \ref{fig:NN2}, \ref{fig:NN3}. Dashed lines correspond to the cross-entropy loss, solid lines correspond to the customized loss function implied by our estimator. The results are average log-losses for each training epoch averaging across $500$ random train/test splitting of the real data. The solid lines are below the dashed lines within the same color, indicating the customized loss leading to better performance under the fixed small downsample rate $\alpha$ given here. The neural networks are trained with 3 dense layers with relu activation and one outer layer with sigmoid output.

Though a maximum likelihood analysis of a neural network model is significantly more challenging than generalized linear model, the insight about downsampling alone is intuitive and should carry over to the case of neural networks in additional to GLMs. For example, we could use a small dataset to train a neural network, and exploit the fact that the output activation function is often a GLM. We can then apply our results to the GLM portion and use this to adjust the sample size according to our optimal sample selection. 

\begin{figure}[htp!]
    \centering
    \begin{subfigure}{0.48\textwidth}
        \centering
        \includegraphics[height=3.5cm,width=\textwidth]{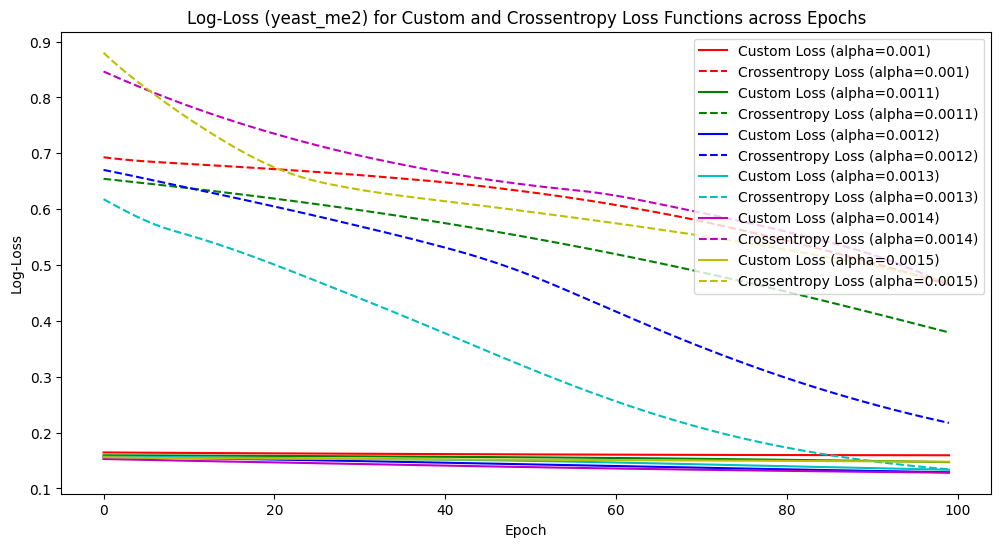} 
        \subcaption{Log-losses: UCI \texttt{yeast\_me2} data with NN}
        \label{fig:NN1}
    \end{subfigure}
    \hfill
    \begin{subfigure}{0.48\textwidth}
        \centering
        \includegraphics[height=3.5cm,width=\textwidth]{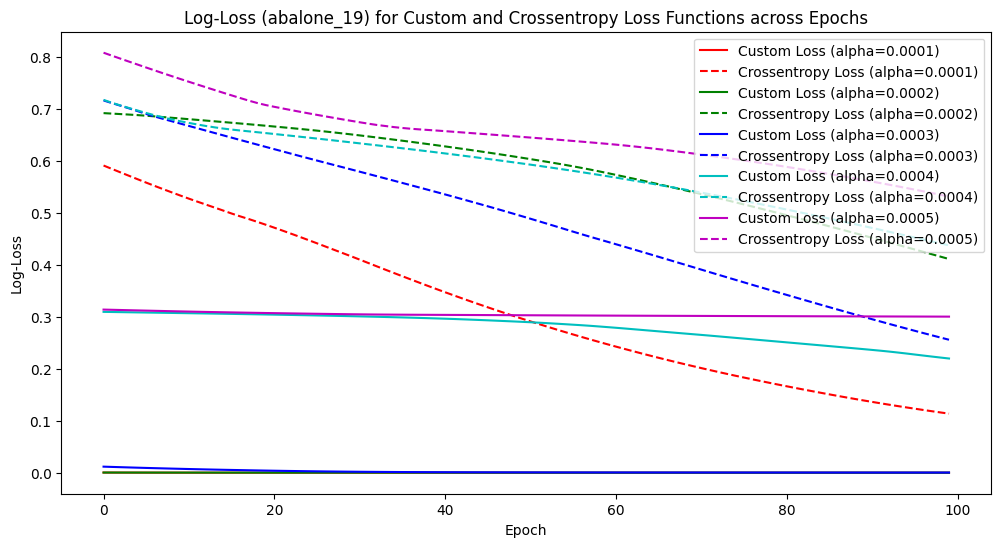} 
        \subcaption{Log-losses: UCI \texttt{abalone\_19} data (NN)}
        \label{fig:NN2}
    \end{subfigure}
    \vfill
    \begin{subfigure}{0.5\textwidth}
        \centering
        \includegraphics[height=3.5cm,width=\textwidth]{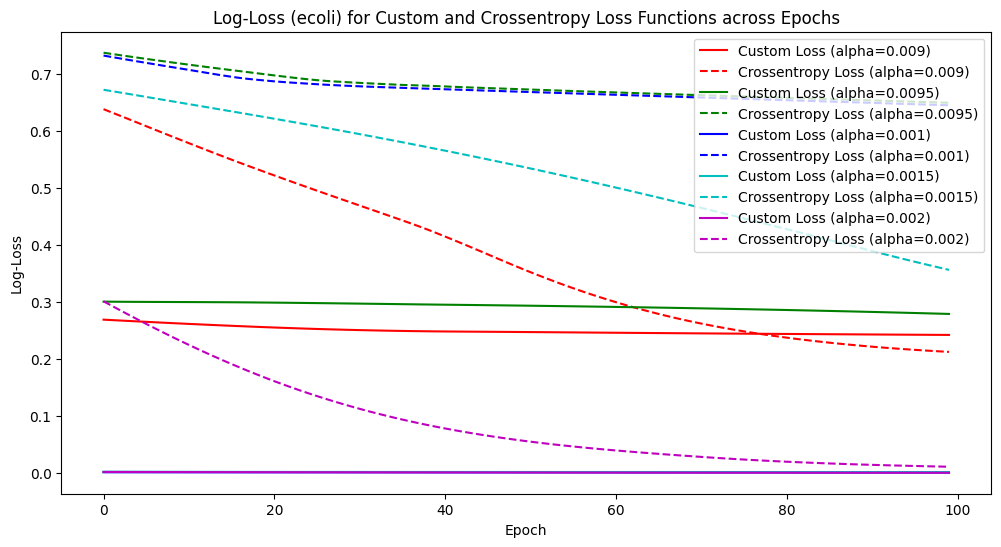} 
        \subcaption{Log-losses: UCI \texttt{ecoli} data with NN}
        \label{fig:NN3}
    \end{subfigure}
\end{figure}

\begin{figure}[htp!]
 \begin{subfigure}{0.48\textwidth}
  \includegraphics[width=\textwidth]{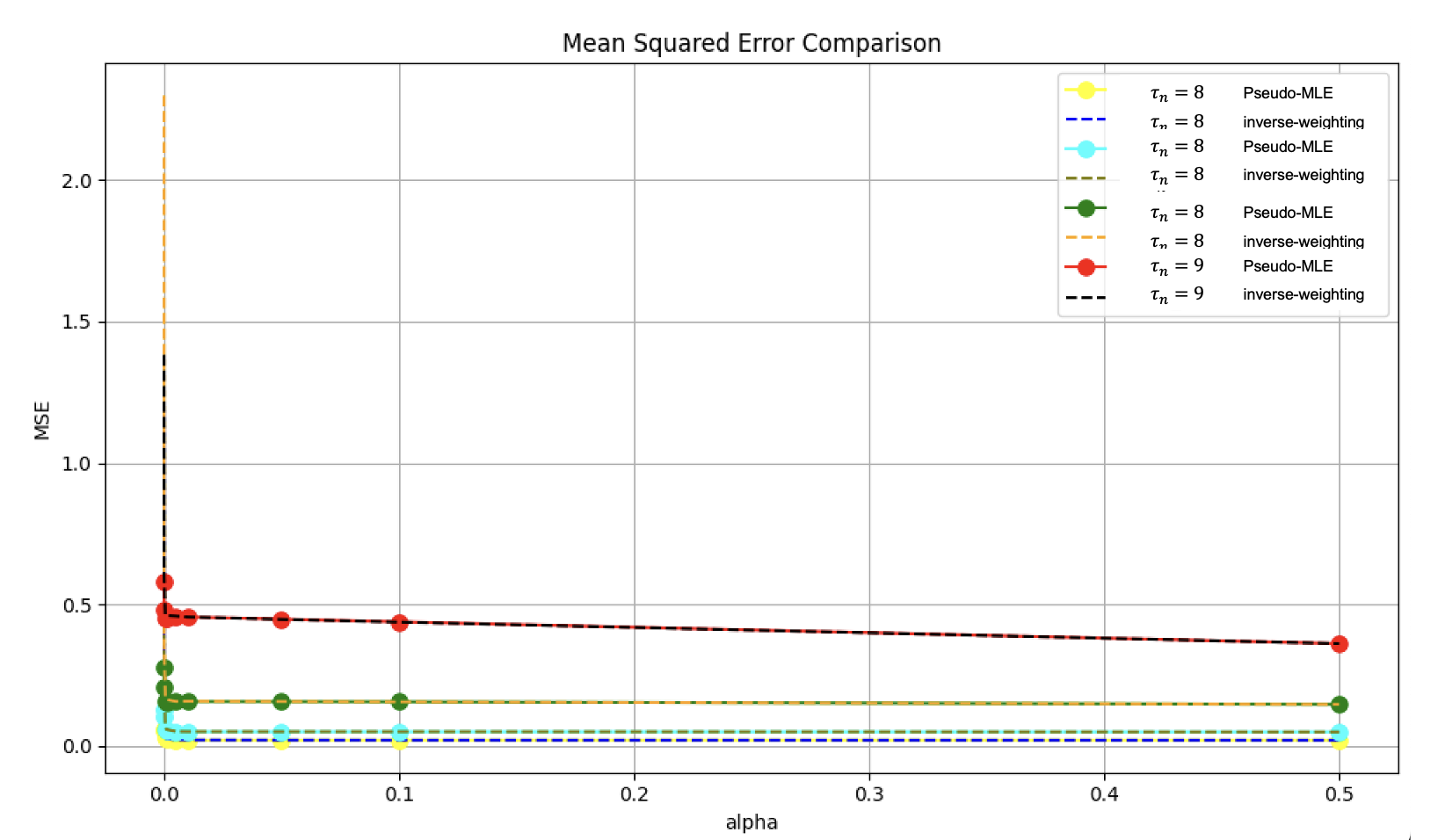}
  \caption{\footnotesize Mean-squared-error of our estimator vs. Inverse-weighting estimator under different downsample rates.}
 \label{fig:convergence:comparison}
 \end{subfigure}
 \hfill
 \begin{subfigure}{0.48\textwidth}
    \includegraphics[width=\textwidth]{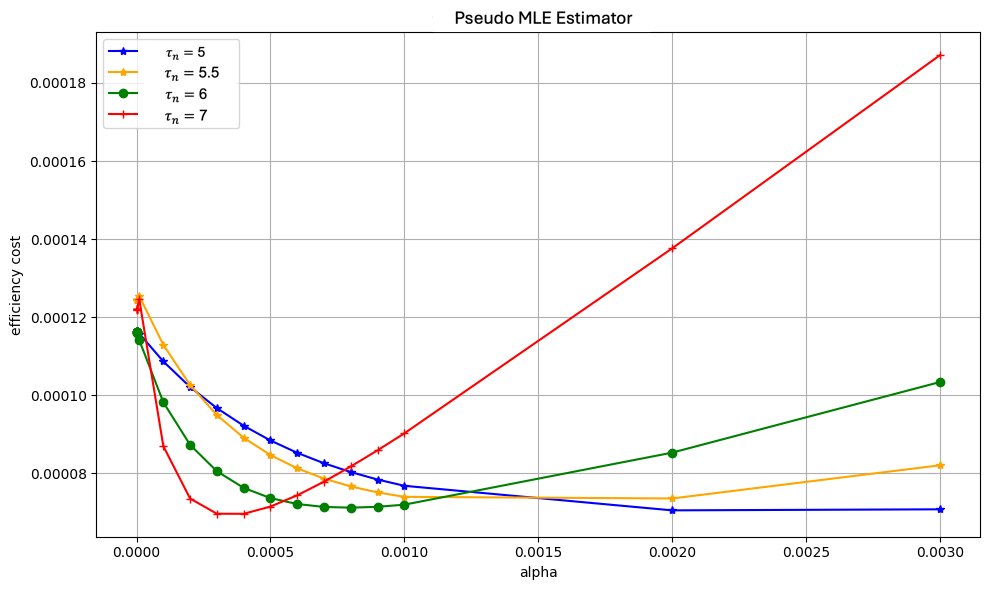}
    \caption{\footnotesize Efficiency cost combining both statistical efficiency and computational cost for $\tau_n=5,5.5,6,7$.}
\label{fig:efficiency-cost}
 \end{subfigure}
 \label{fig:small-alpha}
 \caption{Mean-squared-error and Efficiency costs}
\end{figure}

\begin{figure}[htp!]
 \begin{subfigure}{0.48\textwidth}
     \includegraphics[height=3.5cm,width=\textwidth]{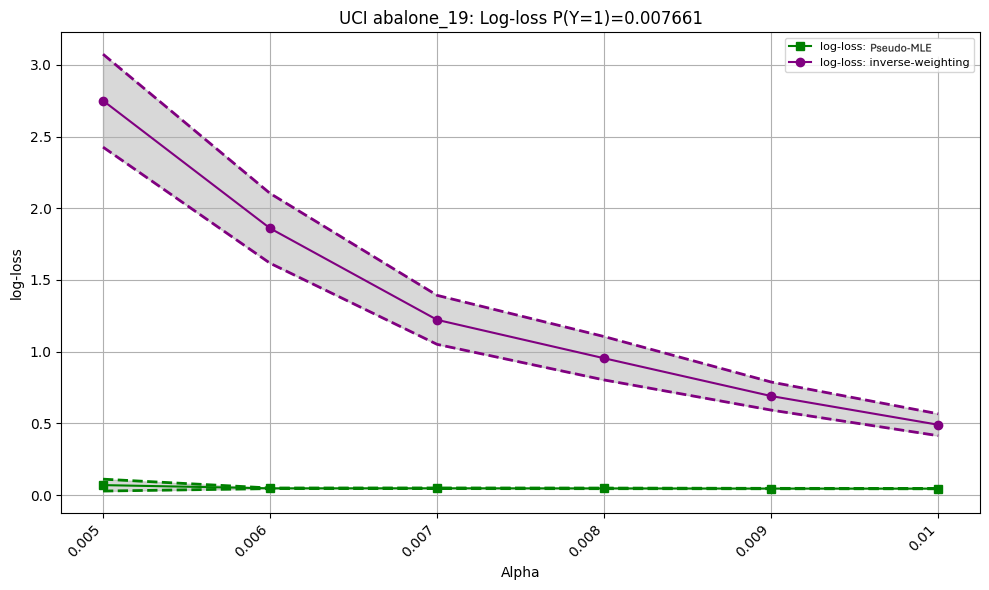}
     \caption{UCI \texttt{abalone\_19}}
     \label{fig:abolone_19} 
 \end{subfigure}
 \hfill
 \begin{subfigure}{0.48\textwidth}
     \includegraphics[height=3.5cm,width=\textwidth]{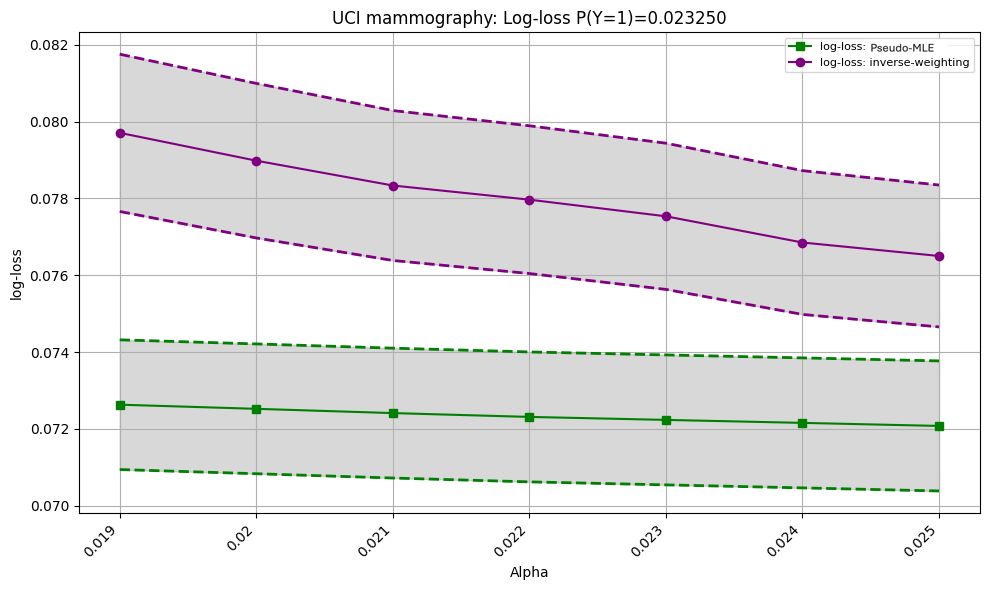}
     \caption{UCI \texttt{mammmography}}
     \label{fig:mammmography}
 \end{subfigure}
 \vfill
 \begin{subfigure}{0.48\textwidth}
     \includegraphics[height=3.5cm,width=\textwidth]{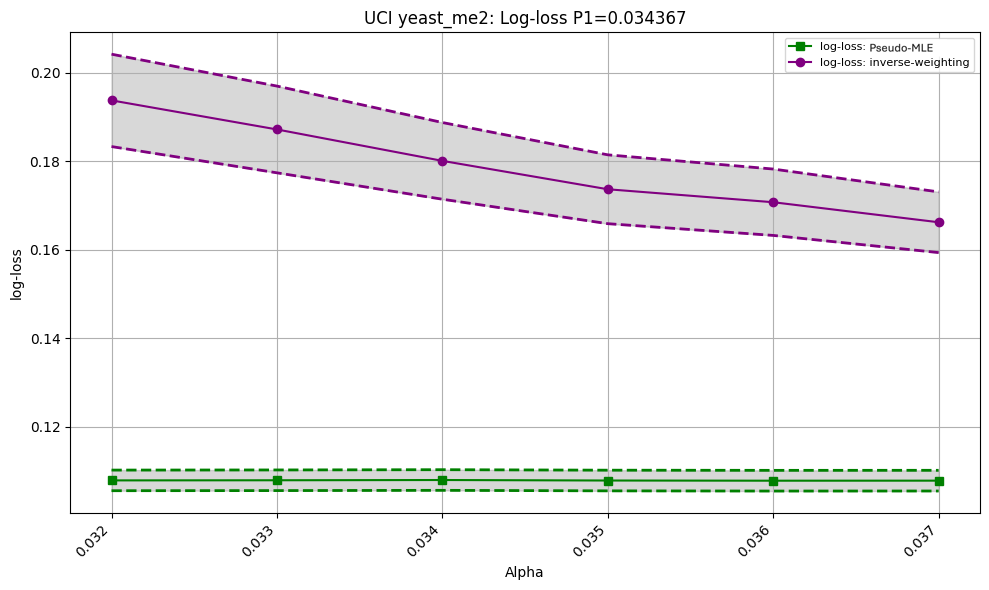}
     \caption{UCI \texttt{yeast\_me2}}
     \label{fig:yeast_me2}
 \end{subfigure}
 \hfill
 \begin{subfigure}{0.48\textwidth}
     \includegraphics[height=3.5cm,width=\textwidth]{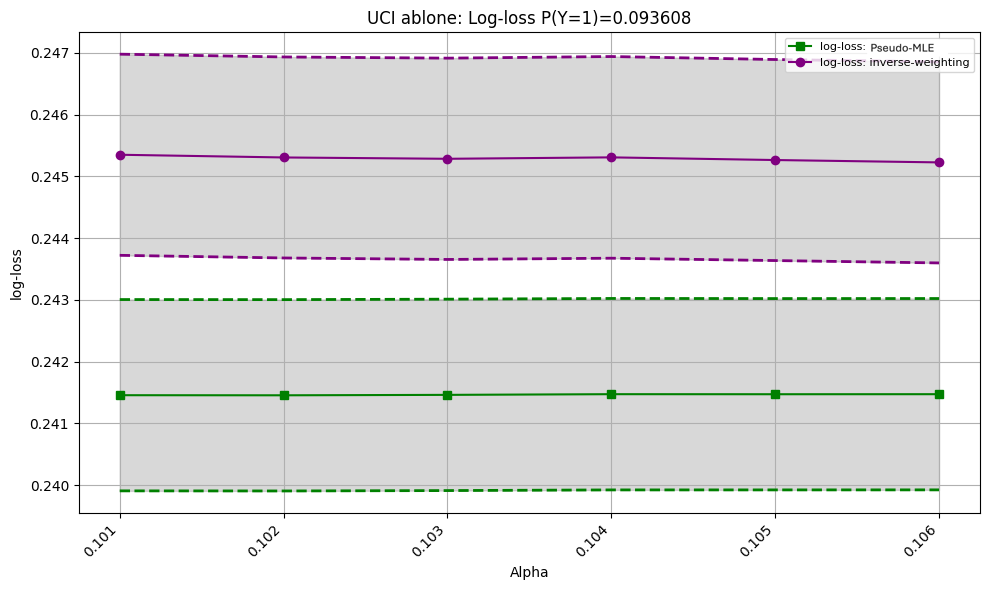}
     \caption{UCI \texttt{abalone}}
     \label{fig:ablone}
 \end{subfigure}
 \vfill
 \begin{subfigure}{0.48\textwidth}
     \includegraphics[height=3.5cm,width=\textwidth]{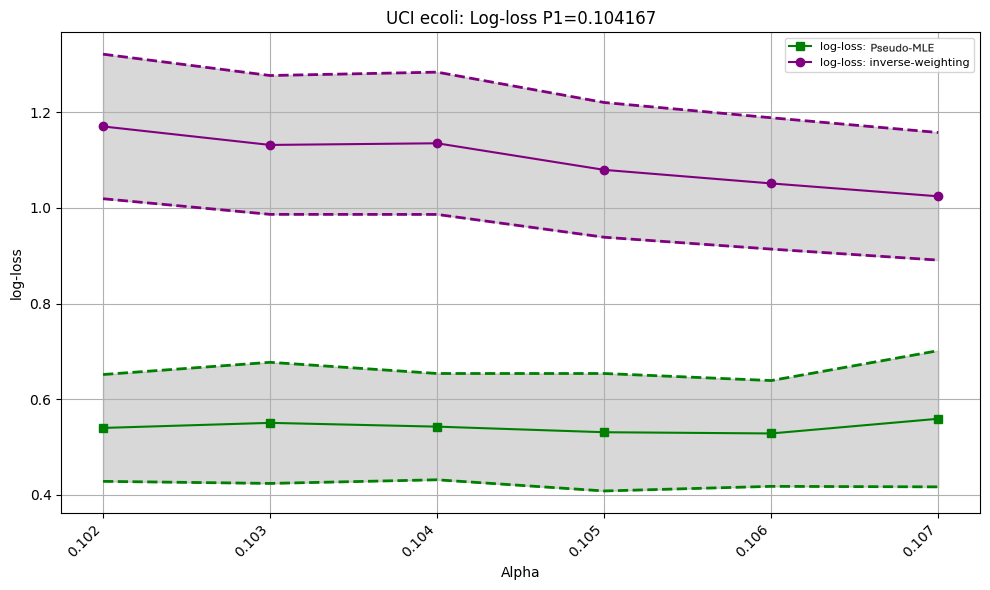}
     \caption{UCI \texttt{ecoli}}
     \label{fig:ecoli}
 \end{subfigure}
 \caption{Log-loss of inverse-weighting (purple) vs. log-loss of pseudo-MLE (green) for $\alpha$ chosen around $\mathbb{P}(Y=1)$ for each data set by applying Logistic Regression. We randomly split the original dataset into $80\%$ for training and $20\%$ for testing during each replication. The log-losses are all computed on test datasets.
 The purple dashed lines correspond to the $95\%$ confidence intervals of the log-loss of inverse-weighting estimator, and the green dashed lines correspond to the $95\%$ confidence intervals of the log-loss of pseudo MLE. The upper and lower ends are computed by $\pm1.96*\frac{\hat{\sigma}}{\sqrt{500}}$ and $\hat{\sigma}$ is the standard deviation of log-loss values at each alpha computed over $500$ random environments.}
 \label{fig:real-data-1}
\end{figure}

\section{Additional Lemmas and Proofs for the Proposed Estimator}\label{appendix:calibration:downsampling}

For the rest of the paper, we use $\tilde{P}$ to denote the joint distibution of downsample variables $(\tilde{X},\tilde{Y})$, use $\tilde{E}$ to denote the expectation with respect to $\tilde{P}$, use $\tilde{P}_N$ for the empirical measure induced by $\{\tilde{X}_i,\tilde{Y}_i\}_{i=1}^N$, and $\tilde{E}_N$ for the expectation taken with respect to $\tilde{P}_N$. 

\begin{lemma}[Counterexample]\label{lemma:counterexample:downsample-use-original}
Suppose $F(z)$ is strictly increasing, and suppose
\begin{itemize}
    \item[1)] $\mathbb{E}_X\left[F^\prime(\tau_n+\theta_*^T X)X\right]\neq\mathbf{0}$, 
    \item[2)] there exists a unique  $\tilde{\theta}_1\in\Theta$ such that
$\mathbb{E}_X\left[\frac{[1-(1-\alpha)F(\tau_n+\theta_*^T X)]F^\prime(\tau_n+\tilde{\theta}_1^T X)X}{1-(1-\alpha)F(\tau_n+\tilde{\theta}_1^T X)}\right]=\mathbf{0}$, 
\end{itemize}
Then (\ref{eq:downsample:original-model}) leads to a biased estimator for $\theta_*$. 

Furthermore, if $\{x\in\mathcal{X}|\tilde{\theta}_1^T x=0\}\neq\emptyset$ and $\{x\in\mathcal{X}|\theta_*^T x=0\}\cap\{x\in\mathcal{X}|\tilde{\theta}_1^T x=0\}\notin\{\emptyset,\mathcal{X}\}$, then the prediction score obtained by the procedure described above (i.e. solving (\ref{eq:downsample:original-model}) and then applying isotonic regression) is also biased.
\end{lemma}

\begin{proof}[Proof of Lemma \ref{lemma:counterexample:downsample-use-original}]
The first-order condition for (\ref{eq:downsample:original-model}) is 
$$\frac{1}{N}\sum_{i=1}^N\tilde{Y}_i\frac{-F^\prime(\tau_n+\hat{\theta}_1^T\tilde{X}_i)\tilde{X}_i^T}{1-F(\tau_n+\hat{\theta}_1^T\tilde{X}_i)}+(1-\tilde{Y}_i)\frac{F^\prime(\tau_n+\hat{\theta}_1^T\tilde{X}_i)\tilde{X}_i^T}{F(\tau_n+\hat{\theta}_1^T\tilde{X}_i)}=\mathbf{0},$$
i.e. $\hat{\theta}_1$ satisfies 
$$\frac{1}{N}\sum_{i=1}^N\frac{(1-\tilde{Y}_i-F(\tau_n+\hat{\theta}_1^T\tilde{X}_i))F^\prime(\tau_n+\hat{\theta}_1^T\tilde{X}_i)\tilde{X}_i^T}{F(\tau_n+\hat{\theta}_1^T\tilde{X}_i)(1-F(\tau_n+\hat{\theta}_1^T\tilde{X}_i))}=\mathbf{0}.$$

By Lemma \ref{prop:conditional independence} and the strong law of large numbers, for any $\theta_1\in\Theta$,

$$\begin{array}{rl}
&\quad\frac{1}{N}\sum_{i=1}^N\frac{(1-\tilde{Y}_i-F(\tau_n+\theta_1^T \tilde{X}_i))F^\prime(\tau_n+\theta_1^T \tilde{X}_i)\tilde{X}_i^T}{F(\tau_n+\theta_1^T \tilde{X}_i)(1-F(\tau_n+\theta_1^T \tilde{X}_i))}\\
&\overset{a.s.}{\longrightarrow}\tilde{E}\left[\frac{(1-\tilde{Y}_i-F(\tau_n+\theta_1^T \tilde{X}_i))F^\prime(\tau_n+\theta_1^T \tilde{X}_i)\tilde{X}_i^T}{F(\tau_n+\theta_1^T \tilde{X}_i)(1-F(\tau_n+\theta_1^T \tilde{X}_i))}\right]\\
&=\tilde{E}\left[\left[\frac{(1-\tilde{Y}_i-F(\tau_n+\theta_1^T \tilde{X}_i))F^\prime(\tau_n+\theta_1^T \tilde{X}_i)\tilde{X}_i^T}{F(\tau_n+\theta_1^T \tilde{X}_i)(1-F(\tau_n+\theta_1^T \tilde{X}_i))}\bigg|\tilde{X}_i\right]\right]\\
&=_{(1)}\tilde{E}\left[\frac{[G(\tau_n+\theta_1^T \tilde{X}_i)-F(\tau_n+\theta_1^T \tilde{X}_i)]F^\prime(\tau_n+\theta_1^T \tilde{X}_i)\tilde{X}_i^T}{F(\tau_n+\theta_1^T \tilde{X}_i)(1-F(\tau_n+\theta_1^T \tilde{X}_i))}\right]\\
&=-(1-\alpha)\tilde{E}\left[\frac{F^\prime(\tau_n+\theta_1^T \tilde{X}_i)\tilde{X}_i^T}{1-(1-\alpha)F(\tau_n+\theta_1^T \tilde{X}_i)}\right]\\
&=_{(2)}-(1-\alpha)\frac{\mathbb{E}_X\left[\frac{[1-(1-\alpha)F(\tau_n+\theta_*^T X)]F^\prime(\tau_n+\theta_1^T X)X}{1-(1-\alpha)F(\tau_n+\theta_1^T X)}\right]}{\mathbb{P}(Y=1)+\alpha\mathbb{P}(Y=0)},
\end{array}$$

where (1) follows from Proposition \ref{prop:downsample-prob} and (2) uses Lemma \ref{lemma:distribution-downsample-covariate}. Thus by Lemma \ref{lemma:thm-vdv-M-estimator}, $\hat{\theta}_1\overset{p}{\rightarrow}\tilde{\theta}_1$ such that 
$\mathbb{E}_X\left[\frac{[1-(1-\alpha)F(\tau_n+\theta_*^T X)]F^\prime(\tau_n+\tilde{\theta}_1^T X)X^T}{1-(1-\alpha)F(\tau_n+\tilde{\theta}_1^T X)}\right]=\mathbf{0}$.

Furthermore, taking in true parameter $\theta_*$, the expected value under probability measure $\tilde{P}$ (i.e. the joint distribution of $(\tilde{Y},\tilde{X})$) is equal to 

$$\begin{array}{rl}
&\quad\tilde{E}\left[\frac{(1-\tilde{Y}_i-F(\tau_n+\theta_*^T\tilde{X}_i))F^\prime(\tau_n+\theta_*^T\tilde{X}_i)\tilde{X}_i^T}{F(\tau_n+\theta_*^T\tilde{X}_i)(1-F(\tau_n+\theta_*^T\tilde{X}_i))}\right]\\
&=\tilde{E}\left[\tilde{E}\left[\frac{(1-\tilde{Y}_i-F(\tau_n+\theta_*^T\tilde{X}_i))F^\prime(\tau_n+\theta_*^T\tilde{X}_i)\tilde{X}_i^T}{F(\tau_n+\theta_*^T\tilde{X}_i)(1-F(\tau_n+\theta_*^T\tilde{X}_i))}\bigg|\tilde{X}_i\right]\right]\\
&=_{(1)}\tilde{E}\left[-\frac{F^\prime(\tau_n+\theta_*^T\tilde{X}_i)\bar{G}(\tau_n+\theta_*^T\tilde{X}_i)\tilde{X}_i^T}{1-F(\tau_n+\theta_*^T\tilde{X}_i)}+\frac{F^\prime(\tau_n+\theta_*^T\tilde{X}_i)G(\tau_n+\theta_*^T\tilde{X}_i)\tilde{X}_i^T}{F(\tau_n+\theta_*^T\tilde{X}_i)}\right]\\
&=_{(2)}\frac{-(1-\alpha)\mathbb{E}_X\left[F^\prime(\tau_n+\theta_*^T X)X^T\right]}{\mathbb{P}(Y=1)+\alpha\mathbb{P}(Y=0)},
\end{array}$$
where (1) uses Proposition \ref{prop:downsample-prob} and (2) uses Lemma \ref{lemma:distribution-downsample-covariate}, and $\mathbb{E}_X$ is the expectation taken with respect to the distribution of $X$ from the full data. Thus under the given conditions, 
$\mathbb{E}_X\left[F^\prime(\tau_n+\theta_*^T X)X^T\right]\neq\mathbf{0}$, while there exists some $\tilde{\theta}_1$, 
$$\begin{array}{rl}
&\quad\mathbb{E}_X\left[\frac{[1-(1-\alpha)F(\tau_n+\theta_*^T X)]F^\prime(\tau_n+\tilde{\theta}_1^T X)X^T}{1-(1-\alpha)F(\tau_n+\tilde{\theta}_1^T X)}\right]=\mathbb{E}_X\left[\frac{[1-(1-\alpha)F(\tau_n+\theta_*^T X)]\varphi(\tau_n+\tilde{\theta}_1^T X)X^T}{1-(1-\alpha)F(\tau_n+\tilde{\theta}_1^T X)}\right]=\mathbf{0}.
\end{array}$$
Thus $\hat{\theta}_1\overset{p}{\rightarrow}\tilde{\theta}_1$ but $\tilde{\theta}_1\neq\theta_*$. 

Consequently, under the third condition in the lemma,
there exists $x_1,x_2\in\mathcal{X}$ such that $\tilde{\theta}_1^T(x_1-x_2)=0$
while $\theta_*^T(x_1-x_2)\neq0$. This is because there exists $\Delta\in\mathcal{X}$ such that $\Delta$ is in the subspace defined by the hyperplane induced by $\tilde{\theta}_1$: $\Delta\in\{x\in\mathcal{X}|\tilde{\theta}_1^T x=0\}$ while $\theta_*^T\Delta\neq0$. Then for any $x_1\in\mathcal{X}$, there exists $\lambda$ sufficiently small such that $x_2=x_1+\lambda\Delta\in\mathcal{X}$, with $\tilde{\theta}_1^T(x_1-x_2)=0$ while  $\theta_*^T(x_1-x_2)\neq0$. Therefore, $\tau_n+\tilde{\theta}_1^T x_1=\tau_n+\tilde{\theta}_1^T x_2$ while $\tau_n+\theta_*^T x_1\neq\tau_n+\theta_*^T x_2$. Thus suppose there exists a monotone transformation $g$ such that $F(\tau_n+\theta_*^T x)=g\circ F(\tau_n+\tilde{\theta}_1^T x)$ for all $x\in\mathcal{X}$, then $F(\tau_n+\theta_*^T x_1)=g\circ F(\tau_n+\tilde{\theta}_1^T x_1)$. Also note that $\tau_n+\tilde{\theta}_1^T x_1=\tau_n+\tilde{\theta}_1^T x_2$ thus $g\circ F(\tau_n+\tilde{\theta}_1^T x_1)=g\circ F(\tau_n+\tilde{\theta}_1^T x_2)=F(\tau_n+\theta_*^T x_2)$, thus $F(\tau_n+\theta_*^T x_1)=F(\tau_n+\theta_*^T x_2)$, which leads to contradiction because $F$ is strictly increasing. 
\end{proof}

\begin{lemma}\label{prop:downsample-prob}
Let $\mathbb{P}$ be the joint distribution of $(Y,X)$, for any $y\in\{0,1\}$ and $x\in\mathcal{X}$,
define 
{\footnotesize\begin{equation}\label{eq:downsample-prob}
\tilde{P}(\tilde{Y}=y,\tilde{X}=x) := \frac{y\mathbb{P}\left(Y=1,X=x\right)+(1-y)\alpha\mathbb{P}(Y=0,X=x)}{\mathbb{P}(Y=1)+\alpha\mathbb{P}(Y=0)},
\end{equation}}
then (\ref{eq:downsample-prob}) defines a probability distribution $\tilde{P}$ with respect to downsample random variables $(\tilde{Y}_i,\tilde{X}_i)$.
\end{lemma}

\begin{proof}[Proof of Lemma \ref{prop:downsample-prob}]
Note that the downsampling procedure is equivalent to 
{\footnotesize\begin{equation}\label{downsample-GLM}
\left(\Tilde{Y},\Tilde{X}\right)=\mathbf{1}\left(Y=1\right)\left(Y,X\right)+\mathbf{1}(Y=0)\mathbf{1}(U\leq\alpha)(Y,X),
\end{equation}}
where $U\sim\textrm{Uniform}\left[0,1\right]$and $U\indep(Y,X)$. 
Thus the joint distribution (density) of $(\tilde{Y},\tilde{X})$ with respect to $\mathbb{P}_{(Y,X)}$ as the joint law of $(Y,X)$ can be written as 
{\footnotesize\begin{align}\label{eq:key:iid}
\mathbb{P}_{(Y,X)}\left(\left(\Tilde{Y},\Tilde{X}\right)=\left(y,x\right)\right)&=\mathbf{1}(y=1)\mathbb{P}\left(Y=1,X=x\right)+\mathbf{1}(y=0)\alpha\mathbb{P}(Y=0,X=x)\notag\\
&=y\mathbb{P}\left(Y=1,X=x\right)+(1-y)\alpha\mathbb{P}(Y=0,X=x).
\end{align}}
So integrating with respect to $(y,x)$ we have 
{\footnotesize$$\int_{\mathcal{X}}\mathbb{P}\left[\left(Y=1,X=x\right)+\alpha\mathbb{P}(Y=0,X=x)\right]dx=\mathbb{P}(Y=1)+\alpha\mathbb{P}(Y=0).$$}
So by definition 
{\footnotesize$$\int_{\mathcal{X}}\sum_{y\in\{0,1\}}\tilde{P}(\tilde{Y}=y,\tilde{X}=x)dx=1.$$}
Further note that {\footnotesize$\tilde{P}(\tilde{Y}=y,\tilde{X}=x)\in[0,1]$} for any $x\in\mathcal{X}$ and $y\in\{0,1\}$. Thus $\tilde{P}$ is indeed a valid probability distribution defined for $(\tilde{Y},\tilde{X})$. 
\end{proof}

\begin{lemma}[i.i.d. Property]\label{prop:conditional independence}
The downsampled data {\footnotesize$\{(\tilde{X}_i,\tilde{Y}_i)\}_{i=1}^N$} are i.i.d. generated with respect to $\tilde{P}$ defined in (\ref{eq:downsample-prob}) of Lemma \ref{prop:downsample-prob}.
\end{lemma}

\begin{proof}[Proof of Lemma \ref{prop:conditional independence}]
Recall that the \textit{Generalized Linear Model} is defined as follows: For some latent variable $Z$, the label is defined as {\footnotesize$Y = \mathbf{1}\left(Z>\tau_n+\theta_*^TX\right), \mbox{ and } \mathbb{P}\left(Y=1\big|X=x\right)=\Bar{F}_Z(\tau_n+\theta_*^Tx).$} For $\forall i\in[N]$, let $(Y_i,X_i)$ denote the full-sample random variable. Thus for $i,j\in[N]$, $i\neq j$, given any pairs of event $(y,x)$, $(y^\prime,x^\prime)$, with $\mathbb{P}$ denoting the joint distribution of the full-sample random variables $(Y_i,X_i)$, and $U$ denote a uniform random variable on $[0,1]$ such that $U\indep\{(X_i,Y_i)\}_{i=1}^n$, by (\ref{eq:key:iid}) we have 
{$$\begin{array}{rl}&\quad\mathbb{P}\left(\left(\tilde{Y}_i,\tilde{X}_i\right)=(y,x),\left(\tilde{Y}_j,\tilde{X}_j\right)=(y^\prime,x^\prime)\right)\\
&=_{(a)}\mathbb{P}\left((Y_i,X_i)=(1,x),\left(\tilde{Y}_j,\tilde{X}_j\right)=(y^\prime,x^\prime)\right)\\
&\quad\quad+\mathbb{P}\left((Y_i,X_i)=(0,x),U\leq\alpha,\left(\tilde{Y}_j,\tilde{X}_j\right)=(y^\prime,x^\prime)\right)\\
&=_{(b)}\mathbb{P}\left((Y_i,X_i)= (1,x)\right)\mathbb{P}\left(\left(\tilde{Y}_j,\tilde{X}_j\right)=(y^\prime,x^\prime)\right)\\
&\quad\quad+\mathbb{P}\left((Y_i,X_i)= (0,x),U\leq\alpha\right)\mathbb{P}\left(\left(\tilde{Y}_j,\tilde{X}_j\right)=(y^\prime,x^\prime)\right)\\
&=_{(c)}\mathbb{P}\left(\left(\tilde{Y}_i,\tilde{X}_i\right)=(y,x)\right)\mathbb{P}\left(\left(\tilde{Y}_j,\tilde{X}_j\right)=(y^\prime,x^\prime)\right),
\end{array}$$}   
where (a) uses the fact that $\mathbf{1}((Y_i,X_i)= (y,x),Y_i=1)$ and $\mathbf{1}((Y_i,X_i)=(y,x),Y_i=0,U\leq\alpha)$ are disjoint events, and (b) uses the fact that $(\tilde{Y}_j,\tilde{X}_j)=\mathcal{L}(Y_j,X_j,U)$ as some joint law of $(Y_j,X_j,U)$, which is independent of $(Y_i,X_i)$. Moreover, (c) uses the fact that {\footnotesize$\mathbb{P}\left(\left(\tilde{Y}_i,\tilde{X}_i\right)=(y,x)\right)=\mathbb{P}\left((Y_i,X_i)= (1,x)\right)+\mathbb{P}\left((Y_i,X_i)=(0,x),U\leq\alpha\right)$} according to (\ref{downsample-GLM}). Thus $(\tilde{Y}_i,\tilde{X}_i)$ and $(\tilde{Y}_j,\tilde{X}_j)$ are independent with respect to $\mathbb{P}$. 

Note that $\tilde{P}=\mathbb{P}/(\mathbb{P}(Y=1)+\alpha\mathbb{P}(Y=0))$, where given $\tau_n,\theta_*$, $\mathbb{P}(Y=1)+\alpha\mathbb{P}(Y=0)$ is a constant, thus $(\tilde{Y}_i,\tilde{X}_i)$ and $(\tilde{Y}_j,\tilde{X}_j)$ are independent with respect to $\tilde{P}$. Obviously $(\tilde{Y}_j,\tilde{X}_j)$ are identically generated. So the result follows.
\end{proof}

\begin{proof}[Proof of Proposition \ref{lemma:downsample GLM}]
From (\ref{downsample-GLM}) we know that for $y\in\{0,1\}$, $x\in\mathcal{X}$, and $\mathbb{P}$ as the joint law of $(Y,X)$, we have 
{\footnotesize$$\begin{array}{rl}
\mathbb{P}\left(\left(\Tilde{Y},\Tilde{X}\right)\in\left(y,x\right)\right)&=\mathbb{P}\left((Y,X)\in(y,x),Y=1\right)+\mathbb{P}\left((Y,X)=(y,x),Y=0,U\leq\alpha\right)\\
&=\mathbf{1}(y=1)\mathbb{P}(Y=1,X=x)+\mathbf{1}(y=0)\alpha\mathbb{P}(Y=0,X=x)
\end{array}$$}
Note that 
{$$\begin{array}{rl}
\tilde{P}\left(\Tilde{Y}=1\big|\Tilde{X}=x\right)&=\frac{\tilde{P}(\tilde{Y}=1,\tilde{X}=x)}{\tilde{P}(\tilde{X}=x)}\\
&=\frac{\mathbb{P}(\tilde{Y}=1,\tilde{X}=x)/(\mathbb{P}(Y=1)+\alpha\mathbb{P}(Y=0))}{\mathbb{P}(\tilde{Y}=1,\tilde{X}=x)/(\mathbb{P}(Y=1)+\alpha\mathbb{P}(Y=0))+\mathbb{P}(\tilde{Y}=0,\tilde{X}=x)/(\mathbb{P}(Y=1)+\alpha\mathbb{P}(Y=0))}\\
&=\frac{\mathbb{P}(\tilde{Y}=1,\tilde{X}=x)}{\mathbb{P}(\tilde{Y}=1,\tilde{X}=x)+\mathbb{P}(\tilde{Y}=0,\tilde{X}=x)}=\mathbb{P}(\tilde{Y}=1|\tilde{X}=x)\\
&=\frac{\mathbb{P}(Y=1,X=x)}{\mathbb{P}(Y=1,X=x)+\alpha\mathbb{P}(Y=0,X=x)}\\
&=\frac{\mathbb{P}(Y=1|X=x)\mathbb{P}(X=x)}{\mathbb{P}(Y=1|X=x)\mathbb{P}(X=x)+\mathbb{P}(Y=0|X=x)\mathbb{P}(X=x)}\\
&=\frac{\mathbb{P}(Y=1\big|X=x)}{\mathbb{P}(Y=1\big|X=x)+\alpha\mathbb{P}(Y=0\big|X=x)}.
\end{array}$$}
When $\alpha=1$, $\mathbb{P}\left(\Tilde{Y}=1\big|\Tilde{X}=x\right)=\mathbb{P}(Y=1\big|X=x)$. Note that $\mathbb{P}\left(Y=1\big|X=x\right)=\Bar{F}_Z(\tau_n+\theta_*^Tx)=1-F_Z(\tau_n+\theta_*^T x),$ then $\mathbb{P}\left(\Tilde{Y}=1\big|\tilde{X}=x\right)=\frac{\Bar{F}_Z(\tau_n+\theta_*^T x)}{\Bar{F}_Z(\tau_n+\theta_*^T x)+(1-\Bar{F}_Z(\tau_n+\theta_*^T x))\alpha}=\frac{\Bar{F}_Z(\tau_n+\theta_*^T x)}{(1-\alpha)\Bar{F}_Z(\tau_n+\theta_*^T x)+\alpha}$. Let $\Bar{G}(z)=\frac{\Bar{F}_Z(z)}{\Bar{F}_Z(z)(1-\alpha)+\alpha},$ then $\mathbb{P}(\tilde{Y}=1|\tilde{X}=x)=\bar{G}(\tau_n+\theta_*^T x)$. Note $\Bar{G}(\infty)=0$, $\Bar{G}^\prime(z)<0$, $\Bar{G}(-\infty)=1$, thus there exists some $W$, such that $\Bar{G}(z)=\Bar{F}_W(z),$ where $\Bar{F}_W(z)=1-F_W(z)$, and $F_W(\cdot)$ is the c.d.f. of $W$, and $\Bar{G}(z)=\frac{1}{(1-\alpha)+\alpha/\Bar{F}_Z}\iff \Bar{F}_Z=\frac{\alpha\Bar{G}}{1-(1-\alpha)\Bar{G}}.$    
\end{proof}

\begin{lemma}[Distribution of $\tilde{X}_i$ with respect to $\tilde{P}$]\label{lemma:distribution-downsample-covariate} The density function of $\tilde{X}_i$ at $\tilde{X}_i=x$ is {$\tilde{\mu}(x)=\frac{[1-(1-\alpha)F(\tau_n+\theta_*^T x)]\mu(x)}{\mathbb{P}(Y=1)+\alpha\mathbb{P}(Y=0)}.$} 
\end{lemma}

\begin{proof}[Proof of Lemma \ref{lemma:distribution-downsample-covariate}]
From (\ref{downsample-GLM}) and previous proofs, with $\mathbb{P}$ denoting the joint law of $(Y,X)$, we have 
{$$\begin{array}{rl}
&\quad\frac{\bar{F}(\tau_n+\theta_*^T x)\mu(x)}{\mathbb{P}(Y=1)+\alpha\mathbb{P}(Y=0)}=\frac{\mathbb{P}\left(Y=1,X=x\right)}{\mathbb{P}(Y=1)+\alpha\mathbb{P}(Y=0)}\\
&=\frac{\mathbb{P}\left(\tilde{Y}=1,\tilde{X}=x\right)}{\mathbb{P}(Y=1)+\alpha\mathbb{P}(Y=0)}=\tilde{P}\left(\tilde{Y}=1,\tilde{X}=x\right)=\bar{G}(\tau_n+\theta_*^T x)\tilde{\mu}(x),
\end{array}$$}
{$$\begin{array}{rl}
&\quad\frac{\alpha F(\tau_n+\theta_*^T x)\mu(x)}{\mathbb{P}(Y=1)+\alpha\mathbb{P}(Y=0)}=\frac{\alpha\mathbb{P}\left(Y=0,X=x\right)}{\mathbb{P}(Y=1)+\alpha\mathbb{P}(Y=0)}\\
&=\frac{\mathbb{P}\left(\tilde{Y}=0,\tilde{X}=x\right)}{\mathbb{P}(Y=1)+\alpha\mathbb{P}(Y=0)}=\tilde{P}\left(\tilde{Y}=0,\tilde{X}=x\right)=G(\tau_n+\theta_*^T x)\tilde{\mu}(x).
\end{array}$$}

so the density function of $\tilde{X}_i$ wtih respect to $\tilde{P}$ at $\tilde{X}_i=x$ is equal to 
{\footnotesize$\tilde{\mu}(x)=\frac{[1-(1-\alpha)F(\tau_n+\theta_*^T x)]\mu(x)}{\mathbb{P}(Y=1)+\alpha\mathbb{P}(Y=0)}.$}
\end{proof}

\begin{proof}[Proof of Proposition \ref{lemma:identification-theta}]
Note that by definition 
the joint distribution (density) of $(\tilde{Y},\tilde{X})$ with respect to $\tilde{P}$ can be written as 
{\footnotesize$$\begin{array}{rl}
&\quad\tilde{P}\left(\Tilde{Y}=y,\Tilde{X}=x\right)\\
&=y\tilde{P}\left(\tilde{Y}=1|\tilde{X}=x\right)\tilde{\mu}(x)+(1-y)\tilde{P}\left(\tilde{Y}=0|\tilde{X}=x\right)\tilde{\mu}(x)\\
&=y\bar{G}(\tau_n+\theta_*^T x)\tilde{\mu}(x)dx+(1-y)G(\tau_n+\theta_*^T x)\tilde{\mu}(x)dx\\
&=\left(\bar{F}(\tau_n+\theta_*^T x)\frac{\mu(x)}{\mathbb{P}(Y=1)+\alpha\mathbb{P}(Y=0)}\right)^y\left(\alpha F(\tau_n+\theta_*^T x)\frac{\mu(x)}{\mathbb{P}(Y=1)+\alpha\mathbb{P}(Y=0)}\right)^{1-y},
\end{array}$$}
where the third equality in the above uses Lemma \ref{lemma:distribution-downsample-covariate}. 

Let $\tilde{E}[\cdot]$ denote the expectation taken with respect to $\tilde{P}$, then we have 
{\footnotesize$$\begin{array}{rl}
&\quad M(\theta_1;\tau_n):=\tilde{E}\left[\log \tilde{P}\left(\tilde{Y}_i,\tilde{X}_i\right)\right]\\
&=\tilde{E}\bigg[\tilde{Y}_i\left(\log(\bar{F}(\tau_n+\theta_1^T \tilde{X}_i))+\log\frac{\mu(\tilde{X}_i)}{\mathbb{P}(Y=1)+\alpha\mathbb{P}(Y=0)}\right)\\
&\quad\quad+(1-\tilde{Y}_i)\left(\log\left(\alpha F(\tau_n+\theta_1^T \tilde{X}_i)\right)+\log\frac{\mu(\tilde{X}_i)}{\mathbb{P}(Y=1)+\alpha\mathbb{P}(Y=0)}\right)\bigg]\\
&=\tilde{E}\left[\tilde{Y}_i\log\bar{F}(\tau_n+\theta_1^T \tilde{X}_i)+(1-\tilde{Y}_i)\log\alpha F(\tau_n+\theta_1^T \tilde{X}_i)+\log\mu(\tilde{X}_i)\right]\\
&\quad\quad-\log\left[\mathbb{P}(Y=1)+\alpha\mathbb{P}(Y=0)\right]\\
&=\tilde{E}\left[\tilde{Y}_i\log\bar{F}(\tau_n+\theta_1^T \tilde{X}_i)+(1-\tilde{Y}_i)\log\alpha F(\tau_n+\theta_1^T \tilde{X}_i)+\log\mu(\tilde{X}_i)\right]\\
&\quad\quad-\log\int_{\mathcal{X}}\left[1-(1-\alpha)F(\tau_n+\theta_1^T x)\right]\mu(x)dx.
\end{array}$$}
We then have {\footnotesize$\theta_*\in\argmax_{\theta_1\in\Theta}M(\theta_1;\tau_n)$}, and note from Lemma \ref{prop:conditional independence} $(\tilde{X}_i,\tilde{Y}_i)$ are i.i.d. with respect to $\tilde{P}$, so by strong law of large numbers,

{$$\begin{array}{rl}
&\frac{1}{N}\sum_{i=1}^N\tilde{Y}_i\log\bar{F}(\tau_n+\theta_1^T \tilde{X}_i)+(1-\tilde{Y}_i)\log\alpha F(\tau_n+\theta_1^T \tilde{X}_i)+\log\mu(\tilde{X}_i)\\
&\quad-\log\int_{\mathcal{X}}\left[1-(1-\alpha)F(\tau_n+\theta_1^T x)\right]\mu(x)dx\\
&\overset{p}{\rightarrow}\tilde{E}\left[\tilde{Y}_i\log\bar{F}(\tau_n+\theta_1^T \tilde{X}_i)+(1-\tilde{Y}_i)\log\alpha F(\tau_n+\theta_1^T \tilde{X}_i)+\log\mu(\tilde{X}_i)\right]\\
&\quad-\log\int_{\mathcal{X}}\left[1-(1-\alpha)F(\tau_n+\theta_1^T x)\right]\mu(x)dx.
\end{array}$$}

Further note that given the down-sample $(\tilde{X}_i,\tilde{Y}_i)$, $\frac{1}{N}\sum_{i=1}^N\log\mu(\tilde{X}_i)$ doesn't depend on $\tau_n$ or $\theta_1$, thus the maximum likelihood estimator can be defined as 
{\footnotesize$$\begin{array}{rl}
\hat{\theta}_*=\argmax_{\theta_1}\frac{1}{N}\sum_{i=1}^N\ell\left(\tilde{X}_i,\tilde{Y}_i,\theta_1;\tau_n\right),
\end{array}$$}
where
{\footnotesize$$\ell(\tilde{x},\tilde{y},\theta_1;\tau_n)=\tilde{y}\log\bar{F}(\tau_n+\theta_1^T \tilde{x})+(1-\tilde{y})\log\alpha F(\tau_n+\theta_1^T \tilde{x})-\log\int_{\mathcal{X}}\left[1-(1-\alpha)F(\tau_n+\theta_1^T x)\mu(x)\right]dx.$$}
Thus the result follows. 
\end{proof}

\section{Proofs for Asymptotic Analysis}\label{section:varying-rate-regime:appendix}
\paragraph{Discussion on the failure of classical MLE analysis} Note that 
{$$\begin{array}{rl}
&\quad\mathbb{E}\left[\tilde{\ell}(\tilde{X},\tilde{Y},\theta_*;\tau_n)\right]=\mathbb{E}_{(\tilde{Y},\tilde{X})}\left[\tilde{Y}\log\bar{F}(\tau_n+\theta_*^T\tilde{X})+(1-\tilde{Y})\log\alpha F(\tau_n+\theta_*^T\tilde{X})\right]\\
&\quad\quad\quad\quad\quad\quad\quad\quad\quad-\mathbb{E}_{(\tilde{Y},\tilde{X})}\{\log\frac{1}{N}\sum_{i=1}^N[1-(1-\alpha)F(\tau_n+\theta_1^T\tilde{X}_i)]\}\\
\\
&=\mathbb{E}_{\tilde{X}}\left[\bar{G}(\tau_n+\theta_*^T\tilde{X})\log\bar{F}(\tau_n+\theta_*^T\tilde{X})+G(\tau_n+\theta_*^T\tilde{X})\log\alpha F(\tau_n+\theta_*^T\tilde{X})\right]\\
&\quad\quad\quad\quad\quad\quad\quad\quad\quad-\mathbb{E}_{(\tilde{Y},\tilde{X})}\{\log\frac{1}{N}\sum_{i=1}^N[1-(1-\alpha)F(\tau_n+\theta_1^T\tilde{X}_i)]\}\\
\\
&=\frac{\mathbb{E}_{X}\left[(1-F(\tau_n+\theta_*^T X))\log(1-F(\tau_n+\theta_1^T X))+\alpha F(\tau_n+\theta_*^T X)\log\alpha F(\tau_n+\theta_1^T X)\right]}{\mathbb{P}(Y=1)+\alpha\mathbb{P}(Y=0)}\\
&\quad\quad\quad\quad\quad\quad\quad\quad\quad-\mathbb{E}_{(\tilde{Y},\tilde{X})}\{\log\frac{1}{N}\sum_{i=1}^N[1-(1-\alpha)F(\tau_n+\theta_1^T\tilde{X}_i)]\}
\end{array}$$}
Note that $F(\tau_n+\theta_1^T x)\rightarrow 1$ as $n\rightarrow\infty$ for any $x\in\mathcal{X}$ and $\theta_1\in\Theta$. Since $x\log x\rightarrow0$ as $x\rightarrow0$, so 
{\footnotesize$\mathbb{E}\left[\tilde{\ell}(\tilde{X},\tilde{Y},\theta_*;\tau_n)\right]\rightarrow\log(\alpha)-\log(\alpha)=0$} regardless of the value of $\theta_*$ if $\alpha>0$. Thus as $n\rightarrow\infty$, the criterion function doesn't rely on the value of $\theta_1$ in the varying-rate regime, thus the classical MLE theory cannot be applied here. 

\subsection{Proof of Theorem \ref{thm:rate-of-convergence-infinity}}
\begin{proof}[Proof of Theorem \ref{thm:rate-of-convergence-infinity}] Given $\tau_n$, note that 
{$$\begin{array}{rl}
\hat{\theta}_{*}&=\argmax_{\theta_1\in\Theta}\frac{1}{N}\sum_{i=1}^N\tilde{Y}_i\log\bar{F}(\tau_{n}+\theta_1^T \tilde{X}_i)+(1-\tilde{Y}_i)\log \alpha F(\tau_{n}+\theta_1^T \tilde{X}_i)\\
&\quad\quad\quad\quad\quad\quad\quad\quad\quad\quad-\log\left\{\frac{1}{N}\sum_{i=1}^N\left[1-(1-\alpha)F(\tau_n+\theta_1^T \tilde{X}_i)\right]\right\}
\end{array}$$}

Denote 
{\small$$\begin{array}{rl}
&\quad L_n(\theta_1)\overset{\Delta}{=}L(\theta_1;\tau_n)\\
&\overset{\Delta}{=}\sum_{i=1}^N\left[\tilde{Y}_i\log\bar{F}(\tau_{n}+\theta_1^T \tilde{X}_i)+(1-\tilde{Y}_i)\log \alpha F(\tau_{n}+\theta_1^T \tilde{X}_i)\right]\\
&\quad\quad-\log\{\frac{1}{N}\sum_{i=1}^N[1-(1-\alpha)F(\tau_n+\theta_1^T \tilde{X}_i)]\}\\
&=\sum_{i=1}^N\left[\tilde{Y}_i\log\frac{\bar{F}(\tau_{n}+\theta_1^T \tilde{X}_i)}{\{\frac{1}{N}\sum_{i=1}^N[1-(1-\alpha)F(\tau_n+\theta_1^T \tilde{X}_i)]\}}+(1-\tilde{Y}_i)\log\frac{\alpha F(\tau_{n}+\theta_1^T \tilde{X}_i)}{\{\frac{1}{N}\sum_{i=1}^N[1-(1-\alpha)F(\tau_n+\theta_1^T \tilde{X}_i)]\}}\right]\\
&=\sum_{i=1}^n\bigg[Y_i\log\frac{\bar{F}(\tau_{n}+\theta_1^T X_i)}{\{\frac{1}{N}\sum_{i=1}^N[1-(1-\alpha)F(\tau_n+\theta_1^T \tilde{X}_i)]\}}\\
&\quad\quad\quad\quad+(1-Y_i)\mathbf{1}(U_i\leq\alpha)\log\frac{\alpha F(\tau_{n}+\theta_1^T X_i)}{\{\frac{1}{N}\sum_{i=1}^N[1-(1-\alpha)F(\tau_n+\theta_1^T \tilde{X}_i)]\}}\bigg],
\end{array}$$}
where $\{U_i\}_{i=1}^n$ is i.i.d. uniform random variable on $[0,1]$, and $U_i\indep\{Y_i,X_i\}_{i=1}^n$. Let $a_n=\sqrt{n(1-F(\tau_n))}$. So $w_n=a_n(\hat{\theta}_{*}-\theta_*)$ is the maximizer of 
{$$H(w):=L_n(\theta_*+a_n^{-1}w)-L_n(\theta_*).$$} 
For $w=a_n(\hat{\theta}-\theta_*)$, define $g(t)=L_n(\theta_*+t(\hat{\theta}-\theta_*))$ for $t\in[0,1]$. By Taylor expansion, for some $\gamma\in(0,1)$, we have 
{$g(1)-g(0)=g^\prime(0)+\frac{1}{2}g^{\prime\prime}(\gamma)$}, i.e. 
{$$\begin{array}{rl}
H(w)&=\nabla_{\theta_1}L_n(\theta_*)^T(a_n^{-1}w)+\frac{1}{2}(\hat{\theta}-\theta_*)^T\nabla_{\theta_1}^2L_n(\theta_*+\gamma(\hat{\theta}-\theta_*))(\hat{\theta}-\theta_*)\\
&=(a_n^{-1}w^T)\nabla_{\theta_1}L_n(\theta_*)+\frac{1}{2}a_n^{-2}w^T\nabla_{\theta_1}^2L_n(\theta_*+\gamma(\hat{\theta}-\theta_*))w.
\end{array}$$}
Denote $J_N(\theta_1):=\frac{1}{N}\sum_{i=1}^N[1-(1-\alpha)F(\tau_n+\theta_1^T \tilde{X}_i)],$
and
{$$\begin{array}{rl}
\ell(x,y,u,\theta_1;\tau_n)&=y\log\frac{\bar{F}(\tau_{n}+\theta_1^T x)}{\{\frac{1}{N}\sum_{i=1}^N[1-(1-\alpha)F(\tau_n+\theta_1^T \tilde{X}_i)]\}}\\
&\quad+(1-y)\mathbf{1}(u\leq\alpha)\log\frac{\alpha F(\tau_{n}+\theta_1^T x)}{\{\frac{1}{N}\sum_{i=1}^N[1-(1-\alpha)F(\tau_n+\theta_1^T \tilde{X}_i)]\}}\\
&=y\log\frac{\bar{F}(\tau_{n}+\theta_1^T x)}{J_N(\theta_1)}+(1-y)\mathbf{1}(u\leq\alpha)\log\frac{\alpha F(\tau_{n}+\theta_1^T x)}{J_N(\theta_1)}.
\end{array}$$}

Then $\nabla_{\theta_1}J_N(\theta_1)=-(1-\alpha)\frac{1}{N}\sum_{i=1}^NF^\prime(\tau_n+\theta_1^T \tilde{X}_i)\tilde{X}_i^T $, and 

{$$\begin{array}{rl}
&\quad\nabla_{\theta_1}\ell(x,y,u,\theta_1;\tau_n)\\
&=y\left[-\frac{F^\prime(\tau_n+\theta_1^T x)x^T}{1-F(\tau_n+\theta_1^T x)}+\frac{(1-\alpha)\frac{1}{N}\sum_{i=1}^NF^\prime(\tau_n+\theta_1^T \tilde{X}_i)\tilde{X}_i^T }{J_N(\theta_1)}\right]\\
&\quad\quad+(1-y)\mathbf{1}(u\leq\alpha)\left[\frac{F^\prime(\tau_n+\theta_1^T x)x^T}{F(\tau_n+\theta_1^T x)}+\frac{(1-\alpha)\frac{1}{N}\sum_{i=1}^NF^\prime(\tau_n+\theta_1^T \tilde{X}_i)\tilde{X}_i^T }{J_N(\theta_1)}\right]\\
&=F^\prime(\tau_n+\theta_1^T x)x^T\left[-\frac{y}{1-F(\tau_n+\theta_1^T x)}+\mathbf{1}(u\leq\alpha)\frac{1-y}{F(\tau_n+\theta_1^T x)}\right]\\
&\quad\quad+\left(y+(1-y)\mathbf{1}(u\leq\alpha)\right)\frac{(1-\alpha)\frac{1}{N}\sum_{i=1}^NF^\prime(\tau_n+\theta_1^T \tilde{X}_i)\tilde{X}_i^T }{J_N(\theta_1)}
\end{array}$$}
and {$$\begin{array}{rl}
&\quad\nabla_{\theta_1}L_n(\theta_1)\\
&=\sum_{i=1}^nF^\prime(\tau_n+\theta_1^T X_i)X_i\left[-\frac{Y_i}{1-F(\tau_n+\theta_1^T X_i)}+\mathbf{1}(U_i\leq\alpha)\frac{1-Y_i}{F(\tau_n+\theta_1^T X_i)}\right]\\
&\quad\quad\quad\quad+\left(Y_i+(1-Y_i)\mathbf{1}(U_i\leq\alpha)\right)\frac{(1-\alpha)\frac{1}{N}\sum_{i=1}^NF^\prime(\tau_n+\theta_1^T \tilde{X}_i)\tilde{X}_i^T }{J_N(\theta_1)}.
\end{array}$$} 

Hence 
{\small$$\begin{array}{rl}
&\quad\nabla_{\theta_1}^2L_n(\theta_1)=\sum_{i=1}^n\nabla_{\theta_1}^2\ell(X_i,Y_i,U_i,\theta_1;\tau_n)\\
&=\sum_{i=1}^nF^{\prime\prime}(\tau_n+\theta_1^T X_i)X_iX_i^T\left[-\frac{Y_i}{1-F(\tau_n+\theta_1^T X_i)}+\mathbf{1}(U_i\leq\alpha)\frac{1-Y_i}{F(\tau_n+\theta_1^T X_i)}\right]\\
&\quad+\sum_{i=1}^nF^\prime(\tau_n+\theta_1^T X_i)^2X_iX_i^T\left[\frac{-Y_i}{(1-F(\tau_n+\theta_1^T X_i))^2}-\frac{\mathbf{1}(U_i\leq\alpha)(1-Y_i)}{F(\tau_n+\theta_1^T X_i)^2}\right]\\
&\quad+\sum_{i=1}^n\left(Y_i+(1-Y_i)\mathbf{1}(U_i\leq\alpha)\right)\left[\frac{(1-\alpha)\frac{1}{N}\sum_{i=1}^NF^{\prime\prime}(\tau_n+\theta_1^T\tilde{X}_i)\tilde{X}_i\tilde{X}_i^T}{J_N(\theta_1)}\right]\\
&\quad+\sum_{i=1}^n\left(Y_i+(1-Y_i)\mathbf{1}(U_i\leq\alpha)\right)\left[\frac{(1-\alpha)^2\{\frac{1}{N}\sum_{i=1}^NF^\prime(\tau_n+\theta_1^T \tilde{X}_i)\tilde{X}_i\}\{\frac{1}{N}\sum_{i=1}^NF^\prime(\tau_n+\theta_1^T \tilde{X}_i)\tilde{X}_i^T\}}{J_N(\theta_1)^2}\right],
\end{array}$$}
and
{$\begin{array}{rl}
H(w)=(a_n^{-1}w^T)\nabla_{\theta_1}L_n(\theta_*)+\frac{1}{2}a_n^{-2}\sum_{i=1}^n\Phi_i(\theta_*+\gamma a_n^{-1}w),
\end{array}$}
where 
{$$\begin{array}{rl}
&\quad\Phi_i(\theta_1)\\
&=F^{\prime\prime}(\tau_n+\theta_1^T X_i)X_iX_i^T\left[-\frac{Y_i}{1-F(\tau_n+\theta_1^T X_i)}+\mathbf{1}(U_i\leq\alpha)\frac{1-Y_i}{F(\tau_n+\theta_1^T X_i)}\right]\\
&\quad+F^\prime(\tau_n+\theta_1^T X_i)^2X_iX_i^T\left[-\frac{Y_i}{(1-F(\tau_n+\theta_1^T X_i))^2}-\frac{\mathbf{1}(U_i\leq\alpha)(1-Y_i)}{F(\tau_n+\theta_1^T X_i)^2}\right]\\
&\quad+\left(Y_i+(1-Y_i)\mathbf{1}(U_i\leq\alpha)\right)\left[\frac{(1-\alpha)\frac{1}{N}\sum_{i=1}^NF^{\prime\prime}(\tau_n+\theta_1^T\tilde{X}_i)\tilde{X}_i\tilde{X}_i^T}{J_N(\theta_1)}\right]\\
&\quad+\left(Y_i+(1-Y_i)\mathbf{1}(U_i\leq\alpha)\right)\left[\frac{(1-\alpha)^2\{\frac{1}{N}\sum_{i=1}^NF^\prime(\tau_n+\theta_1^T \tilde{X}_i)\tilde{X}_i\}\{\frac{1}{N}\sum_{i=1}^NF^\prime(\tau_n+\theta_1^T \tilde{X}_i)\tilde{X}_i^T\}}{J_N(\theta_1)^2}\right]
\end{array}$$}
In the following, we want to show that for some matrices $\mathbf{V}$ and $\tilde{\mathbf{V}}_{\Phi}$, we have 
{$$a_n^{-1}\nabla_{\theta_1}L_n(\theta_*)\overset{d}{\rightarrow}\mathcal{N}\left(\mathbf{0},\mathbf{V}\right),$$}
and for any $u$ and $\gamma\in[0,1]$,
{$\begin{array}{rl}
a_n^{-2}\sum_{i=1}^n\Phi_i(\theta_*+\gamma a_n^{-1}w)\overset{p}{\rightarrow}\tilde{\mathbf{V}}_{\Phi}.
\end{array}$}
Note that 
{$$\begin{array}{rl}
&\quad\lim_{n\rightarrow\infty}\mathbb{E}\Big[F^\prime(\tau_n+\theta_*^T X_i)X_i\left[-\frac{Y_i}{1-F(\tau_n+\theta_*^T X_i)}+\mathbf{1}(U_i\leq\alpha)\frac{1-Y_i}{F(\tau_n+\theta_*^T X_i)}\right]\\
&\quad\quad+\left(Y_i+(1-Y_i)\mathbf{1}(U_i\leq\alpha)\right)\frac{(1-\alpha)\frac{1}{N}\sum_{i=1}^NF^\prime(\tau_n+\theta_*^T\tilde{X}_i)\tilde{X}_i^T }{J_N(\theta_*)}\Big]\\
&=\lim_{n\rightarrow\infty}\mathbb{E}\left[F^\prime(\tau_n+\theta_*^T X_i)X_i\left(-\frac{Y_i}{1-F(\tau_n+\theta_*^T X_i)}+\mathbf{1}(U_i\leq\alpha)\frac{1-Y_i}{F(\tau_n+\theta_*^T X_i)}\right)\right]\\
&\quad+(1-\alpha)[\mathbb{P}(Y=1)+\alpha\mathbb{P}(Y=0)]\frac{\tilde{E}[F^\prime(\tau_n+\theta_*^T\tilde{X})\tilde{X}]}{\tilde{E}[1-(1-\alpha)F(\tau_n+\theta_*^T\tilde{X})]}\\
&=_{(1)}\lim_{n\rightarrow\infty}\left(\alpha-1\right)\mathbb{E}\left[F^\prime(\tau_n+\theta_*^T X_i)X_i\right]\\
&\quad\quad+\frac{(1-\alpha)[\mathbb{P}(Y=1)+\alpha\mathbb{P}(Y=0)]\mathbb{E}[F^\prime(\tau_n+\theta_*^T X_i)(1-(1-\alpha)F(\tau_n+\theta_*^T X_i))X_i]}{\mathbb{E}[(1-(1-\alpha)F(\tau_n+\theta_*^TX))^2]}\\
&=\mathbf{0},
\end{array}$$}
where we use Lemma \ref{lemma:distribution-downsample-covariate} in (1), so {$\mathbb{E}\left[a_n^{-1}\nabla_{\theta_1}L_n(\theta_*)\right]\rightarrow\mathbf{0}.$}
Furthermore, by letting $\theta_1=\theta_*$ in the below, we have 
{\small$$\begin{array}{rl}
&\quad\lim_{n\rightarrow\infty}\mathrm{Cov}\left[a_n^{-1}\nabla_{\theta_1}L_n(\theta_*)\right]\\
&=\lim_{n\rightarrow\infty}a_n^{-2}\sum_{i=1}^n\mathrm{Cov}\Big[F^\prime(\tau_n+\theta_1^T X_i)X_i\left[-\frac{Y_i}{1-F(\tau_n+\theta_1^T X_i)}+\mathbf{1}(U_i\leq\alpha)\frac{1-Y_i}{F(\tau_n+\theta_1^T X_i)}\right]\\
&\quad\quad+\left(Y_i+(1-Y_i)\mathbf{1}(U_i\leq\alpha)\right)\frac{(1-\alpha)\frac{1}{N}\sum_{i=1}^NF^\prime(\tau_n+\theta_*^T\tilde{X}_i)\tilde{X}_i^T }{J_N(\theta_*)}\Big]\\
&=\lim_{n\rightarrow\infty}\sum_{i=1}^n\mathrm{Cov}\Big[\frac{F^\prime(\tau_n+\theta_1^T X_i)}{n(1-F(\tau_n))}X_i\left[-\frac{Y_i}{1-F(\tau_n+\theta_1^T X_i)}+\mathbf{1}(U_i\leq\alpha)\frac{1-Y_i}{F(\tau_n+\theta_1^T X_i)}\right]\\
&\quad\quad+\left(Y_i+(1-Y_i)\mathbf{1}(U_i\leq\alpha)\right)\frac{(1-\alpha)\frac{1}{N}\sum_{i=1}^NF^\prime(\tau_n+\theta_*^T\tilde{X}_i)\tilde{X}_i^T }{J_N(\theta_*)}\Big]\\
&=\lim_{n\rightarrow\infty}\mathrm{Cov}\Big[\frac{F^\prime(\tau_n+\theta_1^T X_i)}{\sqrt{1-F(\tau_n)}}X_i\left[-\frac{Y_i}{1-F(\tau_n+\theta_1^T X_i)}+\mathbf{1}(U_i\leq\alpha)\frac{1-Y_i}{F(\tau_n+\theta_1^T X_i)}\right]\\
&\quad\quad+\left(Y_i+(1-Y_i)\mathbf{1}(U_i\leq\alpha)\right)\frac{(1-\alpha)\frac{1}{N}\sum_{i=1}^NF^\prime(\tau_n+\theta_*^T\tilde{X}_i)\tilde{X}_i^T }{J_N(\theta_*)\sqrt{1-F(\tau_n)}}\Big]\\
&=\lim_{n\rightarrow\infty}\mathbb{E}\left[\frac{F^\prime(\tau_n+\theta_1^T X_i)^2}{1-F(\tau_n)}\left[-\frac{Y_i}{1-F(\tau_n+\theta_1^T X_i)}+\mathbf{1}(U_i\leq\alpha)\frac{1-Y_i}{F(\tau_n+\theta_1^T X_i)}\right]^2X_iX_i^T\right]\\
&\quad+\mathbb{E}\left[\left(Y_i+(1-Y_i)\mathbf{1}(U_i\leq\alpha)\right)^2\right]\frac{(1-\alpha)^2\tilde{E}[F^\prime(\tau_n+\theta_*^T\tilde{X}_i)\tilde{X}_i]\tilde{E}[F^\prime(\tau_n+\theta_*^T\tilde{X}_i)\tilde{X}_i^T]}{\tilde{E}[1-(1-\alpha)F(\tau_n+\theta_*^T\tilde{X})]^2(1-F(\tau_n))}\\
&\quad+2\mathbb{E}\Big[\frac{F^\prime(\tau_n+\theta_1^T X_i)X_i}{\sqrt{1-F(\tau_n)}}\left[-\frac{Y_i}{1-F(\tau_n+\theta_1^T X_i)}+\mathbf{1}(U_i\leq\alpha)\frac{1-Y_i}{F(\tau_n+\theta_1^T X_i)}\right]\\
&\quad\quad\quad\quad\times\frac{\left(Y_i+(1-Y_i)\mathbf{1}(U_i\leq\alpha)\right)}{\tilde{E}[1-(1-\alpha)F(\tau_n+\theta_*^T\tilde{X})]}(1-\alpha)\tilde{E}\left[\frac{F^\prime(\tau_n+\theta_1^T\tilde{X})\tilde{X}^T}{\sqrt{1-F(\tau_n)}}\right]\Big]\\
&=_{(g)}\lim_{n\rightarrow\infty}\mathbb{E}\left[\frac{F^\prime(\tau_n+\theta_*^T X_i)^2}{1-F(\tau_n)}\left[\frac{1}{1-F(\tau_n+\theta_*^T X_i)}+\frac{\alpha}{F(\tau_n+\theta_*^T X_i)}\right]X_iX_i^T\right]\\
&\quad\quad+\frac{(1-\alpha)^2\tilde{E}\left[\frac{F^\prime(\tau_n+\theta_*^T\tilde{X}_i)}{\sqrt{1-F(\tau_n)}}\tilde{X}_i\right]\tilde{E}\left[\frac{F^\prime(\tau_n+\theta_*^T \tilde{X}_i)}{\sqrt{1-F(\tau_n)}}\tilde{X}_i^T\right](\mathbb{P}(Y_i=1)+\alpha\mathbb{P}(Y_i=0))}{\tilde{E}[1-(1-\alpha)F(\tau_n+\theta_*^T\tilde{X})]^2}\\
&\quad\quad+2\mathbb{E}\left[\frac{F^\prime(\tau_n+\theta_*^T X_i)}{\sqrt{1-F(\tau_n)}}\frac{\mathbf{1}(U_i\leq\alpha)-1}{\tilde{E}[1-(1-\alpha)F(\tau_n+\theta_*^T\tilde{X})]}(1-\alpha)\tilde{E}\left[\frac{F^\prime(\tau_n+\theta_*^T\tilde{X}_i)\tilde{X}_i}{\sqrt{1-F(\tau_n)}}\right]X_i^T\right]\\
&=\lim_{n\rightarrow\infty}\mathbb{E}\left[\frac{F^\prime(\tau_n+\theta_*^T X_i)^2}{1-F(\tau_n)}\left[\frac{1}{1-F(\tau_n+\theta_*^T X_i)}+\frac{\alpha}{F(\tau_n+\theta_*^T X_i)}\right]X_iX_i^T\right]\\
&\quad\quad+\frac{(1-\alpha)^2\tilde{E}\left[\frac{F^\prime(\tau_n+\theta_*^T\tilde{X}_i)}{\sqrt{1-F(\tau_n)}}\tilde{X}_i\right]\tilde{E}\left[\frac{F^\prime(\tau_n+\theta_*^T \tilde{X}_i)}{\sqrt{1-F(\tau_n)}}\tilde{X}_i^T\right](\mathbb{P}(Y_i=1)+\alpha\mathbb{P}(Y_i=0))}{\tilde{E}[1-(1-\alpha)F(\tau_n+\theta_*^T\tilde{X})]^2}\\
&\quad\quad-2\frac{(1-\alpha)^2}{\tilde{E}[1-(1-\alpha)F(\tau_n+\theta_*^T\tilde{X})]}\mathbb{E}\left[\frac{F^\prime(\tau_n+\theta_*^T X_i)}{\sqrt{1-F(\tau_n)}}X_i\right]\tilde{E}\left[\frac{F^\prime(\tau_n+\theta_*^T\tilde{X}_i)\tilde{X}_i^T}{\sqrt{1-F(\tau_n)}}\right]\\
&=\mathbb{E}\left[\frac{F^\prime(\tau_n+\theta_*^T X_i)^2}{1-F(\tau_n)}\left[\frac{1}{1-F(\tau_n+\theta_*^T X_i)}+\frac{\alpha}{F(\tau_n+\theta_*^T X_i)}\right]X_iX_i^T\right]\\
&\quad\quad-\frac{(1-\alpha)^2}{\alpha}\mathbb{E}\left[\frac{F^\prime(\tau_n+\theta_*^T X_i)}{1-F(\tau_n)}X_i\right]\mathbb{E}\left[F^\prime(\tau_n+\theta_*^T X_i)X_i^T\right],
\end{array}$$}
where equation (g) uses the fact that 
{\footnotesize$\mathbb{E}[(Y_i+(1-Y_i)\mathbf{1}(U_i\leq\alpha))^2]=\mathbb{P}(Y=1)+\alpha\mathbb{P}(Y=0).$} Thus by dominated convergence theorem and Assumption \ref{ass:F-regularity}, 
{\small$$\begin{array}{rl}
&\quad\lim_{n\rightarrow\infty}\mathrm{Cov}\left[a_n^{-1}\nabla_{\theta_1}L_n(\theta_*)\right]\\
&=\mathbb{E}\left[g_1(\theta_*^T X)^2h(\theta_*^T X)XX^T\right]-\mathbb{E}\left[g_1(\theta_*^T X)h(\theta_*^T X)X\right]\mathbb{E}\left[\lim_{n\rightarrow\infty}\frac{(1-\alpha)^2F^\prime(\tau_n+\theta_*^T X)X^T}{\alpha}\right].
\end{array}$$}
Furthermore, we check the Lindeberg-Feller CLT condition for the asymptotic normality result. Recall that
{$$\begin{array}{rl}
&\quad\nabla_{\theta_1}\ell(x,y,u,\theta_*;\tau_n)\\
&=F^\prime(\tau_n+\theta_1^T x)x\left[-\frac{y}{1-F(\tau_n+\theta_1^T x)}+\mathbf{1}(u\leq\alpha)\frac{1-y}{F(\tau_n+\theta_1^T x)}\right]\\
&\quad+\left(y+(1-y)\mathbf{1}(u\leq\alpha)\right)\frac{(1-\alpha)\frac{1}{N}\sum_{i=1}^NF^\prime(\tau_n+\theta_1^T \tilde{X}_i)\tilde{X}_i^T }{J_N(\theta_1)}.
\end{array}$$}

For $\epsilon>0$, let $\mathcal{A}_i$ denote the event that $\norm{\nabla_{\theta_1}\ell(X_i,Y_i,U_i,\theta_*;\tau_n)}>a_n\epsilon$, we then have 
{\footnotesize$$\begin{array}{rl}
&\quad\sum_{i=1}^n\mathbb{E}\left[\norm{\nabla_{\theta_1}\ell(X_i,Y_i,U_i,\theta_*;\tau_n)}^2\mathbf{1}\left(\norm{\nabla_{\theta_1}\ell(X_i,Y_i,U_i,\theta_*;\tau_n)}>a_n\epsilon\right)\right]\\
&=n\mathbb{E}\left[\norm{\nabla_{\theta_1}\ell(X_i,Y_i,U_i,\theta_*;\tau_n)}^2\mathbf{1}\left(\norm{\nabla_{\theta_1}\ell(X_i,Y_i,U_i,\theta_*;\tau_n)}>a_n\epsilon\right)\right]\\
&\leq_{(1)} 2Cn\mathbb{E}\left[\left(\frac{1-Y_i}{F(\tau_n+\theta_*^T X_i)}\mathbf{1}(U_i\leq\alpha)-\frac{Y_i}{1-F(\tau_n+\theta_*^T X_i)}\right)^2F^\prime(\tau_n+\theta_*^T X_i)^2\norm{X_i}^2\mathbf{1}_{\mathcal{A}_i}\right]\\
&\quad+2Cn\mathbb{E}\left[\frac{(Y_i+(1-Y_i)\mathbf{1}(U_i\leq\alpha))^2(1-\alpha)^2}{J_N(\theta_*)^2}\norm{\tilde{E}[F^\prime(\tau_n+\theta_*^T \tilde{X})\tilde{X}]}^2\mathbf{1}_{\mathcal{A}_i}\right]\\
&=2n\mathbb{E}\left[\mathbb{E}\left[\left(\frac{1-Y_i}{F(\tau_n+\theta_*^T X_i)}\mathbf{1}(U_i\leq\alpha)-\frac{Y_i}{1-F(\tau_n+\theta_*^T X_i)}\right)^2F^\prime(\tau_n+\theta_*^T X_i)^2\norm{X_i}^2\mathbf{1}_{\mathcal{A}_i}\big|X_i\right]\right]\\
&\quad+2Cn\mathbb{E}\left[\mathbb{E}\left[\frac{(Y_i+(1-Y_i)\mathbf{1}(U_i\leq\alpha))^2(1-\alpha)^2}{J_N(\theta_*)^2}\norm{\tilde{E}[F^\prime(\tau_n+\theta_*^T \tilde{X})\tilde{X}]}^2\mathbf{1}_{\mathcal{A}_i}\big|X_i\right]\right]\\
&=2n\mathbb{E}\left[\left(\frac{\mathbf{1}\left(U_i\leq\alpha\right)}{F(\tau_n+\theta_*^T X_i)}+\frac{1}{1-F(\tau_n+\theta_*^T X_i)}\right)F^\prime(\tau_n+\theta_*^T X_i)^2\norm{X_i}^2\mathbf{1}_{\mathcal{A}_i}\right]\\
&\quad+2Cn\frac{(1-\alpha)^2\norm{\tilde{E}[F^\prime(\tau_n+\theta_*^T \tilde{X})\tilde{X}]}^2}{J_N(\theta_*)^2}\mathbb{E}\left[(1-F(\tau_n+\theta_*^T X_i)+\mathbf{1}(U_i\leq\alpha)F(\tau_n+\theta_*^T X_i))\mathbf{1}_{\mathcal{A}_i}\right]\\
&\leq 2n(1-F(\tau_n))\\
&\quad\times\mathbb{E}\bigg[\left(\frac{\mathbf{1}\left(U_i\leq\alpha\right)}{F(\tau_n+\theta_*^T X_i)}+\frac{1}{1-F(\tau_n+\theta_*^T X_i)}\right)\frac{F^\prime(\tau_n+\theta_*^T X_i)^2}{1-F(\tau_n)}\norm{X_i}^2\\
&\quad\quad\quad\times\mathbf{1}\left\{\frac{(1-Y_i)\mathbf{1}(U_i\leq\alpha)F^\prime(\tau_n+\theta_*^T X_i)\norm{X_i}}{F(\tau_n+\theta_*^T X_i)(1-F(\tau_n))}>\frac{n\epsilon}{2}\right\}\bigg]\\
&\quad+2Cn(1-F(\tau_n))\frac{(1-\alpha)^2\norm{\tilde{E}\left[\frac{F^\prime(\tau_n+\theta_*^T\tilde{X})}{1-F(\tau_n)}\tilde{X}\right]}^2}{J_N(\theta_*)^2}\\
&\quad\quad\quad\times\mathbb{E}\left[(1-F(\tau_n+\theta_*^T X_i)+\mathbf{1}(U_i\leq\alpha)F(\tau_n+\theta_*^T X_i))\mathbf{1}_{\mathcal{A}_i}\right]\\
&\overset{(*)}{=}\mathrm{o}(n(1-F(\tau_n)))=\mathrm{o}(a_n^2),
\end{array}$$}
where (1) is because {\footnotesize$\mathbb{E}[\norm{A+B}^2]\leq2\mathbb{E}[\norm{A}^2+\norm{B}^2]$}, and (*) uses dominated convergence theorem and Assumption \ref{ass:F-regularity}, and $C$ is some absolute constant. 

For the rest of the proof, we denote 
{$$g_1(\cdot):=\frac{F^\prime}{1-F(\tau_n)}\frac{1-F(\tau_n)}{1-F}=\frac{-h^{(1)}(\cdot)}{h},$$}
{$$g_2(\cdot):=\frac{F^{\prime\prime}}{1-F(\tau_n)}\frac{1-F(\tau_n)}{1-F}=\frac{-h^{(2)}}{h},$$}
{$$g_3(\cdot):=\frac{F^{\prime\prime\prime}}{1-F(\tau_n)}\frac{1-F(\tau_n)}{1-F}=\frac{-h^{(3)}}{h}.$$}

Thus applying the Lindeberg-Feller central limit theorem as Proposition 2.27 from \cite{van2000asymptotic} (i.e. Lemma \ref{lemma:CLT}), we have 
{$$a_n^{-1}\nabla_{\theta_1}L_n(\theta_*)\overset{d}{\rightarrow}\mathcal{N}\left(\mathbf{0},\mathbf{V}\right),$$}
where {\footnotesize$$\begin{array}{rl}
\mathbf{V}&=\mathbb{E}\left[g_1(\theta_*^T X)^2h(\theta_*^T X)XX^T\right]-\mathbb{E}\left[g_1(\theta_*^T X)h(\theta_*^T X)X\right]\mathbb{E}\left[\lim_{n\rightarrow\infty}\frac{(1-\alpha)^2F^\prime(\tau_n+\theta_*^T X)X^T}{\alpha}\right]\\
&=\mathbb{E}\left[g_1(\theta_*^T X)^2h(\theta_*^T X)XX^T\right]\\
&\quad-\left(\lim_{n\rightarrow\infty}\frac{(1-\alpha)^2(1-F(\tau_n))}{\alpha}\right)\mathbb{E}[h(\theta_*^T X)g_1(\theta_*^T X)X]\mathbb{E}\left[h(\theta_*^T X)g_1(\theta_*^T X)X^T\right]
\end{array}$$} 

Now we want to show that for any $u$ and any $\gamma\in[0,1]$, for some matrix $\tilde{\mathbf{V}}_{\Phi}$ we have 
{$$\begin{array}{rl}
&a_n^{-2}\sum_{i=1}^n\Phi_i(\theta_*+\gamma a_n^{-1}w)\overset{p}{\rightarrow}\tilde{\mathbf{V}}_{\Phi}.
\end{array}$$}

Note that 
{$$\begin{array}{rl}
&\quad \lim_{n\rightarrow\infty}a_n^{-2}\sum_{i=1}^n\Phi_i(\theta_*)\\
&=\lim_{n\rightarrow\infty}\frac{1}{n}\sum_{i=1}^n\frac{F^{\prime\prime}(\tau_n+\theta_*^T X_i)}{1-F(\tau_n)}X_iX_i^T\left[-\frac{Y_i}{1-F(\tau_n+\theta_*^T X_i)}+\mathbf{1}(U_i\leq\alpha)\frac{1-Y_i}{F(\tau_n+\theta_*^T X_i)}\right]\\
&\quad+\frac{1}{n}\sum_{i=1}^n\frac{F^\prime(\tau_n+\theta_*^T X_i)^2}{1-F(\tau_n)}X_iX_i^T\left[-\frac{Y_i}{(1-F(\tau_n+\theta_*^T X_i))^2}-\frac{\mathbf{1}(U_i\leq\alpha)(1-Y_i)}{F(\tau_n+\theta_*^T X_i)^2}\right]\\
&\quad+\frac{1}{n}\sum_{i=1}^n\frac{\left(Y_i+(1-Y_i)\mathbf{1}(U_i\leq\alpha)\right)}{1-F(\tau_n)}\left[\frac{(1-\alpha)\tilde{E}[F^{\prime\prime}(\tau_n+\theta_*^T\tilde{X}_i)\tilde{X}_i\tilde{X}_i^T]}{J_N(\theta_*)}\right]\\
&\quad+\frac{1}{n}\sum_{i=1}^n\frac{(1-\alpha)^2\tilde{E}[F^\prime(\tau_n+\theta_*^T \tilde{X})\tilde{X}]\tilde{E}[F^\prime(\tau_n+\theta_*^T\tilde{X})\tilde{X}^T]}{J_N(\theta_*)^2},
\end{array}$$}
and note that by dominated convergence theorem and Assumption \ref{ass:F-regularity}, 
{$$\begin{array}{rl}
&\quad\lim_{n\rightarrow\infty}\mathbb{E}\left[\frac{F^{\prime\prime}(\tau_n+\theta_*^T X_i)}{1-F(\tau_n)}X_iX_i^T\left(-\frac{Y_i}{1-F(\tau_n+\theta_*^T X_i)}+\mathbf{1}(U_i\leq\alpha)\frac{1-Y_i}{F(\tau_n+\theta_*^T X_i)}\right)\right]\\
&=(\alpha-1)\mathbb{E}\left[h(\theta_*^T X)g_2(\theta_*^T X)XX^T\right],
\end{array}$$}
{$$\begin{array}{rl}
&\quad\lim_{n\rightarrow\infty}\mathbb{E}\left[\frac{F^\prime(\tau_n+\theta_*^T X_i)^2}{1-F(\tau_n)}X_iX_i^T\left[-\frac{Y_i}{(1-F(\tau_n+\theta_*^T X_i))^2}-\frac{\mathbf{1}(U_i\leq\alpha)(1-Y_i)}{F(\tau_n+\theta_*^T X_i)^2}\right]\right]\\
&=-\mathbb{E}\left[g_1(\theta_*^T X)^2h(\theta_*^T X)XX^T\right],
\end{array}$$}
{$$\begin{array}{rl}
&\quad\lim_{n\rightarrow\infty}\mathbb{E}\bigg[\frac{\left(Y_i+(1-Y_i)\mathbf{1}(U_i\leq\alpha)\right)}{1-F(\tau_n)}\bigg[\frac{(1-\alpha)\tilde{E}[F^{\prime\prime}(\tau_n+\theta_*^T\tilde{X}_i)\tilde{X}_i\tilde{X}_i^T]}{\tilde{E}[1-(1-\alpha)F(\tau_n+\theta_*^T\tilde{X})]}\\
&\quad\quad\quad\quad\quad\quad+\frac{(1-\alpha)^2\tilde{E}[F^\prime(\tau_n+\theta_*^T \tilde{X})\tilde{X}]\tilde{E}[F^\prime(\tau_n+\theta_*^T\tilde{X})\tilde{X}^T]}{\tilde{E}[1-(1-\alpha)F(\tau_n+\theta_*^T\tilde{X})]^2}\bigg]\bigg]\\
&=(1-\alpha)\mathbb{E}[g_2(\theta_*^T X)h(\theta_*^T X)XX^T]\\
&\quad+\mathbb{E}[h(\theta_*^T X)g_1(\theta_*^T X)X]\mathbb{E}\left[\lim_{n\rightarrow\infty}\frac{(1-\alpha)^2F^\prime(\tau_n+\theta_*^T X)}{\alpha}X^T\right]\\
&=(1-\alpha)\mathbb{E}[g_2(\theta_*^T X)h(\theta_*^T X)XX^T]\\
&\quad+\left(\lim_{n\rightarrow\infty}\frac{(1-\alpha)^2(1-F(\tau_n))}{\alpha}\right)\mathbb{E}[h(\theta_*^T X)g_1(\theta_*^T X)X]\mathbb{E}\left[h(\theta_*^T X)g_1(\theta_*^T X)X^T\right]
\end{array}$$}

Thus {\footnotesize$$\begin{array}{rl}
&\quad a_n^{-2}\sum_{i=1}^n\Phi_i(\theta_*)\\
&\overset{p}{\rightarrow}-\mathbb{E}\left[g_1(\theta_*^T X)^2h(\theta_*^T X)XX^T\right]\\
&\quad \ +\left(\lim_{n\rightarrow\infty}\frac{(1-\alpha)^2(1-F(\tau_n))}{\alpha}\right)\mathbb{E}[h(\theta_*^T X)g_1(\theta_*^T X)X]\mathbb{E}\left[h(\theta_*^T X)g_1(\theta_*^T X)X^T\right].
\end{array}$$}

For the last step of the proof, we want to show that indeed 
{$$\begin{array}{rl}
\tilde{\mathbf{V}}_{\Phi}&=-\mathbb{E}\left[g_1(\theta_*^T X)^2h(\theta_*^T X)XX^T\right]\\
&\quad+\left(\lim_{n\rightarrow\infty}\frac{(1-\alpha)^2(1-F(\tau_n))}{\alpha}\right)\mathbb{E}[h(\theta_*^T X)g_1(\theta_*^T X)X]\mathbb{E}\left[h(\theta_*^T X)g_1(\theta_*^T X)X^T\right].
\end{array}$$}

Note that 
{$$\begin{array}{rl}
&\quad\nabla_{\theta_1}\Phi_i(\theta_1)\\
&=F^{(3)}(\tau_n+\theta_1^T X_i)X_i^T X_iX_i^T\left[-\frac{Y_i}{1-F(\tau_n+\theta_1^T X_i)}+\mathbf{1}(U_i\leq\alpha)\frac{1-Y_i}{F(\tau_n+\theta_1^T X_i)}\right]\\
&\quad+F^{\prime\prime}(\tau_n+\theta_1^T X_i)F^\prime(\tau_n+\theta_1^T X_i)X_i^T X_iX_i^T\left[-\frac{Y_i}{(1-F(\tau_n+\theta_1^T X_i))^2}-\frac{\mathbf{1}(U_i\leq\alpha)(1-Y_i)}{F(\tau_n+\theta_1^T X_i)^2}\right]\\
&\quad+2F^\prime(\tau_n+\theta_1^T X_i)F^{\prime\prime}(\tau_n+\theta_1^T X_i)X_i^T X_iX_i^T\left[-\frac{Y_i}{(1-F(\tau_n+\theta_1^T X_i))^2}-\frac{\mathbf{1}(U_i\leq\alpha)(1-Y_i)}{F(\tau_n+\theta_1^T X_i)^2}\right]\\
&\quad+F^\prime(\tau_n+\theta_1^T X_i)^2X_i^T X_iX_i^T\Big[-2Y_i(1-F(\tau_n+\theta_1^T X_i))^ {-3}F^\prime(\tau_n+\theta_1^T X_i)\\
&\quad\quad\quad\quad\quad\quad\quad\quad\quad\quad\quad\quad\quad+2\cdot\mathbf{1}(U_i\leq\alpha)(1-Y_i)F(\tau_n+\theta_1^T X_i)^{-3}F^\prime(\tau_n+\theta_1^T X_i)\Big]\\
&\quad+\left(Y_i+(1-Y_i)\mathbf{1}(U_i\leq\alpha)\right)\left[\frac{(1-\alpha)\tilde{E}_N\left[F^{(3)}(\tau_n+\theta_1^T\tilde{X}_i)\tilde{X}_i^T\tilde{X}_i\tilde{X}_i^T\right]}{\tilde{E}_N[1-(1-\alpha)F(\tau_n+\theta_1^T\tilde{X}_i)]}\right]\\
&\quad+(Y_i+(1-Y_i)\mathbf{1}(U_i\leq\alpha))\left[2(1-\alpha)^2\frac{\tilde{E}[F^{(2)}(\tau_n+\theta_1^T\tilde{X})\tilde{X}^T\tilde{X}]\tilde{E}[F^\prime(\tau_n+\theta_1^T\tilde{X})\tilde{X}^T]}{\tilde{E}[1-(1-\alpha)F(\tau_n+\theta_1^T\tilde{X})]^2}\right],
\end{array}$$}
hence
{$$\begin{array}{rl}
&\quad\left|a_n^{-2}\sum_{i=1}^n\norm{\Phi_i(\theta_*+\gamma a_n^{-1}w)}-a_n^{-2}\sum_{i=1}^n\norm{\Phi_i(\theta_*)}\right|\\
&\leq_{(d)} a_n^{-2}\norm{a_n^{-1}w}\sum_{i=1}^n\left|\Delta_i(\theta_*+\tilde{\gamma}a_n^{-1}w)\right|\norm{X_i}^3+a_n^{-2}\norm{a_n^{-1}w}\sum_{i=1}^n\left|\tilde{\Delta}_i(\theta_*+\tilde{\gamma}a_n^{-1}w)\right|\\
&=\frac{\norm{a_n^{-1}w}}{n}\sum_{i=1}^n\frac{\left|\Delta_i(\theta_*+\tilde{\gamma}a_n^{-1}w)\right|}{1-F(\theta_n)}\norm{X_i}^3+\frac{\norm{a_n^{-1}w}}{n}\sum_{i=1}^n\frac{\left|\tilde{\Delta}_i(\theta_*+\tilde{\gamma}a_n^{-1}w)\right|}{1-F(\theta_n)},
\end{array}$$}
where $\tilde{\gamma}$ is some constant such that $\tilde{\gamma}\in(0,1)$, inequality (d) uses mean value theorem and Cauchy-Schwarz inequality and the fact that $\gamma\in(0,1)$, and 
{$$\begin{array}{rl}
&\quad\Delta_i(\theta_*+\tilde{\gamma}a_n^{-1}w)\\
&=F^{(3)}(\tau_n+\theta_1^T X_i)\left[-\frac{Y_i}{1-F(\tau_n+\theta_1^T X_i)}+\mathbf{1}(U_i\leq\alpha)\frac{1-Y_i}{F(\tau_n+\theta_1^T X_i)}\right]\\
&\quad+F^{\prime\prime}(\tau_n+\theta_1^T X_i)F^\prime(\tau_n+\theta_1^T X_i)\left[-\frac{Y_i}{(1-F(\tau_n+\theta_1^T X_i))^2}-\frac{\mathbf{1}(U_i\leq\alpha)(1-Y_i)}{F(\tau_n+\theta_1^T X_i)^2}\right]\\
&\quad+2F^\prime(\tau_n+\theta_1^T X_i)F^{\prime\prime}(\tau_n+\theta_1^T X_i)\left[-\frac{Y_i}{(1-F(\tau_n+\theta_1^T X_i))^2}-\frac{\mathbf{1}(U_i\leq\alpha)(1-Y_i)}{F(\tau_n+\theta_1^T X_i)^2}\right]\\
&\quad+F^\prime(\tau_n+\theta_1^T X_i)^2\Big[-2Y_i(1-F(\tau_n+\theta_1^T X_i))^ {-3}F^\prime(\tau_n+\theta_1^T X_i)\\
&\quad\quad\quad\quad+2\cdot\mathbf{1}(U_i\leq\alpha)(1-Y_i)F(\tau_n+\theta_1^T X_i)^{-3}F^\prime(\tau_n+\theta_1^T X_i)\Big],
\end{array}$$}

{$$\begin{array}{rl}
&\quad\tilde{\Delta}_i(\theta_*+\tilde{\gamma}a_n^{-1}w)\\
&=\left(Y_i+(1-Y_i)\mathbf{1}(U_i\leq\alpha)\right)\left[\frac{(1-\alpha)\tilde{E}_N\left[F^{(3)}(\tau_n+\theta_1^T\tilde{X}_i)\tilde{X}_i^T\tilde{X}_i\tilde{X}_i^T\right]}{\tilde{E}_N[1-(1-\alpha)F(\tau_n+\theta_1^T\tilde{X}_i)]}\right]\\
&\quad+(Y_i+(1-Y_i)\mathbf{1}(U_i\leq\alpha))\left[2(1-\alpha)^2\frac{\tilde{E}[F^{(2)}(\tau_n+\theta_1^T\tilde{X})\tilde{X}^T\tilde{X}]\tilde{E}[F^\prime(\tau_n+\theta_1^T\tilde{X})\tilde{X}^T]}{\tilde{E}[1-(1-\alpha)F(\tau_n+\theta_1^T\tilde{X})]^2}\right].
\end{array}$$}

Again, similar to the previous argument, by Assumption \ref{ass:F-regularity} and dominated convergence theorem, the terms
{\footnotesize$\begin{array}{rl}
\lim_{n\rightarrow\infty}\mathbb{E}\left[\frac{|\Delta_i(\theta_*+\tilde{\gamma}a_n^{-1}w)|}{1-F(\theta_n)}\norm{X_i}^3\right],\lim_{n\rightarrow\infty}\mathbb{E}\left[\frac{|\tilde{\Delta}_i(\theta_*+\tilde{\gamma}a_n^{-1}w)|}{1-F(\theta_n)}\right]
\end{array}$} are bounded, so 
{$$\left|a_n^{-2}\sum_{i=1}^n\norm{\Phi_i(\theta_*+\gamma a_n^{-1}w)}-a_n^{-2}\sum_{i=1}^n\norm{\Phi_i(\theta_*)}\right|=\mathrm{o}_P(1).$$}
Combining with the previous arguments, we have proved that 
{$$\begin{array}{rl}
\tilde{\mathbf{V}}_{\Phi}&=-\mathbb{E}\left[g_1(\theta_*^T X)^2h(\theta_*^T X)XX^T\right]\\
&\quad+\left(\lim_{n\rightarrow\infty}\frac{(1-\alpha)^2(1-F(\tau_n))}{\alpha}\right)\mathbb{E}[h(\theta_*^T X)g_1(\theta_*^T X)X]\mathbb{E}\left[h(\theta_*^T X)g_1(\theta_*^T X)X^T\right].
\end{array}$$}
Recall that $a_n(\hat{\theta}_{*}-\theta_*)$ is the maximizer of 
{$$H(w)=(a_n^{-1}w^T)\nabla_{\theta_1}L_n(\theta_*)+\frac{1}{2}a_n^{-2}w^T\nabla_{\theta_1}^2L_n(\theta_*+\gamma(\hat{\theta}-\theta_*))w,$$}
which is equivalently the minimizer of 
{$-\frac{1}{2}a_n^{-2}w^T\nabla_{\theta_1}^2L_n(\theta_*+\gamma(\hat{\theta}-\theta_*))u-(a_n^{-1}w^T)\nabla_{\theta_1}L_n(\theta_*).$}
Then by the Basic Corollary of \cite{hjort2011asymptotics} (or Lemma \ref{lemma:Pollard}), we have
{$$a_n(\hat{\theta}-\theta_*)=-\tilde{\mathbf{V}}_{\Phi}^{-1}\times a_n^{-1}\nabla_{\theta_1}L_n(\theta_*)+\mathrm{o}_P(1)=\mathbf{V}_{\Phi}^{-1}\times a_n^{-1}\nabla_{\theta_1}L_n(\theta_*)+\mathrm{o}_P(1),$$}
where {$$\begin{array}{rl}
\mathbf{V}_{\Phi}&=\mathbb{E}\left[g_1(\theta_*^T X)^2h(\theta_*^T X)XX^T\right]\\
&\quad-\left(\lim_{n\rightarrow\infty}\frac{(1-\alpha)^2(1-F(\tau_n))}{\alpha}\right)\mathbb{E}[h(\theta_*^T X)g_1(\theta_*^T X)X]\mathbb{E}\left[h(\theta_*^T X)g_1(\theta_*^T X)X^T\right].
\end{array}$$}
By the given condition that $\lim_{n\rightarrow\infty}\frac{(1-\alpha)^2(1-F(\tau_n))}{\alpha}=c$, we have 
{$$\begin{array}{rl}
\mathbf{V}=\mathbf{V}_{\Phi}&=\mathbb{E}\left[g_1(\theta_*^T X)^2h(\theta_*^T X)XX^T\right]-c\mathbb{E}[h(\theta_*^T X)g_1(\theta_*^T X)X]\mathbb{E}\left[h(\theta_*^T X)g_1(\theta_*^T X)X^T\right]\\
&=\mathbb{E}\left[\frac{h^{(1)}(\theta_*^T X)^2}{h(\theta_*^T X)}XX^T\right]-c\mathbb{E}\left[h^{(1)}(\theta_*^T X)X\right]\mathbb{E}\left[h^{(1)}(\theta_*^T X)X^T\right]
\end{array}$$}
Thus we have 
{\footnotesize$\sqrt{n(1-F(\tau_n))}(\hat{\theta}_{*}-\theta_*)\overset{d}{\rightarrow}\mathcal{N}\left(\mathbf{0},{\mathbf{V}}_{\Phi}^{-1}\mathbf{V}{\mathbf{V}}_{\Phi}^{-1}\right)=\mathcal{N}\left(\mathbf{0},\mathbf{V}^{-1}\right).$}
\end{proof}

\subsection{Proof for Theorem \ref{thm:generalize:asymptotics}}
\begin{proof}[Proof of Theorem \ref{thm:generalize:asymptotics}]
By definition, the (scaled) maximum likelihood estimator is now defined as 
{$$\begin{array}{rl}
r(\tau_n)\hat{\theta}_{*}&=\argmax_{\theta_1\in\Theta}\frac{1}{N}\sum_{i=1}^N\tilde{Y}_i\log\bar{F}(\tau_{n}+\theta_1^T \tilde{X}_i)+(1-\tilde{Y}_i)\log \alpha F(\tau_{n}+\theta_1^T \tilde{X}_i)\\
&\quad\quad\quad\quad\quad\quad\quad\quad\quad\quad-\log\{\frac{1}{N}\sum_{i=1}^N[1-(1-\alpha)F(\tau_n+\theta_1^T\tilde{X}_i)]\},
\end{array}$$}
then following similar steps as in the proof of Theorem \ref{thm:rate-of-convergence-infinity}, we use the same definition of $L_n(\theta_1)$ and $a_n$, and we can get {$$a_n^{-1}\nabla_{\theta_1}L_n(r(\tau_n)\theta_*)\overset{d}{\rightarrow}\mathcal{N}(0,\mathbf{V}),$$}
where 
{$$\begin{array}{rl}
\mathbf{V}&=\lim_{n\rightarrow\infty}\mathbb{E}\left[\frac{F^\prime(\tau_n+r(\tau_n)\theta_*^T X_i)^2}{1-F(\tau_n)}\left[\frac{1}{1-F(\tau_n+r(\tau_n)\theta_*^T X_i)}+\frac{\alpha}{F(\tau_n+r(\tau_n)\theta_*^T X_i)}\right]X_iX_i^T\right]\\
&\quad-\lim_{n\rightarrow\infty}\frac{(1-\alpha)^2\mathbb{E}\left[\frac{F^\prime(\tau_n+r(\tau_n)\theta_*^T X_i)}{1-F(\tau_n)}X_i\right]\mathbb{E}\left[F^\prime(\tau_n+r(\tau_n)\theta_*^T X_i)X_i^T\right]}{\tilde{E}[1-(1-\alpha)F(\tau_n+r(\tau_n)\theta_*^T\tilde{X}_i)]}.
\end{array}$$}
Note that by Assumption \ref{ass:generalized:asymptotic} for any $\theta_1\in\Theta$ and $x\in\mathcal{X}$, we have 
{$$g(\theta^\prime x)=\frac{1-F(r(\tau_n)(\theta^\prime x)+\tau_n)}{1-F(\tau_n)},$$}
thus $\frac{F^\prime(\tau_n+r(\tau_n)\theta_*^T x)}{1-F(\tau_n+r(\tau_n)\theta_*^T x)}=\frac{\frac{F^\prime(\tau_n+r(\tau_n)\theta_*^T x)}{1-F(\tau_n)}}{\frac{1-F(\tau_n+r(\tau_n)\theta_*^T x)}{1-F(\tau_n)}}=-\frac{g^\prime(\theta_*^T x)}{g(\theta_*^T x)}$. So by dominated convergence theorem, we have 
{\footnotesize$$\mathbf{V}=\mathbb{E}\left[\frac{g^{(1)}(\theta_*^T X)^2}{g(\theta_*^T X)}XX^T\right]-\lim_{n\rightarrow\infty}\frac{(1-\alpha)^2(1-F(\tau_n))}{\alpha}\mathbb{E}\left[g^{(1)}(\theta_*^T X)X\right]\mathbb{E}\left[g^{(1)}(\theta_*^T X)X^T\right].$$}
Similarly, we can also get 
{$$\begin{array}{rl}
a_n^{-2}\sum_{i=1}^n\Phi_i(r(\tau_n)(\theta_*+\gamma a_n^{-1}w))\overset{p}{\rightarrow}-\mathbf{V}_{\Phi},
\end{array}$$}
where
{$$\mathbf{V}_{\Phi}=\mathbb{E}\left[\frac{g^{(1)}(\theta_*^T X)^2}{g(\theta_*^T X)}XX^T\right]-\lim_{n\rightarrow\infty}\frac{(1-\alpha)^2(1-F(\tau_n))}{\alpha}\mathbb{E}\left[g^{(1)}(\theta_*^T X)X\right]\mathbb{E}\left[g^{(1)}(\theta_*^T X)X^T\right].$$}
Then the rest of the proof follows by similar steps of checking regularity conditions, etc. as in Theorem \ref{thm:rate-of-convergence-infinity}.
\end{proof}

\subsection{Technical Lemmas}
\begin{lemma}[Proposition 2.27 from \cite{van2000asymptotic}]\label{lemma:CLT}
For each $n$ let $Y_{n,1},\ldots,Y_{n,k_n}$ be independent random vectors with finite variances such that for every $\epsilon>0$,
{\footnotesize$\sum_{i=1}^{k_n}\mathbb{E}[\norm{Y_{n,i}}^2]\mathbf{1}\{\norm{Y_{n,i}}>\epsilon\}\rightarrow0,$} and 
{\footnotesize$\sum_{i=1}^{k_n} Y_{n,i}\rightarrow\Sigma,$}
then the sequence $\sum_{i=1}^{k_n}(Y_{n,i}-\mathbb{E}[Y_{n,i}])$ converges in distribution to $\mathcal{N}(\mathbf{0},\Sigma)$. 
\end{lemma}

\begin{lemma}[Basic Corollary of \cite{hjort2011asymptotics}]\label{lemma:Pollard}
Let $A_n(s)=\frac{1}{2}s^\prime Vs+U_n^\prime s+C_n+r_n(s)$, where $V$ is symmetric and postive definite, $U_n$ is stochastically bounded, $C_n$ is arbitrary, and $r_n(s)$ goes to zero in probability for every $s$. Then $\alpha_n=\argmin A_n$ is $o_p(1)$ away from $-V^{-1}U_n$ as the argmin of $\frac{1}{2}s^\prime Vs+U_n^\prime s+C_n$. If also $U_n\overset{d}{\rightarrow}U$ then $\alpha_n\overset{d}{\rightarrow}-V^{-1}U$. 
\end{lemma}

\begin{lemma}[Theorem 5.7 of \cite{van2000asymptotic}]\label{lemma:thm-vdv-M-estimator}
Let $M_n$ be random functions and let $M$ be a fixed function of $\theta$ such that for $\forall \epsilon>0$, {\footnotesize$$\sup_{\theta\in\Theta}\left|M_n(\theta)-M(\theta)\right|\overset{p}{\rightarrow}0,\ \ \mbox{and}\ \ \sup_{\theta:d(\theta,\theta_*)\geq\epsilon}M(\theta)<M(\theta_*).$$}
Then any sequence of $\hat{\theta}_n$ with $M_n(\hat{\theta}_n)\geq M_n(\theta_*)-\mathrm{o}_{\mathbb{P}}(1)$ converges in probability to $\theta_*$.
\end{lemma}

\section{Proofs for Efficiency with a Budget Constraint}

\begin{proof}[Proof of Theorem \ref{thm:asymptotic:efficiency-cost}]
We use $g_1(\cdot)$ to denote $-\frac{h^{(1)}(\cdot)}{h(\cdot)}$. The objective function is equal to 

{$$\begin{array}{rl}
&\quad\lim_{n\rightarrow\infty}\frac{p_1+\alpha(1-p_1)}{\mathrm{tr}\left\{\mathbf{V}\right\}}\\
&=\lim_{n\rightarrow\infty}\frac{p_1+\alpha(1-p_1)}{\mathrm{tr}\left\{\mathbb{E}\left[g_1(\theta_*^T X)^2h(\theta_*^T X)XX^T\right]-c\mathbb{E}\left[g_1(\theta_*^T X)h(\theta_*^T X)X\right]\mathbb{E}\left[g_1(\theta_*^T X)h(\theta_*^T X)X\right]\right\}},
\end{array}$$} 
where {$\lim_{n\rightarrow\infty}\frac{(1-\alpha)^2(1-F(\tau_n))}{\alpha}=c.$} Thus replacing $c$ with $\frac{(1-\alpha)^2(1-F(\tau_n))}{\alpha}$ in the above equation, the objective becomes 
{\small$$\begin{array}{rl}
J_n(\alpha):=\frac{[p_1+\alpha(1-p_1)]}{\mathrm{tr}\left\{\mathbb{E}\left[g_1(\theta_*^T X)^2h(\theta_*^T X)XX^T\right]-\frac{(1-\alpha)^2(1-F(\tau_n))}{\alpha}\mathbb{E}\left[g_1(\theta_*^T X)h(\theta_*^T X)X\right]\mathbb{E}\left[g_1(\theta_*^T X)h(\theta_*^T X)X\right]\right\}}.
\end{array}$$}
So 
{$$\begin{array}{rl}
\log J_n(\alpha)&=\log(p_1+\alpha(1-p_1))\\
&\quad-\log\big[\mathrm{tr}\big\{\mathbb{E}\left[g_1(\theta_*^T X)^2h(\theta_*^T X)XX^T\right]\\
&\quad\quad\quad\quad-\frac{(1-\alpha)^2(1-F(\tau_n))}{\alpha}\mathbb{E}\left[g_1(\theta_*^T X)h(\theta_*^T X)X\right]\mathbb{E}\left[g_1(\theta_*^T X)h(\theta_*^T X)X\right]\big\}\big],
\end{array}$$}
taking derivative with respect to $\alpha$, we have 
{$$\begin{array}{rl}
&\quad\frac{\partial\log J_n(\alpha)}{\partial\alpha}\\
&=\frac{(1-p_1)}{p_1+\alpha(1-p_1)}\\
&\quad-\frac{\left(1/\alpha^2-1\right)(1-F(\tau_n))\mathrm{tr}\left\{\mathbb{E}\left[g_1(\theta_*^T X)h(\theta_*^T X)X\right]\mathbb{E}\left[g_1(\theta_*^T X)h(\theta_*^T X)X\right]\right\}}{\mathrm{tr}\left\{\mathbb{E}\left[g_1(\theta_*^T X)^2h(\theta_*^T X)XX^T\right]-\frac{(1-\alpha)^2(1-F(\tau_n))}{\alpha}\mathbb{E}\left[g_1(\theta_*^T X)h(\theta_*^T X)X\right]\mathbb{E}\left[g_1(\theta_*^T X)h(\theta_*^T X)X\right]\right\}}.
\end{array}$$}
So 
{$$\begin{array}{rl}
&\quad\quad\frac{\partial\log J_n(\alpha)}{\partial\alpha}=0\\
&\iff\frac{\mathbb{E}\left[g_1(\theta_*^T X)^2h(\theta_*^T X)XX^T\right]}{\mathbb{E}\left[g_1(\theta_*^T X)h(\theta_*^T X)X\right]\mathbb{E}\left[g_1(\theta_*^T X)h(\theta_*^T X)X\right]}=\left[\frac{(1-\alpha)^2}{\alpha}+\left(\frac{p_1}{\alpha(1-p_1)}+1\right)\frac{1-\alpha^2}{\alpha}\right](1-F(\tau_n)),
\end{array}$$}
where we use abuse of notation with the equality sign holds above when $n\rightarrow\infty$. Then we have $\alpha^*=\mathrm{o}(1)$, because if $\alpha$ is bounded away from zero, the right hand side of the above equation is $\mathrm{o}(1)$ while the left hand side is $\mathrm{O}(1)$. Let $\beta=\lim_{n\rightarrow\infty}\frac{1-F(\tau_n)}{\alpha}$, since $\alpha\rightarrow0$ as $n\rightarrow\infty$, and note that $\frac{p_1}{1-p_1}\rightarrow0$, then the right hand side converges to $2\beta$, so we have 
{$$\beta=\frac{1}{2}\frac{\mathrm{tr}\left\{\mathbb{E}\left[g_1(\theta_*^T X)^2h(\theta_*^T X)XX^T\right]\right\}}{\mathrm{tr}\left\{\mathbb{E}\left[g_1(\theta_*^T X)h(\theta_*^T X)X\right]\mathbb{E}\left[g_1(\theta_*^T X)h(\theta_*^T X)X\right]\right\}},$$}
indicating that 
{$$\alpha^*=\frac{2(1-F(\tau_n))\mathrm{tr}\left\{\mathbb{E}\left[g_1(\theta_*^T X)h(\theta_*^T X)X\right]\mathbb{E}\left[g_1(\theta_*^T X)h(\theta_*^T X)X\right]\right\}}{\mathrm{tr}\left\{\mathbb{E}\left[g_1(\theta_*^T X)^2h(\theta_*^T X)XX^T\right]\right\}}.$$}

\end{proof}

\section{Proofs for Application to Logistic Regression}
\begin{proof}[Proof of Proposition \ref{prop:logistic-regression:asymptotic}]
Note that {\footnotesize$$h(\theta_1^T x)=\lim_{n\rightarrow\infty}\frac{1-F(\tau_n+\theta_1^T x)}{1-F(\tau_n)}=e^{-\theta_1^T x},$$} $g_1(\theta_1^T x)=\lim_{n\rightarrow\infty}\frac{F^\prime(\tau_n+\theta_1^T x)}{1-F(\tau_n+\theta_1^T x)}=\lim_{n\rightarrow\infty}F(\tau_n+\theta_1^T x)=1$, and also we have $g_2(\theta_1^T x)=\lim_{n\rightarrow\infty}\frac{F^{\prime\prime}(\tau_n+\theta_1^T x)}{1-F(\tau_n+\theta_1^T x)}=-1$, $g_3(\theta_1^T x)=\lim_{n\rightarrow\infty}\frac{F^{(3)}(\tau_n+\theta_1^T x)}{1-F(\tau_n+\theta_1^T x)}=1$, so by Theorem \ref{thm:rate-of-convergence-infinity} we have
{\footnotesize$$\sqrt{\frac{n}{1+e^{\tau_n}}}\left(\hat{\theta}_{*}-\theta_*\right)\overset{d}{\rightarrow}\mathcal{N}\left(\mathbf{0},\mathbf{V}^{-1}\right),\ \ \mbox{where}$$}
{\footnotesize$$\begin{array}{rl}\mathbf{V}&=\mathbb{E}\left[g_1(\theta_*^T X)^2h(\theta_*^T X)XX^T\right]-c\mathbb{E}\left[g_1(\theta_*^T X)h(\theta_*^T X)X\right]\mathbb{E}\left[g_1(\theta_*^T X)h(\theta_*^T X)X\right]\\
&=\mathbb{E}\left[e^{-\theta_*^T X}XX^T\right]-c\mathbb{E}\left[e^{-\theta_*^T X}X\right]\mathbb{E}\left[e^{-\theta_*^T X}X\right]
\end{array}$$}
Thus 
{$$\sqrt{ne^{-\tau_n}}\left(\hat{\theta}_{*}-\theta_*\right)\overset{d}{\rightarrow}\mathcal{N}\left(\mathbf{0},\mathbf{V}^{-1}\right),$$}
then by dominated convergence theorem and Slutsky's theorem, we have 
{$$\sqrt{n\mathbb{E}\left[\frac{1}{1+e^{\tau_n+\theta_*^T X}}\right]}\left(\hat{\theta}_{*}-\theta_*\right)\overset{d}{\rightarrow}\mathcal{N}\left(\mathbf{0},\mathbb{E}_X\left[e^{-\theta_*^T X}\right]\mathbf{V}^{-1}\right),$$}
which gives 
{$$\sqrt{n\mathbb{P}(Y=1)}\left(\hat{\theta}_{*}-\theta_*\right)\overset{d}{\rightarrow}\mathcal{N}\left(\mathbf{0},\mathbb{E}_X\left[e^{-\theta_*^T X}\right]\mathbf{V}^{-1}\right).$$}
\end{proof}

\begin{proof}[Proof of Proposition \ref{prop:efficiency:asymptotic}]
    First note that Assumption \ref{ass:regularity} holds. We use $g_1(\cdot), g_2(\cdot), g_3(\cdot)$ to denote $-h^{(1)}(\cdot)/h(\cdot),-h^{(2)}(\cdot)/h(\cdot),-h^{(3)}(\cdot)/h(\cdot)$. By definition, for {$$F(\tau_n+\theta_1^T x)=\frac{e^{\tau_n+\theta_1^T x}}{1+e^{\tau_n+\theta_1^T x}}, h(\theta_1^T x)=\lim_{n\rightarrow\infty}\frac{1-F(\tau_n+\theta_1^T x)}{1-F(\tau_n)}=\lim_{n\rightarrow\infty}\frac{1+e^{\tau_n}}{1+e^{\tau_n+\theta_1^T x}}=e^{-\theta_1^T x},$$} {$$g_1(\theta_1^T x)=\lim_{n\rightarrow\infty}\frac{F^\prime(\tau_n+\theta_1^T x)}{1-F(\tau_n+\theta_1^T x)}=1,$$}
    {$$g_2(\theta_1^T x)=\lim_{n\rightarrow\infty}\frac{F^{\prime\prime}(\tau_n+\theta_1^T x)}{1-F(\tau_n+\theta_1^T x)}=\lim_{n\rightarrow\infty}\frac{F^{\prime\prime}(\tau_n+\theta_1^T x)}{1-F(\tau_n+\theta_1^T x)}=-1,$$}
    {$g_3(\theta_1^T x)=\lim_{n\rightarrow\infty}\frac{F^{(3)}(\tau_n+\theta_1^T x)}{1-F(\tau_n+\theta_1^T x)}=0$}. Thus the conditions in Assumption \ref{ass:F-regularity} are satisfied. Further note that the conditions of Theorem \ref{thm:rate-of-convergence-infinity} also hold.
    Then by Theorem \ref{thm:asymptotic:efficiency-cost}, 
    {$$\alpha^*=\frac{2(1+e^{\tau_n})^{-1}\mathrm{tr}\left\{\mathbb{E}[e^{-\theta_*^T X}X]\mathbb{E}[e^{-\theta_*^T X}X]\right\}}{\mathrm{tr}\{\mathbb{E}[e^{-\theta_*^T X}XX^T]\}}$$}
\end{proof}


\begin{thebibliography}{10}

\bibitem{arnold2020reflection}
Kellyn~F Arnold, Vinny Davies, Marc de~Kamps, Peter~WG Tennant, John Mbotwa, and Mark~S Gilthorpe.
\newblock Reflection on modern methods: generalized linear models for prognosis and intervention—theory, practice and implications for machine learning.
\newblock {\em International journal of epidemiology}, 49(6):2074--2082, 2020.

\bibitem{basha2022review}
Shaik~Johny Basha, Srinivasa~Rao Madala, Kolla Vivek, Eedupalli~Sai Kumar, and Tamminina Ammannamma.
\newblock A review on imbalanced data classification techniques.
\newblock In {\em 2022 International conference on advanced computing technologies and applications (ICACTA)}, pages 1--6. IEEE, 2022.

\bibitem{bubeck2015convex}
S{\'e}bastien Bubeck et~al.
\newblock Convex optimization: Algorithms and complexity.
\newblock {\em Foundations and Trends{\textregistered} in Machine Learning}, 8(3-4):231--357, 2015.

\bibitem{chawla2010data}
Nitesh~V Chawla.
\newblock Data mining for imbalanced datasets: An overview.
\newblock {\em Data mining and knowledge discovery handbook}, pages 875--886, 2010.

\bibitem{chawla2004special}
Nitesh~V Chawla, Nathalie Japkowicz, and Aleksander Kotcz.
\newblock Special issue on learning from imbalanced data sets.
\newblock {\em ACM SIGKDD explorations newsletter}, 6(1):1--6, 2004.

\bibitem{deng2022model}
Zeyu Deng, Abla Kammoun, and Christos Thrampoulidis.
\newblock A model of double descent for high-dimensional binary linear classification.
\newblock {\em Information and Inference: A Journal of the IMA}, 11(2):435--495, 2022.

\bibitem{ding2005classification}
Beiying Ding and Robert Gentleman.
\newblock Classification using generalized partial least squares.
\newblock {\em Journal of Computational and Graphical Statistics}, 14(2):280--298, 2005.

\bibitem{dobson2018introduction}
Annette~J Dobson and Adrian~G Barnett.
\newblock {\em An introduction to generalized linear models}.
\newblock Chapman and Hall/CRC, 2018.

\bibitem{douzas2017self}
Georgios Douzas and Fernando Bacao.
\newblock Self-organizing map oversampling (somo) for imbalanced data set learning.
\newblock {\em Expert systems with Applications}, 82:40--52, 2017.

\bibitem{drummond2003c4}
Chris Drummond, Robert~C Holte, et~al.
\newblock C4. 5, class imbalance, and cost sensitivity: why under-sampling beats over-sampling.
\newblock In {\em Workshop on learning from imbalanced datasets II}, volume~11, 2003.

\bibitem{elad1997restoration}
Michael Elad and Arie Feuer.
\newblock Restoration of a single superresolution image from several blurred, noisy, and undersampled measured images.
\newblock {\em IEEE transactions on image processing}, 6(12):1646--1658, 1997.

\bibitem{Elkan_2001}
Charles Elkan.
\newblock The foundations of cost-sensitive learning.
\newblock JCAI'01, page 973–978, San Francisco, CA, USA, 2001. Morgan Kaufmann Publishers Inc.

\bibitem{estabrooks2004multiple}
Andrew Estabrooks, Taeho Jo, and Nathalie Japkowicz.
\newblock A multiple resampling method for learning from imbalanced data sets.
\newblock {\em Computational intelligence}, 20(1):18--36, 2004.

\bibitem{fithian2014local}
William Fithian and Trevor Hastie.
\newblock Local case-control sampling: Efficient subsampling in imbalanced data sets.
\newblock {\em Annals of statistics}, 42(5):1693, 2014.

\bibitem{foster2019complexity}
Dylan~J Foster, Ayush Sekhari, Ohad Shamir, Nathan Srebro, Karthik Sridharan, and Blake Woodworth.
\newblock The complexity of making the gradient small in stochastic convex optimization.
\newblock In {\em Conference on Learning Theory}, pages 1319--1345. PMLR, 2019.

\bibitem{glynn1992asymptotic}
Peter~W Glynn and Ward Whitt.
\newblock The asymptotic efficiency of simulation estimators.
\newblock {\em Operations research}, 40(3):505--520, 1992.

\bibitem{gorriz2021connection}
JM~Gorriz, Carmen Jimenez-Mesa, Ferm{\'\i}n Segovia, Javier Ram{\'\i}rez, SiPBA Group, and J~Suckling.
\newblock A connection between pattern classification by machine learning and statistical inference with the general linear model.
\newblock {\em IEEE Journal of Biomedical and Health Informatics}, 26(11):5332--5343, 2021.

\bibitem{haixiang2017learning}
Guo Haixiang, Li~Yijing, Jennifer Shang, Gu~Mingyun, Huang Yuanyue, and Gong Bing.
\newblock Learning from class-imbalanced data: Review of methods and applications.
\newblock {\em Expert systems with applications}, 73:220--239, 2017.

\bibitem{han2005borderline}
Hui Han, Wen-Yuan Wang, and Bing-Huan Mao.
\newblock Borderline-smote: a new over-sampling method in imbalanced data sets learning.
\newblock In {\em International conference on intelligent computing}, pages 878--887. Springer, 2005.

\bibitem{hassan2016modeling}
Amira Kamil~Ibrahim Hassan and Ajith Abraham.
\newblock Modeling insurance fraud detection using imbalanced data classification.
\newblock In {\em Advances in Nature and Biologically Inspired Computing: Proceedings of the 7th World Congress on Nature and Biologically Inspired Computing (NaBIC2015) in Pietermaritzburg, South Africa, held December 01-03, 2015}, pages 117--127. Springer, 2016.

\bibitem{hjort2011asymptotics}
Nils~Lid Hjort and David Pollard.
\newblock Asymptotics for minimisers of convex processes.
\newblock {\em arXiv preprint arXiv:1107.3806}, 2011.

\bibitem{hsu2021proliferation}
Daniel Hsu, Vidya Muthukumar, and Ji~Xu.
\newblock On the proliferation of support vectors in high dimensions.
\newblock In {\em International Conference on Artificial Intelligence and Statistics}, pages 91--99. PMLR, 2021.

\bibitem{huda2016hybrid}
Shamsul Huda, John Yearwood, Herbert~F Jelinek, Mohammad~Mehedi Hassan, Giancarlo Fortino, and Michael Buckland.
\newblock A hybrid feature selection with ensemble classification for imbalanced healthcare data: A case study for brain tumor diagnosis.
\newblock {\em IEEE access}, 4:9145--9154, 2016.

\bibitem{japkowicz2000learning}
Nathalie Japkowicz.
\newblock {\em Learning from Inbalanced Data Sets: Papers from the AAAI Workshop}.
\newblock AAAI Press, 2000.

\bibitem{jiang2021multi}
Dan Jiang, Rongbin Xu, Xin Xu, and Ying Xie.
\newblock Multi-view feature transfer for click-through rate prediction.
\newblock {\em Information Sciences}, 546:961--976, 2021.

\bibitem{johnson2013accelerating}
Rie Johnson and Tong Zhang.
\newblock Accelerating stochastic gradient descent using predictive variance reduction.
\newblock {\em Advances in neural information processing systems}, 26, 2013.

\bibitem{king2001logistic}
Gary King and Langche Zeng.
\newblock Logistic regression in rare events data.
\newblock {\em Political analysis}, 9(2):137--163, 2001.

\bibitem{krawczyk2016learning}
Bartosz Krawczyk.
\newblock Learning from imbalanced data: open challenges and future directions.
\newblock {\em Progress in Artificial Intelligence}, 5(4):221--232, 2016.

\bibitem{ksieniewicz2018undersampled}
Pawel Ksieniewicz.
\newblock Undersampled majority class ensemble for highly imbalanced binary classification.
\newblock In {\em Second International Workshop on Learning with Imbalanced Domains: Theory and Applications}, pages 82--94. PMLR, 2018.

\bibitem{lee2021comparison}
Jungwon Lee, Okkyung Jung, Yunhye Lee, Ohsung Kim, and Cheol Park.
\newblock A comparison and interpretation of machine learning algorithm for the prediction of online purchase conversion.
\newblock {\em Journal of Theoretical and Applied Electronic Commerce Research}, 16(5):1472--1491, 2021.

\bibitem{lee2012estimating}
Kuang-chih Lee, Burkay Orten, Ali Dasdan, and Wentong Li.
\newblock Estimating conversion rate in display advertising from past erformance data.
\newblock In {\em Proceedings of the 18th ACM SIGKDD international conference on Knowledge discovery and data mining}, pages 768--776, 2012.

\bibitem{lee2022downsampling}
Wonjae Lee and Kangwon Seo.
\newblock Downsampling for binary classification with a highly imbalanced dataset using active learning.
\newblock {\em Big Data Research}, 28:100314, 2022.

\bibitem{lemaavztre2017imbalanced}
Guillaume Lema{\~A}{\v{Z}}tre, Fernando Nogueira, and Christos~K Aridas.
\newblock Imbalanced-learn: A python toolbox to tackle the curse of imbalanced datasets in machine learning.
\newblock {\em Journal of machine learning research}, 18(17):1--5, 2017.

\bibitem{li2017rare}
Jinyan Li, Simon Fong, Shimin Hu, Victor~W Chu, Raymond~K Wong, Sabah Mohammed, and Nilanjan Dey.
\newblock Rare event prediction using similarity majority under-sampling technique.
\newblock In {\em Soft Computing in Data Science: Third International Conference, SCDS 2017, Yogyakarta, Indonesia, November 27--28, 2017, Proceedings 3}, pages 23--39. Springer, 2017.

\bibitem{li2017adaptive}
Jinyan Li, Lian-sheng Liu, Simon Fong, Raymond~K Wong, Sabah Mohammed, Jinan Fiaidhi, Yunsick Sung, and Kelvin~KL Wong.
\newblock Adaptive swarm balancing algorithms for rare-event prediction in imbalanced healthcare data.
\newblock {\em PloS one}, 12(7):e0180830, 2017.

\bibitem{mathew2017classification}
Josey Mathew, Chee~Khiang Pang, Ming Luo, and Weng~Hoe Leong.
\newblock Classification of imbalanced data by oversampling in kernel space of support vector machines.
\newblock {\em IEEE transactions on neural networks and learning systems}, 29(9):4065--4076, 2017.

\bibitem{more2017review}
AS~More and Dipti~P Rana.
\newblock Review of random forest classification techniques to resolve data imbalance.
\newblock In {\em 2017 1st International conference on intelligent systems and information management (ICISIM)}, pages 72--78. IEEE, 2017.

\bibitem{nelder1972generalized}
John~Ashworth Nelder and Robert~WM Wedderburn.
\newblock Generalized linear models.
\newblock {\em Journal of the Royal Statistical Society Series A: Statistics in Society}, 135(3):370--384, 1972.

\bibitem{Niculescu-Mizil_Caruana_2005}
Alexandru Niculescu-Mizil and Rich Caruana.
\newblock Predicting good probabilities with supervised learning.
\newblock In {\em Proceedings of the 22nd International Conference on Machine Learning}, ICML '05, page 625–632, New York, NY, USA, 2005. Association for Computing Machinery.

\bibitem{owen2007infinitely}
Art~B Owen.
\newblock Infinitely imbalanced logistic regression.
\newblock {\em Journal of Machine Learning Research}, 8(4), 2007.

\bibitem{rahman2013addressing}
M~Mostafizur Rahman and Darryl~N Davis.
\newblock Addressing the class imbalance problem in medical datasets.
\newblock {\em International Journal of Machine Learning and Computing}, 3(2):224, 2013.

\bibitem{rakhlin2011making}
Alexander Rakhlin, Ohad Shamir, and Karthik Sridharan.
\newblock Making gradient descent optimal for strongly convex stochastic optimization.
\newblock {\em arXiv preprint arXiv:1109.5647}, 2011.

\bibitem{razzaghi2015fast}
Talayeh Razzaghi, Oleg Roderick, Ilya Safro, and Nick Marko.
\newblock Fast imbalanced classification of healthcare data with missing values.
\newblock In {\em 2015 18th International Conference on Information Fusion (Fusion)}, pages 774--781. IEEE, 2015.

\bibitem{shah2021impacts}
Akshay Shah and Siddhesh Nasnodkar.
\newblock The impacts of user experience metrics on click-through rate (ctr) in digital advertising: A machine learning approach.
\newblock {\em Sage Science Review of Applied Machine Learning}, 4(1):27--44, 2021.

\bibitem{shelke2017review}
Mayuri~S Shelke, Prashant~R Deshmukh, and Vijaya~K Shandilya.
\newblock A review on imbalanced data handling using undersampling and oversampling technique.
\newblock {\em Int. J. Recent Trends Eng. Res}, 3(4):444--449, 2017.

\bibitem{sisodia2022data}
Deepti Sisodia and Dilip~Singh Sisodia.
\newblock Data sampling strategies for click fraud detection using imbalanced user click data of online advertising: an empirical review.
\newblock {\em IETE Technical Review}, 39(4):789--798, 2022.

\bibitem{sisodia2023hybrid}
Deepti Sisodia and Dilip~Singh Sisodia.
\newblock A hybrid data-level sampling approach in learning from skewed user-click data for click fraud detection in online advertising.
\newblock {\em Expert Systems}, 40(2):e13147, 2023.

\bibitem{taha2021multilabel}
Adil~Yaseen Taha, Sabrina Tiun, Abdul~Hadi Abd~Rahman, and Ali Sabah.
\newblock Multilabel over-sampling and under-sampling with class alignment for imbalanced multilabel text classification.
\newblock {\em Journal of Information and Communication Technology}, 20(3):423--456, 2021.

\bibitem{van2000asymptotic}
Aad~W Van~der Vaart.
\newblock {\em Asymptotic statistics}, volume~3.
\newblock Cambridge university press, 2000.

\bibitem{wang2020logistic}
HaiYing Wang.
\newblock Logistic regression for massive data with rare events.
\newblock In {\em International Conference on Machine Learning}, pages 9829--9836. PMLR, 2020.

\bibitem{wu2009low}
Xiaolin Wu, Xiangjun Zhang, and Xiaohan Wang.
\newblock Low bit-rate image compression via adaptive down-sampling and constrained least squares upconversion.
\newblock {\em IEEE Transactions on Image Processing}, 18(3):552--561, 2009.

\bibitem{yan2015deep}
Yilin Yan, Min Chen, Mei-Ling Shyu, and Shu-Ching Chen.
\newblock Deep learning for imbalanced multimedia data classification.
\newblock In {\em 2015 IEEE international symposium on multimedia (ISM)}, pages 483--488. IEEE, 2015.

\bibitem{yan2019oversampling}
Yuguang Yan, Mingkui Tan, Yanwu Xu, Jiezhang Cao, Michael Ng, Huaqing Min, and Qingyao Wu.
\newblock Oversampling for imbalanced data via optimal transport.
\newblock In {\em Proceedings of the AAAI Conference on Artificial Intelligence}, volume~33, pages 5605--5612, 2019.

\bibitem{zhang2014downsampled}
Gang Zhang, Lixin Wang, Alistair~P Duffy, Hugh Sasse, Danilo Di~Febo, Antonio Orlandi, and Karol Aniserowicz.
\newblock Downsampled and undersampled datasets in feature selective validation (fsv).
\newblock {\em IEEE transactions on electromagnetic compatibility}, 56(4):817--824, 2014.

\end{thebibliography}
\end{document}